\documentclass{article}
\usepackage[utf8]{inputenc}
\usepackage{authblk}
\usepackage{setspace}
\usepackage[margin=1.25in]{geometry}
\usepackage{graphicx}
\graphicspath{ {./figures/} }
\usepackage{subcaption}
\usepackage{amsmath}
\usepackage{booktabs}
\usepackage{wrapfig}
\usepackage{amssymb}
\usepackage{mathtools}
\usepackage{amsthm}
\usepackage{url}
\usepackage[numbers]{natbib}
\usepackage{xcolor}
\usepackage{hyperref}
\usepackage{bbm}
\usepackage{nicefrac}
\usepackage{multirow}
\usepackage{multicol}
\usepackage{dsfont}
\usepackage{amsfonts}
\usepackage{adjustbox}

\theoremstyle{plain}
\newtheorem{theorem}{Theorem}[section]

\newtheorem{lemma}[theorem]{Lemma}
\newtheorem{corollary}[theorem]{Corollary}
\theoremstyle{definition}

\newtheorem{assumption}[theorem]{Assumption}
\theoremstyle{remark}

\newcommand{\norm}[1]{\left\lVert#1\right\rVert}
\newcommand{\abs}[1]{\left\vert#1\right\vert}

\newcommand{\R}{\mathbb R}

\renewcommand{\thefootnote}{\fnsymbol{footnote}}

\def\Cov{\mbox{Cov}}

\def\Var{\mbox{Var}}

\def\E{\mathbb{E}}

\def\qv{\mathbf q}

\newcommand{\Ic}{\mathcal{I}}

\newcommand{\Qc}{\mathcal{Q}}
\newcommand{\Rc}{\mathcal{R}}

\newcommand{\Tc}{\mathcal{T}}

\newcommand{\Vc}{\mathcal{V}}

\newcommand{\Zc}{\mathcal{Z}}

\def\1v{\mathbf 1}
\def\0v{\mathbf 0}

\def\1{\mathbbm{1}}
\def\P{\mathbb{P}}

\DeclareMathOperator*{\esssup}{ess\,sup}
\newcommand{\ind}{\mathds{1}}
\newcommand{\epshat}{\widehat\epsilon}
\newcommand{\FH}{F_{\widehat\epsilon}^{H}}
\newcommand{\FZ}{F_{\widehat\epsilon}}
\newcommand{\FhatH}{\widehat F_{T,\hat\epsilon}^{H}}
\newcommand{\floor}[1]{\lfloor#1\rfloor}

\colorlet{severeUC}{black!40}  
\definecolor{mildUC}{RGB}{180,65,65}

\title{Conformal Prediction for Time Series with \\ Deep Sequence Models}

\author[1]{Junghwan Lee$^\dagger$}
\author[1]{Jonghyeok Lee$^\dagger$}
\author[1]{Yao Xie$^*$}

\affil[1]{H. Milton Stewart School of Industrial and Systems Engineering, Georgia Institute of Technology}

\date{}

\begin{document}

\maketitle

\begingroup
\renewcommand{\thefootnote}{\fnsymbol{footnote}}
\footnotetext[2]{Equal contribution.}
\footnotetext[1]{Correspondence: yao.xie@isye.gatech.edu}
\endgroup

\begin{abstract}
Recent advances in deep learning for time series prediction have amplified the need for reliable uncertainty quantification. Conformal prediction has gained attention as a distribution-free framework for constructing prediction intervals with coverage guarantees. However, its coverage guarantees rely on data exchangeability, an assumption generally violated in time series. Active research has focused on developing conformal prediction methods for time series that overcome this limitation. While deep sequence models, such as recurrent neural networks and Transformers, have often been used in conformal prediction for time series, limited work has systematically studied how deep sequence models can be utilized in conformal prediction for time series. In this work, we systematically investigate the use of deep sequence models in conformal prediction for time series through three approaches: conditional quantile regression, conditional quantile function estimation, and localized conformal prediction. We provide a theoretical analysis establishing asymptotic conditional coverage guarantees for all three approaches under suitable assumptions. Through comprehensive experiments on real-world datasets, we demonstrate the effectiveness of leveraging deep sequence models into conformal prediction for time series. 
\end{abstract}


\section{Introduction}

Time series prediction aims to forecast future outcomes from past sequential observations of features and outcomes~\citep{box2015time}, and has been widely applied across a broad range of downstream tasks, such as climate modeling, biological science, and finance~\citep{lim2021time}. Recent advances in deep learning have led to the development of various foundation models for time series prediction~\citep{miller2024survey, kim2025comprehensive, wen2023transformers, kong2025deep}. The growing use of such foundation models highlights the need for reliable uncertainty quantification in time series prediction. Conformal prediction has been actively adopted in time series prediction as a reliable and distribution-free framework for uncertainty quantification, constructing prediction sets with desired coverage guarantees~\citep{zhou2025conformal}.

\begin{figure}[ht]
    \centering
    \includegraphics[width=0.85\linewidth]{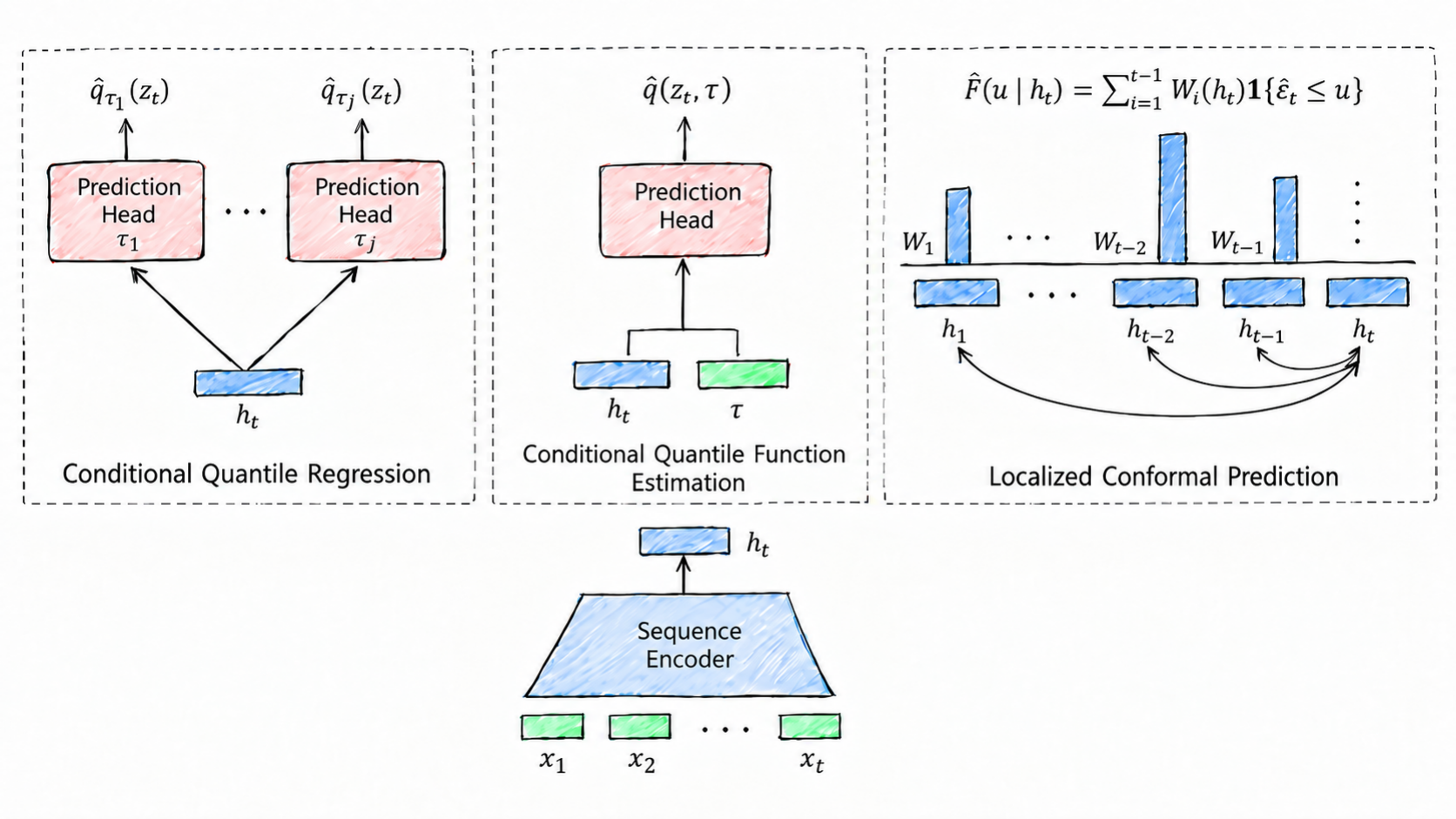}
    \caption{Three approaches for leveraging deep sequence models in conformal prediction for time series. All methods use a deep sequence model to encode the time-series context, and the resulting representation is used differently within each method: (1) quantile regression with quantile-specific prediction heads conditioned on the representation; (2) quantile function estimation for arbitrary quantile levels conditioned on the representation; and (3) similarity-based weighting computed from the representations for estimating the conditional cumulative distribution function.}
    \label{fig:methods_figure}
\end{figure}

Conformal prediction provides prediction intervals with coverage guarantees under the assumption that the data are exchangeable, without requiring additional distributional assumptions~\citep{vovk2005algorithmic,vovk2012conditional}. Despite these attractive properties for uncertainty quantification, applying conformal prediction to time series forecasting remains challenging as the temporal dependence present in time series generally violates the exchangeability assumption.  Consequently, substantial research has focused on extending conformal prediction to time series by developing methods that account for temporal dependence and the resulting violations of exchangeability~\citep{xu2021conformal,xu2023conformal,auer2024conformal,lee2024kernel,menadjiev2026distmatch}. Within this line of research, deep sequence models, such as recurrent neural networks~\citep{elman1990finding,hochreiter1997long,cho2014learning} and Transformers~\citep{vaswani2017attention}, have been increasingly adopted due to their ability to capture complex temporal dependencies~\citep{lee2024transformer,chen2024conformalized,li2025neural,cini2025relational,neglia2026rescp,lee2026flow}. However, a systematic understanding of how deep sequence models can be leveraged and their practical effectiveness in conformal prediction for time series remains limited.

To address this gap, we make the following contributions:
\begin{itemize}
    \item We propose three methods that leverage deep sequence models for conformal prediction for time series: conditional quantile regression, conditional quantile function estimation, and localized conformal prediction.

    \item We establish asymptotic conditional coverage guarantees for all three methods under suitable assumptions.

    \item We conduct comprehensive experiments on real-world datasets, demonstrating the competitive performance of the proposed methods and providing practical insights into their use.
\end{itemize}

\subsection{Related Works}
\label{sec:literature_review}

\paragraph{Conformal prediction beyond the exchangeability assumption.}


A line of work on conformal prediction under distribution shift focuses on adaptively updating the calibration procedure online. \citet{gibbs2021adaptive} address distribution shifts in sequential data by adaptively updating a coverage level based on past coverage errors, and \citet{zaffran2022adaptive} extend this framework to time series by proposing an aggregation-based approach to improve its adaptivity. More recent work along this line includes~\citet{bhatnagar2023improved,angelopoulos2023conformal,angelopoulos2024online,gibbs2024conformal,areces2025online}. Another line of work on conformal prediction under distribution shift between the calibration and test set is by reweighting the data points. \citet{tibshirani2019conformal} address covariate shift by reweighting calibration samples according to the likelihood ratio between the test and training covariate distributions, requiring this ratio to be known or accurately estimated. \citet{barber2023conformal} uses predefined weights, such as exponentially decaying weights that assign greater importance to more recent observations. However, determining appropriate weights remains challenging when the underlying shift is unknown, and inaccurate weight estimates or poorly chosen weighting schemes may lead to coverage degradation or inefficient prediction intervals. Although not specifically designed to address non-exchangeability, localized conformal prediction~\citep{guan2023localized} uses similarity-based weights to emphasize calibration samples near the test point and adapt prediction intervals to local distributional characteristics.



\paragraph{Conformal prediction for time series.}

Early approaches combine ensemble-based predictors with sequential conformal calibration. \citet{xu2021conformal,xu2023conformal} use bootstrap ensembles with a rolling window of prediction residuals, and \citet{jensen2022ensemble} combine ensemble quantile regression with conformal calibration to adapt prediction interval widths. Related to the idea of localized conformal prediction~\citep{guan2023localized}, a growing body of work reweights calibration observations based on their relevance to the current prediction. \citet{auer2024conformal} use modern Hopfield networks to assign greater weight to observations from similar temporal regimes, and \citet{lee2024kernel} employ a reweighted Nadaraya--Watson estimator to obtain data-adaptive weights for quantile estimation of prediction residuals. \citet{chen2024conformalized} use deep sequence models to learn latent representations for attention-based reweighting of calibration samples. \citet{neglia2026rescp} use similarities between reservoir states to reweight historical residuals, capturing local temporal dynamics without training a calibration network.

Another important body of work estimates conditional quantiles of future prediction residuals given historical observations. \citet{xu2023sequential} use quantile random forests to exploit dependencies among residuals, and \citet{lee2024transformer} extend this approach using a Transformer to leverage the past residuals and covariates. \citet{cini2025relational} combine sequence models with graph neural networks to estimate residual quantiles across correlated time series while learning the relational structure from data. \citet{menadjiev2026distmatch} introduce DistMatch, which groups residual histories into distributionally similar bins using Kolmogorov--Smirnov statistics and performs bin-specific quantile estimation.

\subsection{Problem Setup}
\label{sec:problem_setup}

Consider a sequence of observations $\{ (X_t, Y_t) \}_{t \ge 1}$, where $X_t \in \mathbb{R}^d$ denotes a $d$-dimensional feature vector and $Y_t \in \mathbb{R}$ denotes a continuous scalar outcome. Given a forecast horizon $\ell \ge 1$, let $Z_t$ denote the contextual information available at time $t$, used to construct a prediction interval for $Y_{t+\ell}$. For instance, if the most recent $w$ features are used as context, then $Z_t = [X_{t-(w-1)}, \ldots , X_t]$.

We assume that a point prediction model $\widehat f$ is available to produce
\begin{equation}
    \widehat Y_{t + \ell} = \widehat f(Z_t'), \quad \epshat_{t+\ell} = Y_{t+\ell} - \widehat f(Z_t').
\end{equation}
Here, $Z'_t$ denotes the information used by the point prediction model and need not coincide with $Z_t$. The choice of $Z'_t$ depends on the prediction model $\widehat f$ and can include any information available to the predictor at the time of prediction. For example, a standard regression model may use $Z'_t=X_t$, whereas a time-series foundation model uses a longer historical context. We use the signed prediction residual $\epshat_{t + \ell}$ as the nonconformity score to construct prediction intervals.

Write $T$ for the number of observed context--residual pairs used to fit a quantile model or to calibrate a fixed representation. All the methods we propose estimate residual quantiles, denoted by $\widehat q_{T, \tau}(z)$. For fixed levels $0 < \alpha_l < \alpha_h < 1$ with $\alpha_h - \alpha_l = 1 - \alpha$, define
\begin{equation}\label{eq:common-interval}
    \widehat C_T^{\alpha, \ell}(z, \widehat y) = [\widehat y + \widehat q_{T, \alpha_l}(z), \widehat y + \widehat q_{T, \alpha_h}(z)].
\end{equation}
The observed pairs are $\{(Z_i,\epshat_{i+\ell})\}_{i=1}^T$, so the last calibration outcome is available at time $T+\ell$. The prediction interval at time $t \ge T + \ell$ is evaluated at $(Z_t, \widehat Y_{t+\ell})$. Throughout the analysis, $t_T\ge T+\ell$ is an arbitrary deterministic sequence of forecast origins.


We seek narrow intervals with coverage conditional on the observed context, a stronger criterion than average marginal coverage~\citep{vovk2012conditional}. However, exact distribution-free conditional coverage with finite-length intervals is generally impossible~\citep{lei2014distribution}. Our theoretical target is therefore
\begin{equation}\label{eq:conditional-target}
    \P \left( Y_{t_T+\ell}\in\widehat C_T^{\alpha,\ell}(Z_{t_T},\widehat Y_{t_T+\ell}) \mid Z_{t_T} \right) = 1 - \alpha +o_P(1).
\end{equation}
Since the coverage error is bounded by one, \eqref{eq:conditional-target} also implies marginal coverage converging to $1-\alpha$.

\section{Methods}

In this section, we propose three approaches to leverage deep sequence models in conformal prediction for time series. Figure~\ref{fig:methods_figure} provides an overview of the three proposed methods. The three methods use a sequence encoder but use its encoded representation in different ways: as a finite collection of quantile heads, a head that also takes the quantile level as input, or kernel weights for observed calibration residuals. Given the context $z_t$ at time $t$, an encoder $\operatorname{Enc}_{\theta'}(\cdot)$ maps $z_t$ to a latent representation $h_t$ that summarizes the relevant contextual information, as $h_t=\operatorname{Enc}_{\theta'}(z_t)$, where $\theta'$ denotes the parameters of the encoder. We write $H_t=\operatorname{Enc}_{\theta'}(Z_t)$ for the corresponding random representation.


\subsection{Conditional Quantile Regression using Deep Sequence Models}
\label{sec:cqr}

Conditional quantile regression with deep sequence models has been explored in probabilistic time series forecasting \citep{wen2017multi,chen2020probabilistic,lim2021temporal}. Building on this idea, a straightforward way to leverage deep sequence models in conformal prediction for time series is to model the conditional quantiles of prediction residuals given their historical context. Specifically, a deep sequence model encodes the historical time series context into a latent representation, which is then used as input to multiple prediction heads, one for each target quantile level. The resulting quantile estimates are used to construct prediction intervals. 

Let $\mathcal{T}=\{\tau_1<\ldots<\tau_J\} \subset (0,1)$ denote the set of quantile levels at which the conditional quantiles of the prediction residuals are estimated.
We assume that the number of quantile levels $J$ is fixed throughout training and inference.
For notational simplicity, we omit the subscript
when referring to a generic quantile level $\tau\in\mathcal{T}$.
In an unconstrained multi-head parameterization, for each $\tau\in\mathcal{T}$ the corresponding quantile-specific prediction head $g_{\theta^\tau}(\cdot)$ takes $h_t$ as input and produces an estimate of the conditional $\tau$-quantile of the prediction residual $\widehat q_{T,\tau}$, given by
\[
\widehat q_{T,\tau}(z_t)=g_{\theta^\tau}(h_t)=g_{\theta^\tau}\left(\operatorname{Enc}_{\theta'}(z_t)\right),
\]
where $\theta^\tau$ denotes the parameters of the prediction head associated with quantile level $\tau$. 

Rather than employing a separate encoder for each quantile level in $\mathcal{T}$, a common architecture uses a single shared encoder together with multiple quantile-specific prediction heads \citep{wen2017multi,lim2021temporal}. Transformer \citep{vaswani2017attention} and recurrent neural networks \citep{hochreiter1997long,cho2014learning} are common choices for the encoder.

Although the model is trained using quantile loss, quantile crossing may occur due to estimation, approximation, or optimization errors \citep{bondell2010noncrossing}. To prevent quantile crossing, we parameterize the prediction heads using nonnegative increments. The head corresponding to the lowest quantile level predicts its quantile value directly, while each subsequent head produces a nonnegative increment that is added to the preceding quantile estimate. Specifically,
we define
\[
\begin{aligned}
\widehat q_{T,\tau_1}(z_t)
&= g_{\theta^{\tau_1}}(h_t),\\
\widehat q_{T,\tau_j}(z_t)
&= \widehat q_{T,\tau_{j-1}}(z_t)
+ \sigma\!\left(g_{\theta^{\tau_j}}(h_t)\right),
\quad j = 2,\ldots,J,
\end{aligned}
\]
where
$\sigma:\mathbb{R}\to[0,\infty)$ is a nonnegative activation function, such as ReLU or softplus. This construction guarantees
\[
\widehat{q}_{T,\tau_1}(z_t) \leq \cdots \leq \widehat{q}_{T,\tau_J}(z_t),
\]
for every context $z_t$. We refer to this conditional quantile regression approach using a deep sequence encoder as DSCP-QR throughout the paper.


\subsection{Conditional Quantile Function Estimation using Deep Sequence Models}
\label{sec:method_cqfe}

Conditional quantile regression with deep sequence models provides a direct way to estimate prediction residual quantiles given historical context. However, the multi-head formulation requires a separate prediction head for each target quantile level, restricting estimation to a predefined set of levels. To address this limitation, we consider conditional quantile function estimation, in which the quantile level is provided as an additional model input alongside the encoded historical context. This allows a single prediction head to estimate arbitrary quantile levels. Similar formulations have been explored in distributional reinforcement learning \citep{dabney2018implicit,yang2019fully}, uncertainty estimation \citep{tagasovska2019single}, and probabilistic time series forecasting \citep{gouttes2021probabilistic}. 


Rather than specifying a finite set of target quantile levels, we allow the quantile level $\tau$ to vary continuously over $(0,1)$.
The prediction head combines $h_t$ with the requested scalar level $\tau$ to estimate the corresponding conditional quantile of the prediction residual:
\[
\widehat q_{T, \tau}(z_t)
=
g_{\theta''}\left(h_t, \tau \right)
=
g_{\theta''}\left(
\operatorname{Enc}_{\theta'}(z_t),\tau
\right),
\]
where $\theta''$ denotes the prediction-head parameters. 

Unlike the finite-head approach, this construction does not restrict the endpoints to quantile levels specified during model construction. By varying the requested levels, the same trained model can produce intervals with different nominal coverage levels and different allocations of nominal miscoverage between the lower and upper tails, without additional prediction heads or retraining.

Similar to conditional quantile regression, quantile crossing may occur when estimating quantiles using the learned conditional quantile function. However, the nonnegative-increment parameterization used in \S~\ref{sec:cqr} is not applicable here, since the conditional quantile function is designed to estimate quantiles at arbitrary levels rather than at a finite set of prespecified quantile levels. To prevent quantile crossing, we therefore construct the prediction head as a neural network that is partially monotonic with respect to the quantile level $\tau$, while remaining unrestricted with respect to the context representation $h_t$. Specifically, the prediction head is constrained to satisfy
$$
g_{\theta''}(h_t,\tau_1)
\leq
g_{\theta''}(h_t,\tau_2),
\quad
\text{for any } \tau_1<\tau_2.
$$
Consequently, the estimated conditional quantile function is nondecreasing in the quantile level, thereby preventing quantile crossing. The details of the partially monotonic prediction head are provided in Appendix~\ref{sec:monotonic_prediction_head}. We refer to this conditional quantile function estimation approach using a deep sequence encoder as DSCP-QFE throughout the paper.

\subsection{Localized Conformal Prediction using Deep Sequence Models}

Localized conformal prediction \citep{guan2023localized} assigns greater weight to calibration samples that are similar to the test sample, rather than weighting all calibration samples equally. Building on this idea for time series prediction setup, we use a sequence model to encode historical contexts into latent representations and compute similarity-based calibration weights in the learned representation space. These weights are then used to estimate the conditional cumulative distribution function (CDF) of the prediction residuals. 

We train the encoder using quantile losses at multiple fixed quantile levels. This encourages the learned representation to retain information about the context for conditional distribution of the residuals. Consequently, contexts with similar conditional distribution are expected to exhibit similar learned representations, making distances in the representation space informative for identifying relevant calibration samples. The resulting representations therefore can be used to construct similarity-based calibration weights, assigning greater weight to calibration contexts that are closer to the test context in the learned representation space.


For this method, the latent representation serves as a learned state variable for local calibration.
For an $\ell$-step ahead prediction, our goal is to estimate the true conditional CDF of the prediction residual,
\[
F^H_{\epshat}(u \mid h) \coloneqq \P \left(\epshat_{t+\ell} \le u \mid H_t = h \right).
\]
Given a calibration set $\{(H_i,\epshat_{i+\ell})\}_{i=1}^T$, nonnegative kernel $K$, bandwidth $b_T > 0$, and representation metric $d_H$, we define the weighted empirical CDF $\FhatH (\cdot \mid h)$ at a query representation $h$ as
\begin{align*}
K_{T,i}(h) &= K(d_H(H_i, h) / b_T),&
S_T(h) &= \sum_{i=1}^T K_{T,i}(h),\\
W_{T,i}(h) &= K_{T,i}(h) / S_T(h),&
\FhatH (u \mid h) &= \sum_{i=1}^T W_{T,i}(h) \ind\{\epshat_{i+\ell} \le u\},
\end{align*}
when $S_T(h) > 0$; otherwise set $\FhatH(u \mid h) = \ind\{u \ge 0\}$.

We obtain quantiles by inversion of this weighted empirical CDF:
\begin{align*}
\widehat q^H_{T,\tau}(h) &= \inf\{u: \FhatH(u \mid h) \ge \tau\},\\
\widehat q_{T,\tau}(z) &= \widehat q^H_{T,\tau}(\operatorname{Enc}_{\theta'}(z)),\qquad
&\widehat F_{T,\epshat}(u \mid z) = \FhatH(u \mid \operatorname{Enc}_{\theta'}(z)).
\end{align*}
Here, the inverse is simply computed by sorting the calibration residuals and accumulating their weights. It is nondecreasing in $\tau$, so the resulting interval \eqref{eq:common-interval} has ordered endpoints and quantile crossing will not occur. We refer to this localized conformal prediction using a deep sequence encoder as DSCP-Local throughout the paper.



\section{Theoretical Analysis}
\label{sec:theory}

Our analysis establishes asymptotic conditional coverage through quantile consistency for DSCP-QR and DSCP-QFE, and conditional CDF consistency for DSCP-Local. Throughout, $\widehat f$ is fixed, $t_T \ge T+\ell$ is deterministic, and $\alpha_l, \alpha_h$ are fixed with $\alpha_h - \alpha_l = 1 - \alpha$. Write $W_t = (Z_t, \epshat_{t+\ell})$, $\Zc = \operatorname{supp}(P_Z)$, and $q_\tau^\star(z)$ for the true conditional residual quantile. The full assumptions and proofs are provided in Appendix~\ref{app:proofs}.


\subsection{Asymptotic Conditional Coverage for DSCP-QR}\label{subsec:theory_QR}

DSCP-QR fits all target quantiles through a shared encoder. Therefore, we require a single encoder--head configuration to approximate the target quantiles jointly and training to attain the joint empirical minimum up to an $o_P(1)$ error. To relate this training criterion to population loss under temporal dependence, we assume $W_t$ is stationary and $\alpha$-mixing. These requirements and the remaining regularity conditions are stated precisely in Assumptions \ref{assump:stationary_mixing_process}--\ref{assump:joint_empirical_risk_minimization}.

\begin{theorem}[Conditional coverage for DSCP-QR]\label{thm:main-quantile-coverage}\label{thm:main-qr-coverage}
Under Assumptions \ref{assump:stationary_mixing_process}--\ref{assump:joint_empirical_risk_minimization}, the DSCP-QR interval~\eqref{eq:common-interval}, with $\alpha_l,\alpha_h \in \Tc$, satisfies
\begin{equation}\label{eq:main-quantile-coverage}
    \P\left(Y_{t_T+\ell}\in\widehat C_T^{\alpha,\ell}(Z_{t_T},\widehat Y_{t_T+\ell})\mid Z_{t_T}\right) = 1 - \alpha + o_P(1).
\end{equation}
\end{theorem}

\textit{Proof sketch.}
    For the quantile loss $\rho_\tau(u)=u(\tau-\ind\{u<0\})$, fixed weights $w_\tau>0$ with $\sum_{\tau\in\Tc}w_\tau=1$, and $\qv=(q_\tau)_{\tau\in\Tc}$, define the expected quantile losses
    \[
        R_\tau(q) = \E \left[\rho_\tau\{\epshat_{t+\ell}-q(Z_t)\}\right],\quad
        \Rc(\qv) = \sum_{\tau\in\Tc} w_\tau R_\tau(q_\tau).
    \]
    Let $\widehat{\Rc}_T$ be the empirical risk obtained by replacing the expectations in $R_\tau$ with averages over $\{(Z_i,\epshat_{i+\ell})\}_{i=1}^T$, and let $r_T$ be the optimization error in Assumption~\ref{assump:joint_empirical_risk_minimization}. Write $\widehat{\qv}_T = (\widehat q_{T,\tau})_{\tau\in\Tc}$, $\qv^\star = (q_\tau^\star)_{\tau\in\Tc}$, and let $\mathfrak Q_T$ be the jointly feasible model class. Comparing the training loss $\widehat{\Rc}_T$ with $\Rc$ gives
    \[
        \begin{aligned}
        0 &\le\Rc(\widehat{\qv}_T)-\Rc(\qv^\star)\\
        &\le 2\underbrace{\sup_{\qv\in\mathfrak Q_T}|\widehat{\Rc}_T(\qv)-\Rc(\qv)|}_{\text{empirical--population error}}
        \quad+\underbrace{\inf_{\qv\in\mathfrak Q_T}\Rc(\qv)-\Rc(\qv^\star)}_{\text{joint approximation error}}
        +\underbrace{r_T}_{\text{optimization error}} \xrightarrow{p} 0.
        \end{aligned}
    \]
    The first term vanishes by mixing and compactness; the remaining terms vanish by assumption. Thus the fitted quantiles achieve the same limiting expected loss as the true quantiles. The density and compactness conditions ensure that any fixed uniform quantile error entails a strictly positive loss gap, so
    \[
        \max_{\tau\in\Tc}\sup_{z\in\Zc}|\widehat q_{T,\tau}(z)-q_\tau^\star(z)| \xrightarrow{p} 0.
    \]
    Let $\widehat A_T$ and $A_T^\star$ be the coverage events for the fitted and true quantile endpoints. Stationarity and CDF continuity at the true endpoints give $\P(\widehat A_T\triangle A_T^\star)\to0$. Since $\P(A_T^\star\mid Z_{t_T})=1-\alpha$,
    \[
        \E\left|\P(\widehat A_T\mid Z_{t_T})-(1-\alpha)\right|
        \le\P(\widehat A_T\triangle A_T^\star) \longrightarrow 0,
    \]
    which implies \eqref{eq:main-quantile-coverage} (see Appendix \ref{subsec:theory-qr}).

\subsection{Asymptotic Conditional Coverage for DSCP-QFE}

DSCP-QFE learns one quantile function across levels sampled from $\nu$. As in DSCP-QR, we require accurate approximation and optimization, now for the loss averaged over these levels, under stationary $\alpha$-mixing data. The sampling distribution $\nu$ has full support on a fixed compact interval $\Ic \subset (0,1)$, so every neighborhood of a target level contributes to the objective. The full conditions are stated in Assumption \ref{assump:stationary_mixing_process} and Assumptions \ref{assump:qf-sampling}--\ref{assump:qf-optimization}.

\begin{theorem}[Conditional coverage for DSCP-QFE]\label{thm:main-qfe-coverage}
Under Assumption \ref{assump:stationary_mixing_process} and Assumptions \ref{assump:qf-sampling}--\ref{assump:qf-optimization}, the DSCP-QFE interval \eqref{eq:common-interval}, with $\alpha_l,\alpha_h \in \Ic$, satisfies~\eqref{eq:main-quantile-coverage}.
\end{theorem}

\textit{Proof sketch.}
    Write $\widehat q_T(z,\tau)=\widehat q_{T,\tau}(z)$, $q^\star(z,\tau)=q_\tau^\star(z)$ and $\Rc(q)=\int_{\Ic}R_\tau(q(\cdot,\tau))\nu(d\tau)$, with $R_\tau$ defined above. The same decomposition as DSCP-QR gives
    \[
        \begin{aligned}
        0 &\le \Rc(\widehat q_T)-\Rc(q^\star)\\
        &= \int_{\Ic}\underbrace{\big[R_\tau(\widehat q_T(\cdot,\tau))-R_\tau(q_\tau^\star)\big]}_{\text{additional expected loss at level }\tau \ge 0}\nu(d\tau)
        \xrightarrow{p} 0.
        \end{aligned}
    \]
    The local density condition identifies the true quantile at each level. Joint continuity and full support ensure that an error at any $(z,\tau)$ contributes positive loss on a neighborhood of that pair. Compactness therefore converts the displayed loss convergence into
    \[
        \sup_{z\in\Zc,\tau\in\Ic}|\widehat q_T(z,\tau)-q^\star(z,\tau)| \xrightarrow{p} 0.
    \]
    At the two fixed endpoint levels, the coverage-event comparison from DSCP-QR yields \eqref{eq:main-quantile-coverage} (see Appendix \ref{subsec:theory-qf}).

\subsection{Asymptotic Conditional Coverage for DSCP-Local}
\label{subsec:theory_LCP}

For DSCP-Local, the encoder is fixed during calibration and inference, and the weighted residual CDF targets $\FH(u\mid h)=\P(\epshat_{t+\ell}\le u\mid H_t=h)$. The primary coverage guarantee is therefore conditional on the representation $H_t$.

We assume $V_t = (H_t,\epshat_{t+\ell})$ is strictly stationary and $\beta$-mixing. Local estimation requires continuous conditional residual CDFs that vary little among nearby representations, and sufficient probability mass in these neighborhoods. With a bounded, compactly supported kernel, the bandwidth shrinks to eliminate smoothing bias while satisfying the dependence-adjusted variance condition. The precise conditions are stated in Assumptions \ref{ass:H1}--\ref{ass:H4}.

\begin{theorem}[Representation-conditional consistency and coverage]\label{thm:main-local-coverage}
Under Assumptions \ref{ass:H1}--\ref{ass:H4}, the DSCP-Local estimator with a fixed encoder satisfies
\begin{align}
\sup_u|\FhatH(u\mid H_{t_T})-\FH(u\mid H_{t_T})|&\xrightarrow{p}0,\label{eq:main-H-cdf}\\
\P\left(Y_{t_T+\ell}\in\widehat C_T^{\alpha,\ell}(Z_{t_T},\widehat Y_{t_T+\ell})\mid H_{t_T}\right)
&= 1-\alpha+o_P(1).\label{eq:main-H-coverage}
\end{align}
\end{theorem}

\textit{Proof sketch.}
    Local mass, CDF regularity, and the bandwidth condition yield uniform CDF consistency for $P_H$-almost every fixed representation. A $\beta$-mixing coupling argument extends consistency to the dependent test representation $H_{t_T}$. The estimated quantiles are then bracketed by oracle quantiles at neighboring probability levels, which gives \eqref{eq:main-H-coverage} (see Appendix \ref{subsubsec:H-conditional}).

Extension to context-conditional coverage requires residual sufficiency: the conditional residual distribution given $Z_t$ depends on $Z_t$ only through $H_t$. This condition permits dimension reduction while preserving the conditional distribution relevant to coverage. Detailed discussion can be found in Appendix \ref{subsubsec:Z-conditional}.

\begin{corollary}[Context-conditional coverage for DSCP-Local]\label{cor:main-Z-coverage}
Under the conditions of Theorem~\ref{thm:main-local-coverage}, additionally suppose $W_t$ is strictly stationary and, for every $u$ and $P_Z$-almost every $z$,
\begin{equation}\label{eq:main-residual-sufficiency}
\P(\epshat_{t+\ell}\le u\mid Z_t=z)=\FH(u\mid\operatorname{Enc}_{\theta'}(z)).
\end{equation}
Then the DSCP-Local interval also satisfies \eqref{eq:main-quantile-coverage}.
\end{corollary}

\textit{Proof sketch.}
    Let $D_T = \sup_u|\FhatH(u \mid H_{t_T}) - \FH(u \mid H_{t_T})|$ and let $\widehat A_T$ denote the DSCP-Local coverage event. By \eqref{eq:main-residual-sufficiency} and CDF continuity, an oracle interval between levels $a<b$ has probability $b-a$ conditional on $Z_{t_T}$. On $\{D_T\le\eta\}$, each fitted endpoint at level $\tau$ lies between the oracle quantiles at levels $\tau-\eta$ and $\tau+\eta$. Thus, for any sufficiently small fixed $\eta>0$,
    \[
        \begin{aligned}
        \E|\P(\widehat A_T\mid Z_{t_T})-(1-\alpha)|
        &\le 2\eta+\P(D_T>\eta).
        \end{aligned}
    \]
    Since $D_T \xrightarrow{p} 0$ by \eqref{eq:main-H-cdf}, taking $T \to \infty$ and then $\eta \downarrow 0$ gives convergence in mean, hence in probability (see Appendix \ref{subsubsec:Z-conditional}).

\begin{table}[!h]
\centering
\caption{Coverage gap ($\Delta$Cov), prediction interval width (PI-W), and Winkler score (WS), averaged across sequences within each dataset, for all methods. LR, LSTM, and Chronos indicate the base predictor. The target coverage is $1-\alpha=0.9$. Results with undercoverage are color-coded according to the magnitude of the coverage gap: results with
$-0.5\alpha \leq \Delta\mathrm{Cov} < -0.25\alpha$ are shown in {\color{mildUC}red}, while those with
$\Delta\mathrm{Cov} < -0.5\alpha$ are {\color{severeUC}greyed out}. The best three results except the undercoveraged results based on Winkler score are \underline{underlined}.}
\label{table:main_results_1}

\begin{adjustbox}{max width=\linewidth}
\begin{tabular}{lccccccccc}
\toprule
& \multicolumn{9}{c}{Air} \\
\cmidrule(lr){2-10}
Method
& \multicolumn{3}{c}{LR}
& \multicolumn{3}{c}{LSTM}
& \multicolumn{3}{c}{Chronos} \\
\cmidrule(lr){2-4}\cmidrule(lr){5-7}\cmidrule(lr){8-10}
& $\Delta$Cov & PI-W & WS
& $\Delta$Cov & PI-W & WS
& $\Delta$Cov & PI-W & WS \\

\midrule
Split CP    & $+0.028_{\scriptscriptstyle \pm .011}$ & $100.5_{\scriptscriptstyle \pm 11.1}$ & $150.0_{\scriptscriptstyle \pm 14.6}$ & $-0.005_{\scriptscriptstyle \pm .014}$ & $73.9_{\scriptscriptstyle \pm 8.21}$ & $137.2_{\scriptscriptstyle \pm 15.1}$ & $+0.011_{\scriptscriptstyle \pm .010}$ & $187.7_{\scriptscriptstyle \pm 17.2}$ & $315.6_{\scriptscriptstyle \pm 21.8}$ \\
NexCP       & $-0.001_{\scriptscriptstyle \pm .001}$ & $86.8_{\scriptscriptstyle \pm 10.6}$ & $141.9_{\scriptscriptstyle \pm 15.4}$ & $-0.004_{\scriptscriptstyle \pm .001}$ & $76.1_{\scriptscriptstyle \pm 9.10}$ & $133.6_{\scriptscriptstyle \pm 14.9}$ & $-0.005_{\scriptscriptstyle \pm .001}$ & $180.2_{\scriptscriptstyle \pm 14.2}$ & $286.0_{\scriptscriptstyle \pm 18.9}$ \\
HopCPT      & $\textcolor{mildUC}{-0.027_{\scriptscriptstyle \pm .032}}$ & $\textcolor{mildUC}{79.1_{\scriptscriptstyle \pm 14.9}}$ & 
$\textcolor{mildUC}{148.5_{\scriptscriptstyle \pm 22.7}}$ & 
$\textcolor{mildUC}{-0.030_{\scriptscriptstyle \pm .028}}$ & 
$\textcolor{mildUC}{67.1_{\scriptscriptstyle \pm 15.7}}$ & 
$\textcolor{mildUC}{134.1_{\scriptscriptstyle \pm 15.6}}$ & 
$-0.010_{\scriptscriptstyle \pm .038}$ & 
$221.0_{\scriptscriptstyle \pm 84.5}$ & 
$347.3_{\scriptscriptstyle \pm 101.2}$ \\

SPCI
& $\textcolor{severeUC}{-0.110_{\scriptscriptstyle \pm .006}}$
& $\textcolor{severeUC}{75.9_{\scriptscriptstyle \pm 8.86}}$
& $\textcolor{severeUC}{160.7_{\scriptscriptstyle \pm 16.9}}$
& $\textcolor{severeUC}{-0.110_{\scriptscriptstyle \pm .006}}$
& $\textcolor{severeUC}{64.2_{\scriptscriptstyle \pm 7.89}}$
& $\textcolor{severeUC}{140.5_{\scriptscriptstyle \pm 16.5}}$
& $\textcolor{severeUC}{-0.088_{\scriptscriptstyle \pm .009}}$
& $\textcolor{severeUC}{81.9_{\scriptscriptstyle \pm 7.31}}$
& $\textcolor{severeUC}{160.5_{\scriptscriptstyle \pm 13.3}}$ \\

KOWCPI      & $+0.012_{\scriptscriptstyle \pm .009}$ & $92.8_{\scriptscriptstyle \pm 10.1}$ & $150.2_{\scriptscriptstyle \pm 14.6}$ & $-0.003_{\scriptscriptstyle \pm .010}$ & $73.3_{\scriptscriptstyle \pm 7.81}$ & $137.1_{\scriptscriptstyle \pm 14.5}$ & $+0.005_{\scriptscriptstyle \pm .009}$ & $181.2_{\scriptscriptstyle \pm 15.5}$ & $316.0_{\scriptscriptstyle \pm 22.0}$ \\
ResCP       & $-0.012_{\scriptscriptstyle \pm .005}$ & $80.6_{\scriptscriptstyle \pm 10.3}$ & $\underline{135.3}_{\scriptscriptstyle \pm 14.3}$ & $-0.019_{\scriptscriptstyle \pm .006}$ & $67.6_{\scriptscriptstyle \pm 9.86}$ & $\underline{119.8}_{\scriptscriptstyle \pm 13.4}$ & $+0.033_{\scriptscriptstyle \pm .003}$ & $126.6_{\scriptscriptstyle \pm 10.4}$ & $171.4_{\scriptscriptstyle \pm 12.6}$ \\

DistMatch   & $-0.004_{\scriptscriptstyle \pm .006}$ & $83.8_{\scriptscriptstyle \pm 9.26}$ & $139.1_{\scriptscriptstyle \pm 14.9}$ & $-0.005_{\scriptscriptstyle \pm .005}$ & $71.3_{\scriptscriptstyle \pm 7.68}$ & $125.3_{\scriptscriptstyle \pm 13.4}$ & $+0.039_{\scriptscriptstyle \pm .004}$ & $120.5_{\scriptscriptstyle \pm 10.3}$ & $\underline{169.0}_{\scriptscriptstyle \pm 13.4}$ \\
  
DSCP-QR     & $+0.012_{\scriptscriptstyle \pm .016}$ & $85.3_{\scriptscriptstyle \pm 10.5}$ & $\underline{121.3}_{\scriptscriptstyle \pm 12.5}$ & $-0.021_{\scriptscriptstyle \pm .017}$ & $61.1_{\scriptscriptstyle \pm 9.32}$ & $\underline{103.4}_{\scriptscriptstyle \pm 12.1}$ & $+0.014_{\scriptscriptstyle \pm .014}$ & $87.1_{\scriptscriptstyle \pm 10.7}$ & $\underline{125.5}_{\scriptscriptstyle \pm 10.6}$ \\

DSCP-QFE
& $\textcolor{severeUC}{-0.064_{\scriptscriptstyle \pm .015}}$
& $\textcolor{severeUC}{61.3_{\scriptscriptstyle \pm 6.40}}$
& $\textcolor{severeUC}{132.7_{\scriptscriptstyle \pm 15.6}}$
& $\textcolor{severeUC}{-0.126_{\scriptscriptstyle \pm .026}}$
& $\textcolor{severeUC}{37.9_{\scriptscriptstyle \pm 4.73}}$
& $\textcolor{severeUC}{124.4_{\scriptscriptstyle \pm 16.3}}$
& $\textcolor{severeUC}{-0.091_{\scriptscriptstyle \pm .032}}$
& $\textcolor{severeUC}{52.5_{\scriptscriptstyle \pm 7.39}}$
& $\textcolor{severeUC}{170.1_{\scriptscriptstyle \pm 13.1}}$ \\

DSCP-Local & $-0.015_{\scriptscriptstyle \pm .003}$ & $81.0_{\scriptscriptstyle \pm 10.0}$ & $\underline{133.2}_{\scriptscriptstyle \pm 14.4}$ & $-0.013_{\scriptscriptstyle \pm .005}$ & $67.8_{\scriptscriptstyle \pm 8.67}$ & $\underline{116.0}_{\scriptscriptstyle \pm 13.3}$ & $+0.027_{\scriptscriptstyle \pm .006}$ & $117.4_{\scriptscriptstyle \pm 9.14}$ & $\underline{160.7}_{\scriptscriptstyle \pm 11.4}$ \\

\midrule
& \multicolumn{9}{c}{Solar} \\
\cmidrule(lr){2-10}
Method
& \multicolumn{3}{c}{LR}
& \multicolumn{3}{c}{LSTM}
& \multicolumn{3}{c}{Chronos} \\
\cmidrule(lr){2-4}\cmidrule(lr){5-7}\cmidrule(lr){8-10}
& $\Delta$Cov & PI-W & WS
& $\Delta$Cov & PI-W & WS
& $\Delta$Cov & PI-W & WS \\

\midrule
Split CP    & $-0.001_{\scriptscriptstyle \pm .024}$ & $115.4_{\scriptscriptstyle \pm 85.3}$ & $202.6_{\scriptscriptstyle \pm 81.1}$ & $-0.002_{\scriptscriptstyle \pm .020}$ & $88.1_{\scriptscriptstyle \pm 30.2}$ & $176.3_{\scriptscriptstyle \pm 33.0}$ & $+0.008_{\scriptscriptstyle \pm .010}$ & $165.9_{\scriptscriptstyle \pm 23.0}$ & $270.4_{\scriptscriptstyle \pm 26.9}$ \\
NexCP       & $+0.001_{\scriptscriptstyle \pm .002}$ & $111.8_{\scriptscriptstyle \pm 22.6}$ & $181.0_{\scriptscriptstyle \pm 32.0}$ & $-0.002_{\scriptscriptstyle \pm .002}$ & $89.6_{\scriptscriptstyle \pm 21.5}$ & $161.5_{\scriptscriptstyle \pm 29.3}$ & $-0.001_{\scriptscriptstyle \pm .001}$ & $158.2_{\scriptscriptstyle \pm 24.7}$ & $249.5_{\scriptscriptstyle \pm 30.4}$ \\
HopCPT      & $+0.001_{\scriptscriptstyle \pm .031}$ & $106.1_{\scriptscriptstyle \pm 25,6}$ & $\underline{148.4}_{\scriptscriptstyle \pm 37.7}$ & $+0.007_{\scriptscriptstyle \pm .035}$ & $76.7_{\scriptscriptstyle \pm 24.0}$ & $123.6_{\scriptscriptstyle \pm 40.5}$ & $+0.004_{\scriptscriptstyle \pm .035}$ & $114.1_{\scriptscriptstyle \pm 25.9}$ & $179.2_{\scriptscriptstyle \pm 49.3}$ \\

SPCI
& $\textcolor{severeUC}{-0.106_{\scriptscriptstyle \pm .028}}$
& $\textcolor{severeUC}{77.5_{\scriptscriptstyle \pm 13.2}}$
& $\textcolor{severeUC}{155.3_{\scriptscriptstyle \pm 24.1}}$
& $\textcolor{severeUC}{-0.087_{\scriptscriptstyle \pm .067}}$
& $\textcolor{severeUC}{56.0_{\scriptscriptstyle \pm 12.8}}$
& $\textcolor{severeUC}{120.7_{\scriptscriptstyle \pm 23.6}}$
& $\textcolor{severeUC}{-0.077_{\scriptscriptstyle \pm .016}}$
& $\textcolor{severeUC}{60.0_{\scriptscriptstyle \pm 8.97}}$
& $\textcolor{severeUC}{119.8_{\scriptscriptstyle \pm 16.5}}$ \\

KOWCPI      & $+0.004_{\scriptscriptstyle \pm .010}$ & $114.9_{\scriptscriptstyle \pm 48.3}$ & $196.3_{\scriptscriptstyle \pm 59.2}$ & $+0.005_{\scriptscriptstyle \pm .011}$ & $91.2_{\scriptscriptstyle \pm 22.5}$ & $172.6_{\scriptscriptstyle \pm 29.9}$ & $-0.010_{\scriptscriptstyle \pm .014}$ & $154.4_{\scriptscriptstyle \pm 28.5}$ & $263.9_{\scriptscriptstyle \pm 30.4}$ \\
ResCP       & $\textcolor{mildUC}{-0.026_{\scriptscriptstyle \pm .007}}$ & $\textcolor{mildUC}{85.4_{\scriptscriptstyle \pm 19.0}}$ & 
$\textcolor{mildUC}{158.2_{\scriptscriptstyle \pm 31.4}}$ & $+0.008_{\scriptscriptstyle \pm .005}$ & $67.0_{\scriptscriptstyle \pm 14.3}$ & $\underline{107.4}_{\scriptscriptstyle \pm 21.5}$ & $+0.011_{\scriptscriptstyle \pm .004}$ & $86.5_{\scriptscriptstyle \pm 13.7}$ & $\underline{130.3}_{\scriptscriptstyle \pm 17.0}$ \\
DistMatch   & $+0.007_{\scriptscriptstyle \pm .008}$ & $100.3_{\scriptscriptstyle \pm 20.1}$ & $168.6_{\scriptscriptstyle \pm 31.4}$ & $+0.014_{\scriptscriptstyle \pm .009}$ & $81.1_{\scriptscriptstyle \pm 17.3}$ & $145.2_{\scriptscriptstyle \pm 25.7}$ & $+0.024_{\scriptscriptstyle \pm .006}$ & $98.7_{\scriptscriptstyle \pm 15.9}$ & $153.2_{\scriptscriptstyle \pm 20.9}$ \\

DSCP-QR   & $-0.019_{\scriptscriptstyle \pm .107}$ & $81.9_{\scriptscriptstyle \pm 14.6}$ & $\underline{121.6}_{\scriptscriptstyle \pm 82.9}$ & $-0.006_{\scriptscriptstyle \pm .120}$ & $55.4_{\scriptscriptstyle \pm 10.8}$ & $\underline{84.0}_{\scriptscriptstyle \pm 84.0}$ & $+0.007_{\scriptscriptstyle \pm .020}$ & $62.9_{\scriptscriptstyle \pm 9.97}$ & $\underline{82.4}_{\scriptscriptstyle \pm 11.5}$ \\ 

DSCP-QFE
& $\textcolor{severeUC}{-0.114_{\scriptscriptstyle \pm .019}}$
& $\textcolor{severeUC}{55.9_{\scriptscriptstyle \pm 9.50}}$
& $\textcolor{severeUC}{159.7_{\scriptscriptstyle \pm 41.8}}$
& $\textcolor{severeUC}{-0.120_{\scriptscriptstyle \pm .061}}$
& $\textcolor{severeUC}{33.4_{\scriptscriptstyle \pm 8.74}}$
& $\textcolor{severeUC}{122.8_{\scriptscriptstyle \pm 37.2}}$
& $\textcolor{severeUC}{-0.103_{\scriptscriptstyle \pm .035}}$
& $\textcolor{severeUC}{42.1_{\scriptscriptstyle \pm 8.71}}$
& $\textcolor{severeUC}{123.3_{\scriptscriptstyle \pm 15.6}}$ \\

DSCP-Local  & $-0.001_{\scriptscriptstyle \pm .008}$ & $88.6_{\scriptscriptstyle \pm 18.4}$ & $\underline{123.8}_{\scriptscriptstyle \pm 27.6}$ & $+0.016_{\scriptscriptstyle \pm .010}$ & $64.0_{\scriptscriptstyle \pm 15.3}$ & $\underline{87.0}_{\scriptscriptstyle \pm 17.6}$ & $+0.011_{\scriptscriptstyle \pm .007}$ & $83.1_{\scriptscriptstyle \pm 14.3}$ & $\underline{109.9}_{\scriptscriptstyle \pm 13.2}$ \\

\midrule
& \multicolumn{9}{c}{Sapflow} \\
\cmidrule(lr){2-10}
Method
& \multicolumn{3}{c}{LR}
& \multicolumn{3}{c}{LSTM}
& \multicolumn{3}{c}{Chronos} \\
\cmidrule(lr){2-4}\cmidrule(lr){5-7}\cmidrule(lr){8-10}
& $\Delta$Cov & PI-W & WS
& $\Delta$Cov & PI-W & WS
& $\Delta$Cov & PI-W & WS \\

\midrule
Split CP    & $+0.032_{\scriptscriptstyle \pm .027}$ & $803.5_{\scriptscriptstyle \pm 577.6}$ & $1197.0_{\scriptscriptstyle \pm 839.7}$ & $+0.022_{\scriptscriptstyle \pm .022}$ & $1318.4_{\scriptscriptstyle \pm 2199.0}$ & $2045.5_{\scriptscriptstyle \pm 3441.5}$ & $+0.004_{\scriptscriptstyle \pm .035}$ & $2380.8_{\scriptscriptstyle \pm 1929.1}$ & $3854.7_{\scriptscriptstyle \pm 2892.2}$ \\
NexCP       & $+0.002_{\scriptscriptstyle \pm .005}$ & $656.4_{\scriptscriptstyle \pm 427.6}$ & $1067.9_{\scriptscriptstyle \pm 731.7}$ & $-0.002_{\scriptscriptstyle \pm .005}$ & $1179.7_{\scriptscriptstyle \pm 1774.7}$ & $1722.0_{\scriptscriptstyle \pm 2626.9}$ & $-0.002_{\scriptscriptstyle \pm .005}$ & $2154.9_{\scriptscriptstyle \pm 1774.7}$ & $3602.8_{\scriptscriptstyle \pm 2810.5}$ \\
HopCPT      & $+0.015_{\scriptscriptstyle \pm .059}$ & $790.8_{\scriptscriptstyle \pm 619.4}$ & $1101.1_{\scriptscriptstyle \pm 777.9}$ & $-0.016_{\scriptscriptstyle \pm .075}$ & $886.7_{\scriptscriptstyle \pm 927.5}$ & $1311.9_{\scriptscriptstyle \pm 1127.4}$ & $-0.001_{\scriptscriptstyle \pm .069}$ & $2313.5_{\scriptscriptstyle \pm 1832.1}$ & $3088.5_{\scriptscriptstyle \pm 1985.2}$ \\
SPCI        & $+0.044_{\scriptscriptstyle \pm .022}$ & $704.6_{\scriptscriptstyle \pm 484.2}$ & $\underline{979.4}_{\scriptscriptstyle \pm 667.9}$ & $+0.031_{\scriptscriptstyle \pm .026}$ & $650.5_{\scriptscriptstyle \pm 466.9}$ & $\underline{913.4}_{\scriptscriptstyle \pm 642.3}$ & $\textcolor{mildUC}{-0.032_{\scriptscriptstyle \pm .057}}$ & $\textcolor{mildUC}{793.0_{\scriptscriptstyle \pm 585.1}}$ & $\textcolor{mildUC}{1351.1_{\scriptscriptstyle \pm 942.4}}$ \\
KOWCPI      & $+0.021_{\scriptscriptstyle \pm .058}$ & $875.2_{\scriptscriptstyle \pm 705.2}$ & $1277.3_{\scriptscriptstyle \pm 895.5}$ & $+0.009_{\scriptscriptstyle \pm .065}$ & $977.2_{\scriptscriptstyle \pm 789.7}$ & $1429.3_{\scriptscriptstyle \pm 1029.9}$ & $\textcolor{mildUC}{-0.033_{\scriptscriptstyle \pm .076}}$ & $\textcolor{mildUC}{2024.9_{\scriptscriptstyle \pm 1730.2}}$ & $\textcolor{mildUC}{2773.7_{\scriptscriptstyle \pm 2044.1}}$ \\
ResCP       & $-0.004_{\scriptscriptstyle \pm .024}$ & $534.0_{\scriptscriptstyle \pm 361.2}$ & $\underline{905.4}_{\scriptscriptstyle \pm 634.6}$ & $-0.002_{\scriptscriptstyle \pm .024}$ & $593.1_{\scriptscriptstyle \pm 527.6}$ & $\underline{962.9}_{\scriptscriptstyle \pm 777.3}$ & $+0.021_{\scriptscriptstyle \pm .026}$ & $1373.2_{\scriptscriptstyle \pm 1111.2}$ & $\underline{1854.1}_{\scriptscriptstyle \pm 1309.8}$ \\
DistMatch   & $+0.024_{\scriptscriptstyle \pm .017}$ & $655.2_{\scriptscriptstyle \pm 454.0}$ & $988.3_{\scriptscriptstyle \pm 688.5}$ & $+0.031_{\scriptscriptstyle \pm .016}$ & $706.4_{\scriptscriptstyle \pm 528.5}$ & $1039.6_{\scriptscriptstyle \pm 754.9}$ & $+0.046_{\scriptscriptstyle \pm .019}$ & $1422.1_{\scriptscriptstyle \pm 1021.1}$ & $\underline{1868.7}_{\scriptscriptstyle \pm 1257.6}$ \\
DSCP-QR     & $+0.005_{\scriptscriptstyle \pm .064}$ & $708.5_{\scriptscriptstyle \pm 511.9}$ & $994.5_{\scriptscriptstyle \pm 687.8}$ & $-0.008_{\scriptscriptstyle \pm .111}$ & $839.3_{\scriptscriptstyle \pm 917.8}$ & $1142.3_{\scriptscriptstyle \pm 1023.3}$ & $\textcolor{severeUC}{-0.061_{\scriptscriptstyle \pm .065}}$ & $\textcolor{severeUC}{864.8_{\scriptscriptstyle \pm 549.1}}$ & $\textcolor{severeUC}{1450.2_{\scriptscriptstyle \pm 953.0}}$ \\
DSCP-QFE
& $\textcolor{mildUC}{-0.046_{\scriptscriptstyle \pm .116}}$
& $\textcolor{mildUC}{609.6_{\scriptscriptstyle \pm 421.5}}$
& $\textcolor{mildUC}{1209.8_{\scriptscriptstyle \pm 829.9}}$
& $\textcolor{severeUC}{-0.115_{\scriptscriptstyle \pm .128}}$
& $\textcolor{severeUC}{527.0_{\scriptscriptstyle \pm 461.3}}$
& $\textcolor{severeUC}{1188.3_{\scriptscriptstyle \pm 884.6}}$
& $\textcolor{severeUC}{-0.137_{\scriptscriptstyle \pm .101}}$
& $\textcolor{severeUC}{563.4_{\scriptscriptstyle \pm 442.8}}$
& $\textcolor{severeUC}{1885.1_{\scriptscriptstyle \pm 1216.4}}$ \\
DSCP-Local   & $+0.001_{\scriptscriptstyle \pm .012}$ & $558.2_{\scriptscriptstyle \pm 375.5}$ & $\underline{864.2}_{\scriptscriptstyle \pm 619.6}$ & $+0.007_{\scriptscriptstyle \pm .013}$ & $632.1_{\scriptscriptstyle \pm 550.7}$ & $\underline{911.4}_{\scriptscriptstyle \pm 743.6}$ & $+0.017_{\scriptscriptstyle \pm .010}$ & $1400.5_{\scriptscriptstyle \pm 1102.3}$ & $\underline{1860.4}_{\scriptscriptstyle \pm 1382.5}$ \\

\bottomrule
\end{tabular}
\end{adjustbox}
\end{table}

\section{Experiments}

\subsection{Setup}

\paragraph{Datasets.}

We evaluate the proposed methods and baselines on three real-world datasets that are commonly used as benchmarks in conformal prediction for time-series. The Air Quality dataset~\citep{zhang2017cautionary} contains air quality measurements and 10 associated features collected from 12 locations in Beijing, China. The Solar dataset~\citep{sengupta2018national} contains half-hourly measurements of solar radiation along with eight environmental features at 50 different locations. The Sapflow dataset, obtained from the SAPFLUXNET database~\citep{poyatos2020global}, contains measurements of water movement within plants. Since the individual sequences in the Sapflow dataset vary in length, we use a subset of 24 time series, each containing between 15,000 and 20,000 observations with varying sampling frequencies, following~\citet{auer2024conformal}. Additional details of the datasets are provided in Appendix~\ref{appendix:datasets}.

\paragraph{Base predictors.}
We consider three base predictors: linear regression, a recurrent neural network (RNN), and Chronos~\citep{ansari2024chronos,ansari2025chronos}. For linear regression, we use a sliding window of the past $k$ time steps of features and outcomes to predict the outcome at the next time step, and train a separate model for each individual time-series sequence. For the RNN, we use LSTM cells to perform the same one-step-ahead prediction using the past $k$ time steps of features and outcomes. Unlike linear regression, a single RNN model is trained jointly using the training splits of all time series sequences. Since Chronos is a pretrained foundation model for time series forecasting, we use its zero-shot forecasts without fine tuning. Additional details on the base predictors are provided in Appendix~\ref{appendix:base_predictors}.

\paragraph{Proposed and baseline methods.}

For the three proposed methods, we consider LSTM and Transformer as the sequence encoder. Based on the literature review in \S~\ref{sec:literature_review}, we consider several representative conformal prediction methods as baseline methods, with an emphasis on methods developed for time series or distribution shift: NexCP~\citep{barber2023conformal}, SPCI~\citep{xu2023sequential}, HopCPT~\citep{auer2024conformal}, KOWCPI~\citep{lee2024kernel}, ResCP~\citep{neglia2026rescp}, and DistMatch~\citep{menadjiev2026distmatch}. For each time-series sequence, the first 33\% is used as the training set. The first 67\% of the remaining sequence are used for calibration, and the rest are used for testing, resulting in an approximate 33:45:22 split. The training set is used to train the base predictor, the calibration set is used to calibrate the prediction intervals, and the test set is used for evaluation. We consider a one-step-ahead prediction setting where the prediction interval for time $t+1$ is constructed using the historical context at time $t$. Implementation details and hyperparameter selection procedures for the proposed methods and baselines are provided in Appendices~\ref{appendix:proposed_methods} and~\ref{appendix:baselines}, respectively.

\paragraph{Evaluation metrics.}

We evaluate all methods using three metrics: coverage gap, prediction interval width, and Winkler score~\citep{winkler1972decision}. The coverage gap measures the difference between empirical and target coverage. Prediction interval width measures the size of the prediction interval. The Winkler score jointly accounts for interval width and coverage violations. Each metric is computed over the test set and then averaged across sequences within each dataset. Detailed definitions are provided in Appendix~\ref{appendix:eval_metrics}.

\subsection{Results}
\paragraph{Deep sequence models provide efficient prediction intervals for time series.}

Table~\ref{table:main_results_1} reports the average coverage gap, interval width, and Winkler score for all methods across three datasets at a target coverage level of $0.9$. The results for Split CP and NexCP suggest that relying on exchangeability or simply assigning greater weights to recent observations may be insufficient to construct efficient prediction intervals for time series. For each proposed method, we report the best results obtained using either an RNN or a Transformer encoder. DSCP-Local achieves the best overall performance while maintaining stable coverage across all three datasets. DSCP-QR also performs well, often outperforming the strongest baselines, although it occasionally exhibits undercoverage. Similar trends hold at higher and lower target coverage levels of $0.95$ and $0.8$, as shown in Tables~\ref{table:main_results_2} and~\ref{table:main_results_3}, respectively.

We observe that the RNN encoder generally outperforms the Transformer encoder in our experiments. One possible explanation is that perhaps relatively short-term temporal dependencies are sufficient for modeling the conditional residual distributions in these settings, although further analysis is needed to assess this hypothesis. DSCP-QFE exhibits undercoverage, possibly because its partially monotonic prediction heads limit model expressiveness. The ablation study below, in which this architectural constraint is removed, provides empirical support for this interpretation.

\paragraph{The effect of architectural constraints for preventing quantile crossing on prediction interval efficiency.}

Table~\ref{table:ablation_crossing} reports ablation results for DSCP-QR and DSCP-QFE without their respective architectural constraints for preventing quantile crossing. For DSCP-QR, removing the incremental prediction heads has marginal effect on performance but leads to quantile crossing. In contrast, removing the partially monotonic heads substantially improves the performance of DSCP-QFE.

\section{Conclusion}

In this study, we propose three methods that leverage deep sequence models for conformal prediction in time series. We establish asymptotic conditional coverage guarantees for all three methods under suitable assumptions. Experiments on three real-world datasets against a broad set of baseline methods demonstrate the effectiveness of deep sequence models for conformal prediction for time series. We believe that our theoretical analysis and empirical findings provide useful insights and practical guidance for future research. Future work should address the limitations of this study by developing experimental protocols to assess conditional coverage empirically and designing architectures that prevent quantile crossing without compromising prediction interval efficiency. Another promising direction is to develop novel methods that more effectively leverage deep sequence models for time-series conformal prediction.


\section*{Acknowledgments}

This work is partially supported by NSF CMMI-2112533 and the Coca-Cola Foundation.

\bibliography{refs}

\appendix

\section{Detailed Theoretical Analysis and Proofs}\label{app:proofs}

This appendix gives the full assumptions and intermediate results underlying Section \ref{sec:theory}.

Throughout this section, the point predictor $\widehat f$ is fixed. The encoder is learned jointly with the quantile heads in DSCP-QR and DSCP-QFE, and is fixed in DSCP-Local. For the stationary pair $(Z_t, \epshat_{t+\ell})$ and a quantile level $\tau \in (0,1)$, write
\[
F_{\epshat}(u \mid z) \coloneqq \P \left( \epshat_{t+\ell} \leq u \mid Z_t=z \right), \quad q^\star_\tau(z) \coloneqq \inf \left\{ u \in \R:
F_{\epshat}(u\mid z) \geq \tau \right\},
\]
as the conditional distribution of the prediction residual, and the corresponding true conditional $\tau$-quantile function. The forecast origins $t_T \ge T + \ell$ and endpoint levels $\alpha_l$, $\alpha_h$ are as in Section \ref{sec:problem_setup}.

\subsection{DSCP-QR}\label{subsec:theory-qr}

We first analyze DSCP-QR, which utilizes the finite collection of quantile heads trained with a shared encoder. The argument proceeds from uniform convergence of the training risk to uniform quantile consistency, and then to conditional coverage.

Let $\mathcal T = \{\tau_1,\ldots,\tau_J\} \subset (0,1)$ be fixed. The shared encoder and its heads induce a deterministic, nonempty joint class.
We collect the estimated conditional quantile functions into the vector
\(
\widehat{\mathbf q}_T = \left(\widehat q_{T,\tau} \right)_{\tau\in\mathcal{T}},
\)
and denote the corresponding vector of true conditional quantile functions by
\(
\mathbf q^\star = \left( q^\star_\tau \right)_{\tau\in\mathcal{T}}.
\)
For each $\tau\in\mathcal{T}$, let $\mathcal{Q}_{T,\tau}$ denote the candidate class of conditional quantile functions at level $\tau$. The shared-encoder architecture induces a joint class of vector-valued quantile functions
\[
\mathfrak{Q}_T \subseteq \prod_{\tau\in\mathcal{T}} \mathcal{Q}_{T,\tau}.
\]
By construction,
\(
\widehat{\mathbf q}_T\in\mathfrak{Q}_T.
\)

The conditional quantile functions are estimated using the quantile loss
\(
\rho_\tau(u) = u \left(\tau - \ind\{u<0\}\right),
\)
where $\ind\{\cdot\}$ denotes the indicator function. For a candidate function $q\in\mathcal{Q}_{T,\tau}$, define the empirical quantile risk and the corresponding population risk by
\begin{equation}
\widehat R_{T,\tau}(q) = \frac{1}{T} \sum_{i=1}^{T} \rho_\tau ( \epshat_{i+\ell} - q(Z_i)), \quad
R_\tau(q) = \E\left[\rho_\tau(\epshat_{t+\ell} - q(Z_t) ) \right].
\end{equation}
For a generic vector
\(
\qv = \left(q_\tau\right)_{\tau\in\Tc},
\)
define the joint empirical risk and the corresponding joint population risk by
\[
\widehat{\Rc}_T(\qv) = \sum_{\tau\in\Tc} w_\tau \widehat R_{T,\tau}(q_\tau),\quad
\Rc(\qv) = \sum_{\tau\in\Tc} w_\tau R_\tau (q_\tau),
\]
where the fixed weights satisfy
\(
w_\tau>0
\)
and
\(
\sum_{\tau\in\Tc} w_\tau=1.
\)
The model is trained using the joint empirical risk $\widehat{\Rc}_T$ over the class $\mathfrak{Q}_T$.

\begin{assumption}[Stationarity and strong mixing]
\label{assump:stationary_mixing_process}
The process
\(
\{W_t=(Z_t,\epshat_{t+\ell})\}_{t\geq1}
\)
is strictly stationary and strongly mixing, with mixing coefficients
$\alpha_W(k)$ satisfying
\[
\alpha_W(k)\longrightarrow0.
\]
\end{assumption}

\begin{assumption}[Residuals and conditional quantile regularity]
\label{assump:residual_bound_regularity}

The following conditions hold.

\begin{enumerate}

\item[(i)] (Bounded residuals)
There exists a constant $B_{\epshat}<\infty$ such that
\(
|\epshat_{t+\ell}| \leq B_{\epshat}
\)
almost surely for all $t\geq 1$.

\item[(ii)] (Regularity around the conditional quantiles)
For each $\tau\in\mathcal T$, the true conditional $\tau$-quantile $q^\star_\tau(z)$ is unique and continuous on $\mathcal Z := \text{supp}(P_Z)$.
Moreover, there exist constants $c_f>0$ and $\delta>0$ such that, for every $\tau\in\mathcal T$ and every $z\in\mathcal Z$, the conditional distribution of $\epshat_{t+\ell}$ given $Z_t=z$ admits a density $f_{\epshat}(\cdot\mid z)$ on
\(
\left[ q^\star_\tau(z)-\delta,\, q^\star_\tau(z)+\delta \right],
\)
satisfying
\[
f_{\epshat}(u \mid z) \geq c_f\; \text{ if }\; \abs{u - q^\star_\tau(z)} \le \delta, \quad z \in \Zc,\; \tau \in \Tc.
\]
\end{enumerate}
\end{assumption}
Bounded residuals provide an integrable envelope for the quantile loss, and local absolute continuity excludes atoms at the target quantiles.

\begin{assumption}[Model classes]
\label{assump:model_class}
For each $\tau\in\mathcal T$, there exists a deterministic function class
$
\overline{\Qc}_\tau = \{q_\theta:\theta\in\overline{\Theta}_\tau\}
$
such that
\(
\mathcal Q_{T,\tau} \subseteq \overline{\Qc}_\tau
\)
for all $T$. The parameter space
$\overline{\Theta}_\tau$ and $\Zc = \operatorname{supp}(P_Z)$ are compact. The mapping
$
(\theta,z) \mapsto q_\theta(z)
$
is jointly continuous. Moreover, there exists $B_Q<\infty$ bounding all candidate functions:
\[
\sup_{\tau\in\mathcal T} \sup_{\eta\in\overline{\Theta}_\tau} \sup_{z\in\mathcal Z} |q_\eta(z)| \leq B_Q.
\]
\end{assumption}
Assumption \ref{assump:model_class} holds, for example, when the context space is compact and all candidate models belong to a fixed finite neural architecture whose component maps are jointly continuous in their inputs and parameters, with parameters restricted to a common compact set.
Under Assumption \ref{assump:model_class}, each
$\overline{\mathcal Q}_\tau$ is compact under
\(
\|q\|_\infty = \sup_{z\in\mathcal Z}|q(z)|.
\)
Moreover, since $\rho_\tau(u)\leq |u|$, Assumptions
\ref{assump:residual_bound_regularity} and \ref{assump:model_class} imply
that, uniformly over $\tau\in\mathcal T$ and
$q\in\overline{\mathcal Q}_\tau$,
\[
\rho_\tau (\epshat_{t+\ell}-q(Z_t)) \leq |\epshat_{t+\ell}|+|q(Z_t)| \leq B_{\epshat}+B_Q.
\]

\begin{assumption}[Joint approximation]
\label{assump:quantile_approximation_error}
The shared-encoder classes satisfy, as $T \to \infty$,
$$
a_T^{\mathrm{joint}} \coloneqq \inf_{\mathbf q\in\mathfrak Q_T} \left\{\mathcal R(\mathbf q) - \mathcal R(\mathbf q^\star) \right\} \longrightarrow 0.
$$
\end{assumption}

Assumption~\ref{assump:quantile_approximation_error} does not require
$\mathbf q^\star \in \mathfrak Q_T$ for any finite $T$. It requires only that the model class contain jointly feasible quantile-function vectors whose joint population risk approaches that of the true conditional quantile vector as $T \to \infty$.

\begin{assumption}[Approximate empirical risk minimization]
\label{assump:joint_empirical_risk_minimization}
The jointly trained estimator
\(
\widehat{\mathbf q}_T
\in\mathfrak Q_T
\)
satisfies
\[
\widehat{\mathcal R}_T
\left(
\widehat{\mathbf q}_T
\right)
\leq
\inf_{\mathbf q\in\mathfrak Q_T}
\widehat{\mathcal R}_T(\mathbf q)
+
r_T,
\]
where $r_T \ge 0$ and
\(
r_T\xrightarrow{p}0.
\)
\end{assumption}

Assumption \ref{assump:joint_empirical_risk_minimization} is the standard approximate-extremum condition \citep{newey1994large}. This requirement concerns the achieved objective value relative to the global empirical infimum, and the quantity $r_T$ allows for optimization error in the model training.

\begin{lemma}[Uniform convergence of the empirical risk]
\label{lemma:uniform_convergence_empirical_risk}
Under Assumptions~\ref{assump:stationary_mixing_process}, \ref{assump:residual_bound_regularity}(i), and
\ref{assump:model_class}, for every $\tau\in\mathcal T$,
\[
\sup_{q\in\mathcal Q_{T,\tau}}
\left|
\widehat R_{T,\tau}(q)-R_\tau(q)
\right|
\xrightarrow{p}0.
\]
\end{lemma}

\begin{proof}[Proof of Lemma~\ref{lemma:uniform_convergence_empirical_risk}]
    Fix $\tau\in\mathcal T$. Since $\mathcal Q_{T,\tau}\subseteq \overline{\mathcal Q}_\tau$, it suffices to establish uniform convergence over the fixed class $\overline{\mathcal Q}_\tau$. For $q\in\overline{\mathcal Q}_\tau$, write
    \[
        L_{\tau,q}(W_i)=\rho_\tau\{\epshat_{i+\ell}-q(Z_i)\}, \quad B_X=B_{\epshat}+B_Q.
    \]
    Then,
    \[
    \widehat R_{T,\tau}(q) = \frac{1}{T}\sum_{i=1}^{T}L_{\tau,q}(W_i), \quad
    R_\tau(q) = \E \left[L_{\tau,q}(W_i)\right].
    \]
    Since the quantile loss satisfies $0\leq \rho_\tau(u)\leq |u|$, Assumptions \ref{assump:residual_bound_regularity}(i) and \ref{assump:model_class} give
    \[
        0\le L_{\tau,q}(W_i)\le |\epshat_{i+\ell}|+|q(Z_i)|\le B_X.
    \]
    By stationarity, the centered variables defined as
    \begin{equation}\label{eq:def_X}
        X_{i,q}=L_{\tau,q}(W_i)-R_\tau(q),
    \end{equation}
    so that $\sum_{i=1}^T X_{i,q} / T = \widehat R_{T,\tau}(q)-R_\tau(q)$, satisfy $\E X_{i,q}=0$ and $|X_{i,q}|\le B_X$ almost surely.
    Since $X_{t,q}$ is a measurable transformation of $W_t$, by Assumption~\ref{assump:stationary_mixing_process}, the process $\{X_{t,q}\}_{t\geq1}$ is also strictly stationary and strongly mixing, with mixing coefficients bounded by $\alpha_W(k)$. 
    Thus,
    \begin{equation}\label{eq:variance_bound}
        \sum_{i=1}^T\operatorname{Var}(X_{i,q})\le TB_X^2,
    \end{equation}
    and the covariance inequality for strong mixing processes \citep{rio1993covariance,rio2017asymptotic} gives
    \begin{equation}\label{eq:qr_cov_ineq_mixing}
        |\Cov(X_{i,q},X_{i+k,q})| \le 4B_X^2\alpha_W(k).
    \end{equation}
    Consequently,
    \begin{align}
        2\sum_{1\le i<j\le T}\operatorname{Cov}(X_{i,q},X_{j,q}) &=2\sum_{k=1}^{T-1}\sum_{i=1}^{T-k} \operatorname{Cov}(X_{i,q},X_{i+k,q})\notag\\ &\le 8B_X^2\sum_{k=1}^{T-1}(T-k)\alpha_W(k)\notag\\ &\le 8B_X^2T\sum_{k=1}^{T-1}\alpha_W(k). \label{eq:covariance_bound}
    \end{align}
    Combining \eqref{eq:variance_bound}--\eqref{eq:covariance_bound}, since $\alpha_W(k)\to0$ in Assumption~\ref{assump:stationary_mixing_process},
    \begin{align*}
        \operatorname{Var}\{\widehat R_{T,\tau}(q)-R_\tau(q)\} &=\frac{1}{T^2}\left\{ \sum_{i=1}^T\operatorname{Var}(X_{i,q}) +2\sum_{1\le i<j\le T}\operatorname{Cov}(X_{i,q},X_{j,q}) \right\}\\ &\le\frac{B_X^2}{T} +\frac{8B_X^2}{T}\sum_{k=1}^{T-1}\alpha_W(k) \longrightarrow0.
    \end{align*}
    Indeed, for every $a>0$, choose $k_a$ such that $\alpha_W(k) \le a$ for $k\ge k_a$. Then, for $T>k_a$,
    \[
        0\le\frac1T\sum_{k=1}^{T-1}\alpha_W(k) \le\frac1T\sum_{k=1}^{k_a-1}\alpha_W(k)+a, \qquad \limsup_{T\to\infty}\frac1T\sum_{k=1}^{T-1}\alpha_W(k)\le a.
    \]
    Letting $a \to 0$ proves the required Ces\`aro convergence. Since $\E[\widehat R_{T,\tau}(q)-R_\tau(q)] = 0$, Chebyshev's inequality implies, for every fixed $q\in\overline{\mathcal Q}_\tau$,
    \begin{equation}\label{eq:pointwise_convergence}
        \widehat R_{T,\tau}(q) - R_\tau(q) \xrightarrow{p} 0.
    \end{equation}
    
    We now use this pointwise convergence to establish uniform convergence over $\overline{\mathcal Q}_\tau$.
    The quantile loss is $1$-Lipschitz, meaning that
    \begin{equation}
        |\rho_\tau(u)-\rho_\tau(v)| \leq |u-v|, \quad \forall u,v \in \R.
    \label{eq:1lipschitz_quantile_loss}
    \end{equation}
    Hence, for $q,q'\in\overline{\mathcal Q}_\tau$,
    \begin{align}
        |\widehat R_{T,\tau}(q)-\widehat R_{T,\tau}(q')|
        &\le\frac1T\sum_{i=1}^T \left|\rho_\tau\{\epshat_{i+\ell}-q(Z_i)\} -\rho_\tau\{\epshat_{i+\ell}-q'(Z_i)\}\right|\notag\\
        &\le\frac1T\sum_{i=1}^T|q(Z_i)-q'(Z_i)| \le\|q-q'\|_\infty, \label{eq:empirical_risk_diff_bound}\\
        |R_\tau(q)-R_\tau(q')|
        &\le\E|q(Z_i)-q'(Z_i)|\le\|q-q'\|_\infty. \label{eq:population_risk_diff_bound}
    \end{align}
    Under Assumption \ref{assump:model_class}, joint continuity on the compact set $\overline\Theta_\tau\times\mathcal Z$ implies
    \[
        \theta_n \to \theta \quad\implies\quad \sup_{z \in \Zc} |q_{\theta_n}(z) - q_\theta(z)| \to 0.
    \]
    Thus $\overline{\mathcal Q}_\tau$ is the continuous image of the compact parameter set in the uniform norm, and is compact. Given $\varepsilon > 0$, choose a finite net $q_1, \dots, q_N \in \overline \Qc_\tau$ such that for every $q \in \overline \Qc_\tau$, $\| q - q_j \|_\infty < \varepsilon / 4$ for some $j \in \{1,\dots,N\}$.
    The triangle inequality and \eqref{eq:empirical_risk_diff_bound}--\eqref{eq:population_risk_diff_bound} yield
    \[
        \sup_{q\in\overline{\mathcal Q}_\tau} |\widehat R_{T,\tau}(q)-R_\tau(q)| \le\frac{\varepsilon}{2} +\max_{1\le j\le N}|\widehat R_{T,\tau}(q_j)-R_\tau(q_j)|.
    \]
    By \eqref{eq:pointwise_convergence} and the union bound,
    \begin{align*}
        \P\left( \sup_{q\in\overline{\mathcal Q}_\tau} |\widehat R_{T,\tau}(q)-R_\tau(q)|>\varepsilon\right) \le\sum_{j=1}^N\P\left( |\widehat R_{T,\tau}(q_j)-R_\tau(q_j)|>\frac{\varepsilon}{2}\right) \longrightarrow0.
    \end{align*}
    Taking the supremum over the smaller class $\Qc_{T,\tau}$ proves the claim.
    
\end{proof}

\begin{corollary}[Uniform convergence of the joint risk]
\label{corollary:joint_uniform_convergence}
Under the conditions of
Lemma~\ref{lemma:uniform_convergence_empirical_risk},
\[
\sup_{\mathbf q\in\mathfrak Q_T}
\left|
\widehat{\mathcal R}_T(\mathbf q)
-
\mathcal R(\mathbf q)
\right|
\xrightarrow{p}0.
\]
\end{corollary}

\begin{proof}[Proof of Corollary~\ref{corollary:joint_uniform_convergence}]
    For every $\qv \in \mathfrak Q_T$, each coordinate belongs to $\Qc_{T,\tau}$. Therefore,
    \begin{align*}
        \sup_{\mathbf q\in\mathfrak Q_T} |\widehat{\mathcal R}_T(\mathbf q)-\mathcal R(\mathbf q)|
        &=\sup_{\mathbf q\in\mathfrak Q_T} \left|\sum_{\tau\in\mathcal T}w_\tau \{\widehat R_{T,\tau}(q_\tau)-R_\tau(q_\tau)\}\right|\\
        &\le\sum_{\tau\in\mathcal T}w_\tau \sup_{q\in\mathcal Q_{T,\tau}} |\widehat R_{T,\tau}(q)-R_\tau(q)|
        \xrightarrow{p} 0,
    \end{align*}
    where the last convergence follows from Lemma \ref{lemma:uniform_convergence_empirical_risk} since $|\Tc| <\infty$ and the weights are fixed.
\end{proof}

\begin{lemma}[Population excess risk convergence]
\label{lemma:population_excess_risk_convergence}
Under the conditions of Corollary~\ref{corollary:joint_uniform_convergence} and Assumptions \ref{assump:quantile_approximation_error} and \ref{assump:joint_empirical_risk_minimization},
\[
\mathcal R
\left(
\widehat{\mathbf q}_T
\right)
-
\mathcal R(\mathbf q^\star)
\xrightarrow{p}0.
\]
\end{lemma}

\begin{proof}[Proof of Lemma~\ref{lemma:population_excess_risk_convergence}]
    Write
    \[
        \Delta_T=\sup_{\mathbf q\in\mathfrak Q_T} |\widehat{\mathcal R}_T(\mathbf q)-\mathcal R(\mathbf q)| \xrightarrow{p}0,
    \]
    by Corollary \ref{corollary:joint_uniform_convergence}. By Assumption \ref{assump:quantile_approximation_error}, we can choose deterministic $\qv^\circ_T \in \mathfrak Q_T$ such that
    \begin{equation}\label{eq:q_choice_m1}
        \Rc(\qv_T^\circ) - \Rc(\qv^\star) \le a_T^{\mathrm{joint}} + T^{-1}.
    \end{equation}
    The role of $\mathbf q_T^\circ$ is to provide a near-optimal population risk approximation to the true quantile vector $\mathbf q^\star$ within the model class $\mathfrak Q_T$.
    Decomposing the risk term gives
    \begin{align*}
        0\le{}&\mathcal R(\widehat{\mathbf q}_T)-\mathcal R(\mathbf q^\star)\\
        ={}&[\mathcal R(\widehat{\mathbf q}_T) -\widehat{\mathcal R}_T(\widehat{\mathbf q}_T)] +[\widehat{\mathcal R}_T(\widehat{\mathbf q}_T) -\widehat{\mathcal R}_T(\mathbf q_T^\circ)]
        + [\widehat{\mathcal R}_T(\mathbf q_T^\circ) -\mathcal R(\mathbf q_T^\circ)] +[\mathcal R(\mathbf q_T^\circ)-\mathcal R(\mathbf q^\star)]\\ \le{}&2\Delta_T+r_T+a_T^{\mathrm{joint}}+T^{-1} \xrightarrow{p} 0.
    \end{align*}
    The second term is at most $r_T$ by Assumption \ref{assump:joint_empirical_risk_minimization}. The last convergence follows from Assumptions \ref{assump:quantile_approximation_error} and \ref{assump:joint_empirical_risk_minimization}.
\end{proof}

\begin{corollary}[Coordinatewise population excess risk convergence]
\label{corollary:coordinatewise_excess_risk}
Under the conditions of
Lemma~\ref{lemma:population_excess_risk_convergence}, for every
$\tau\in\mathcal T$,
\[
R_\tau\left(\widehat q_{T,\tau}\right)
-
R_\tau\left(q^\star_\tau\right)
\xrightarrow{p}0.
\]
\end{corollary}

\begin{proof}[Proof of Corollary~\ref{corollary:coordinatewise_excess_risk}]

    For every $\tau\in\mathcal T$, by definition,
    \[
        R_\tau(\widehat q_{T,\tau}) - R_\tau(q_\tau^\star) \ge 0.
    \]
    Since $w_\tau > 0$ is fixed,
    \begin{align*}
        0 \le R_\tau(\widehat q_{T,\tau})-R_\tau(q_\tau^\star) &\le\frac{1}{w_\tau}\sum_{\tau'\in\mathcal T}w_{\tau'} \{R_{\tau'}(\widehat q_{T,\tau'})-R_{\tau'}(q_{\tau'}^\star)\}\\
        &=\frac{\mathcal R(\widehat{\mathbf q}_T)-\mathcal R(\mathbf q^\star)} {w_\tau}\xrightarrow{p} 0,
    \end{align*}
    by Lemma~\ref{lemma:population_excess_risk_convergence}.
\end{proof}

\begin{lemma}[Uniform consistency of the estimated conditional quantile functions] \label{lemma:quantile_consistency} 
Under the conditions of
Corollary~\ref{corollary:coordinatewise_excess_risk} and Assumptions~\ref{assump:residual_bound_regularity}--\ref{assump:model_class},
$$
\max_{\tau\in\mathcal{T}} \sup_{z\in\mathcal Z} \left| \widehat q_{T,\tau}(z)-q^\star_\tau(z) \right| \xrightarrow{p}0. 
$$ 
\end{lemma}

\begin{proof}[Proof of Lemma~\ref{lemma:quantile_consistency}]
    Fix $\tau\in\mathcal T$. For $q\in\overline{\mathcal Q}_\tau$, write
    \begin{equation*}
    D_\tau(q) = R_\tau(q)-R_\tau(q^\star_\tau).
    \end{equation*}
    For $a \in \R$ and $z \in \Zc$, the conditional excess quantile loss is
    \begin{align*}
    \Gamma_\tau(a,z) &= \E\left[ \rho_\tau(\epshat_{t+\ell}-a) -\rho_\tau\{\epshat_{t+\ell}-q_\tau^\star(z)\} \mid Z_t=z\right]\\
    &=\int_{q_\tau^\star(z)}^a\{F_{\epshat}(u\mid z)-\tau\}\,du. \end{align*}
    Local absolute continuity in Assumption \ref{assump:residual_bound_regularity}(ii) gives
    \begin{equation*}
        F_{\epshat}\left(q^\star_\tau(z)\mid z\right)=\tau.  
    \end{equation*}
    Suppose first $a>q^\star_\tau(z)$ and let $d = \min \{ a - q^\star_\tau(z) , \delta \}$, where $\delta$ is the neighborhood radius in Assumption \ref{assump:residual_bound_regularity}(ii). Monotonicity of the CDF and its local density lower bound imply
    \begin{align*}
        \Gamma_\tau(a,z)
        &\ge \int_{q^\star_\tau(z)}^{q^\star_\tau(z) + d} \left[ F_{\epshat} (u \mid z) - \tau \right] du\\
        &=\int_{q_\tau^\star(z)}^{q_\tau^\star(z)+d} \int_{q_\tau^\star(z)}^u f_{\epshat}(v\mid z) dv du\\
        &\ge \int_{q^\star_\tau(z)}^{q^\star_\tau(z) + d} c_f \left( u - q^\star_\tau(z) \right) du
        = \frac{c_f d^2}{2}.
    \end{align*}
    If $a < q^\star_\tau(z)$, put $d = \min \{q^\star_\tau(z) - a, \delta\}$. The same argument gives
    \begin{align*}
        \Gamma_\tau(a,z)
        &= \int_a^{q^\star_\tau(z)} \left[ \tau - F_{\epshat} (u \mid z) \right] du\\
        &\ge \int_{q_\tau^\star(z) - d}^{q_\tau^\star(z)} \int_u^{q_\tau^\star(z)} f_{\epshat}(v\mid z) dv du\\
        &\ge \int_{q^\star_\tau(z) - d}^{q^\star_\tau(z)} c_f \left( q^\star_\tau(z) - u \right) du
        = \frac{c_f d^2}{2}.
    \end{align*}
    The case $a = q^\star_\tau(z)$ gives zero on both sides. Thus, for $a\in\R$ and $z\in\Zc$, 
    \begin{equation}\label{eq:gamma_lower_bound}
        \Gamma_\tau(a,z)
        \geq
        \frac{c_f}{2}
        \min\left\{
        |a-q^\star_\tau(z)|,\delta
        \right\}^{2}.
    \end{equation}
    In particular,
    \[
        D_\tau(q) = \E \Gamma_\tau(q(Z_t),Z_t) \ge \frac{c_f}{2} \E\min\{|q(Z_t)-q_\tau^\star(Z_t)|,\delta\}^2 \ge 0,
    \]
    and therefore
    \[
        D_\tau(q)=0 \quad \implies \quad q(Z_t)=q_\tau^\star(Z_t) \;\text{ a.s.}
    \]
    For $q\in\overline{\Qc}_\tau$, both $q$ and $q_\tau^\star$ are continuous by Assumptions \ref{assump:model_class} and \ref{assump:residual_bound_regularity}(ii), respectively. If $q(z_0) \ne q_\tau^\star(z_0)$ at some $z_0 \in \Zc$, continuity provides a relative neighborhood $N$ of $z_0$ and a constant $c>0$ with $|q(z)-q_\tau^\star(z)| \ge c$ on $N$. Since $\Zc = \operatorname{supp}(P_Z)$, $P_Z(N)>0$, and
    \[
        D_\tau(q)\ge\frac{c_f}{2}\min\{c,\delta\}^2P_Z(N)>0.
    \]
    Consequently,
    \[
        D_\tau(q)=0\quad\implies\quad q=q_\tau^\star \;\text{ on }\Zc.
    \]
    
    To translate population excess-risk convergence into uniform convergence of the estimated quantile function, we show that any candidate function that remains a fixed distance away from $q_\tau^\star$ in the uniform norm must incur a strictly positive excess population risk. Given $\varepsilon>0$, define the compact set
    \begin{equation*}
        \mathcal A_{\tau,\varepsilon} = \left\{ q\in\overline{\Qc}_\tau : \|q-q^\star_\tau\|_\infty \geq \varepsilon \right\}.
    \end{equation*}
    This set contains all candidate functions in $\overline{\mathcal Q}_\tau$ that are not uniformly close to $q^\star_{\tau}$. If $\mathcal A_{\tau,\varepsilon}$ is empty, then $\P \left( \|\hat{q}_{T,\tau} - q^\star_{\tau} \|_\infty \geq \varepsilon \right) = 0$.
    Otherwise, the inequality
    \[
        |D_\tau(q)-D_\tau(q')| =|R_\tau(q)-R_\tau(q')|\le\|q-q'\|_\infty
    \]
    shows that $D_\tau$ is continuous. Since $D_\tau(q)>0$ for every $q\in\mathcal A_{\tau,\varepsilon}$, continuity and compactness imply
    \[
        c_{\tau,\varepsilon} = \inf_{q\in\mathcal A_{\tau,\varepsilon}} D_\tau(q) > 0.
    \]
    Since $\widehat q_{T,\tau} \in \overline{\mathcal Q}_\tau$,
    \begin{align*}
        \P(\|\widehat q_{T,\tau}-q_\tau^\star\|_\infty \ge \varepsilon) \le \P\{D_\tau(\widehat q_{T,\tau}) \ge c_{\tau,\varepsilon}\} \longrightarrow 0,
    \end{align*}
    by Corollary \ref{corollary:coordinatewise_excess_risk}. Since $|\Tc| < \infty$, the union bound proves the claim.
\end{proof}

\begin{theorem}[Asymptotic conditional coverage]\label{thm:method1_conditional_coverage}
    Suppose Assumptions \ref{assump:stationary_mixing_process}--\ref{assump:joint_empirical_risk_minimization} hold. Fix $\alpha_l, \alpha_h \in \Tc$ with $\alpha_h - \alpha_l = 1 - \alpha$. Then, for any deterministic sequence of $t_T \ge T + \ell$, as $T \to \infty$,
    \[
    \left|\P\left(Y_{t_T+\ell} \in \widehat C_T^{\alpha,\ell}(Z_{t_T}, \widehat f(Z'_{t_T})) \mid Z_{t_T}\right)-(1-\alpha) \right| \xrightarrow{p} 0.
    \]
\end{theorem}

\begin{proof}[Proof of Theorem~\ref{thm:method1_conditional_coverage}]
    For the fixed levels $\alpha_l,\alpha_h$ in the theorem, define the oracle interval
    \[
        C^{\star,\alpha,\ell}(z,\widehat y) =[\widehat y+q_{\alpha_l}^\star(z), \widehat y+q_{\alpha_h}^\star(z)].
    \]
    Local absolute continuity in Assumption \ref{assump:residual_bound_regularity} at the two quantiles implies
    \begin{align}\label{eq:endpoint_identity}
        F_{\epshat}(q_\tau^\star(z)\mid z) = F_{\epshat}(q_\tau^\star(z)- \mid z) = \tau, \quad \tau\in\{\alpha_l,\alpha_h\}.
    \end{align}
    With the same point prediction $\widehat Y_{t_T+\ell} =\widehat f(Z'_{t_T})$ in both estimated and oracle intervals, define the coverage events
    \begin{align*}
        \widehat A_T &= \left\{Y_{t_T+\ell}\in \widehat C_T^{\alpha,\ell}(Z_{t_T},\widehat Y_{t_T+\ell})\right\} 
        = \{\widehat q_{T,\alpha_l}(Z_{t_T}) \le\epshat_{t_T+\ell}\le\widehat q_{T,\alpha_h}(Z_{t_T})\},\\
        A_T^\star &= \left\{Y_{t_T+\ell}\in C^{\star,\alpha,\ell}(Z_{t_T},\widehat Y_{t_T+\ell})\right\}
        =\{q_{\alpha_l}^\star(Z_{t_T}) \le\epshat_{t_T+\ell}\le q_{\alpha_h}^\star(Z_{t_T})\}.
    \end{align*}
    Stationarity and \eqref{eq:endpoint_identity} give
    \begin{align*}
        \P(A_T^\star\mid Z_{t_T}) &= F_{\epshat}(q_{\alpha_h}^\star(Z_{t_T})\mid Z_{t_T}) - F_{\epshat}(q_{\alpha_l}^\star(Z_{t_T})-\mid Z_{t_T})\\
        &=\alpha_h-\alpha_l
        = 1-\alpha \quad \text{ a.s.}
    \end{align*}
    Set
    \[
        d_T=\max_{\tau\in\{\alpha_l,\alpha_h\}} \|\widehat q_{T,\tau}-q_\tau^\star\|_\infty \xrightarrow{p} 0
    \]
    by Lemma \ref{lemma:quantile_consistency}. Given $\eta \in (0, \delta)$, on the event $\{d_T\leq\eta\}$, both estimated quantile endpoints differ from their oracle counterparts by at most $\eta$. Hence, the estimated and oracle coverage events can differ only if the residual lies within $\eta$ of one of the two oracle endpoints. Therefore,
    \[
        \widehat A_T\triangle A_T^\star \subseteq\{d_T>\eta\} \cup\bigcup_{\tau\in\{\alpha_l,\alpha_h\}} \{ |\epshat_{t_T+\ell}-q_\tau^\star(Z_{t_T})|\le\eta \},
    \]
    where $\triangle$ denotes the symmetric difference. Thus union bound gives
    \begin{equation}\label{eq:events_diff_bound}
        \P(\widehat A_T\triangle A_T^\star) \le \P(d_T>\eta) + \sum_{\tau\in\{\alpha_l,\alpha_h\}} \P \{ |\epshat_{t_T+\ell}-q_\tau^\star(Z_{t_T})| \le \eta \}.
    \end{equation}
    For either endpoint, stationarity gives the bound independent of $T$ as
    \begin{equation}\label{eq:residual_oracle_diff}
        \P\{|\epshat_{t_T+\ell}-q_\tau^\star(Z_{t_T})|\le\eta\} = \P\{|\epshat_{1+\ell}-q_\tau^\star(Z_1)|\le\eta\} \eqqcolon b_\tau(\eta).
    \end{equation}
    By continuity from above of probability and local absolute continuity in Assumption~\ref{assump:residual_bound_regularity}(ii),
    \begin{align*}
        \lim_{\eta \downarrow 0}b_\tau(\eta)
        = \P \{\epshat_{1+\ell} = q_\tau^\star(Z_1)\}
        = \E \left[ F_{\epshat}(q_\tau^\star(Z_1)\mid Z_1) -F_{\epshat}(q_\tau^\star(Z_1)-\mid Z_1) \right] = 0.
    \end{align*}
    Combining \eqref{eq:events_diff_bound} and \eqref{eq:residual_oracle_diff},
    \begin{equation}\label{eq:limsup_events_bound}
        \limsup_{T\to\infty}\P(\widehat A_T\triangle A_T^\star) \le\limsup_{T\to\infty}\P(d_T>\eta) +\sum_{\tau\in\{\alpha_l,\alpha_h\}}b_\tau(\eta) =\sum_{\tau\in\{\alpha_l,\alpha_h\}}b_\tau(\eta).
    \end{equation}
    Letting $\eta\downarrow0$ yields $\P(\widehat A_T\triangle A_T^\star)\to 0$. The exact oracle coverage now gives
    \begin{align*}
        |\P(\widehat A_T\mid Z_{t_T})-(1-\alpha)| &= |\E[\ind\{\widehat A_T\}-\ind\{A_T^\star\} \mid Z_{t_T}]|\\
        &\le \E [|\ind\{\widehat A_T\}-\ind\{A_T^\star\}| \mid Z_{t_T}]\\
        &= \P(\widehat A_T\triangle A_T^\star\mid Z_{t_T}).
    \end{align*}
    Taking expectations and applying Markov's inequality gives
    \begin{align*}
        &\P \left( |\P(\widehat A_T\mid Z_{t_T})-(1-\alpha)|>\varepsilon \right)\\
        &\le \frac{1}{\varepsilon} \E |\P(\widehat A_T\mid Z_{t_T})-(1-\alpha)| 
        \le \frac{\P(\widehat A_T\triangle A_T^\star)}{\varepsilon} \longrightarrow 0,
    \end{align*}
    for every $\varepsilon > 0$. This proves the claim.
\end{proof}

\subsection{DSCP-QFE}\label{subsec:theory-qf}

DSCP-QFE estimates one function of the context and the quantile level. Unless otherwise mentioned, we retain the notation of the preceding subsection.

Let $\Ic = [\underline\tau, \overline\tau] \subset (0,1)$
be a compact interval of admissible quantile levels. While the inference can be conducted at any quantile levels in $(0,1)$, we introduce $\mathcal I$ for theoretical analysis.

For $(z,\tau)\in\mathcal Z\times\mathcal I$, define the true quantile function
$$
q^\star(z,\tau):=q^\star_\tau(z)
=\inf\{u\in\mathbb R:F_{\hat\epsilon}(u\mid z)\geq\tau\}.
$$

For a generic quantile function, $q:\Zc\times\Ic\to\R$, we write $q(z,\tau) = q_\tau(z)$. The estimated function is denoted by $\widehat q_T$, and
\(
\widehat q_{T,\tau}(z) = \widehat q_T(z,\tau).
\)
Thus, the notation for each estimated conditional quantile agrees with that used for conditional quantile regression. The architecture is represented by
\[
q_\theta(z,\tau)
=g_{\theta''}\left(\operatorname{Enc}_{\theta'}(z),\tau\right),
\quad \theta=(\theta',\theta''),
\]
Let $\mathfrak Q_T$ denote
the class of jointly feasible candidate functions, so that
$\widehat q_T\in\mathfrak Q_T$.

Let $\nu$ be a probability distribution on $\Ic$. The integrated
population risk is
\begin{equation}
\label{eq:qf-population-risk}
\mathcal R(q)=\int_{\Ic} R_\tau(q(\cdot, \tau))\nu(d\tau) = \E \left[\rho_U \{ \hat\epsilon_{t+\ell}-q(Z_t,U)\}\right],
\end{equation}
where $U\sim\nu$ is independent of the process  $\{ (Z_t,\hat\epsilon_{t+\ell}) \}_{t\geq1}$. The true conditional quantile function minimizes this risk over measurable candidates for which the risk is finite.

For independently sampled quantile levels $U_1,\ldots,U_T$, the randomized empirical risk is
\begin{equation}
\label{eq:qf-randomized-risk}
\widehat{\Rc}_T(q) = \frac{1}{T}\sum_{t=1}^T \rho_{U_t}\{\epshat_{t+\ell}-q(Z_t,U_t)\},
\end{equation}
which uses one sampled quantile level per sample. A fixed number of levels per pair can be treated by averaging their losses. If levels are resampled during optimization, the corresponding integrated empirical objective is
\begin{equation}
\label{eq:qf-integrated-risk}
\widetilde{\Rc}_T(q) = \frac{1}{T}\sum_{t=1}^T\int_{\Ic}
\rho_\tau\{\epshat_{t+\ell}-q(Z_t,\tau)\}\nu(d\tau).
\end{equation}

\begin{assumption}[Quantile-level sampling]
\label{assump:qf-sampling}
The variables $U_t$, $t \ge 1$, are independent and identically distributed with law $\nu$, independent of the process $\{W_t\}_{t\geq1}$. The probability law $\nu$ has full support on $\Ic$, i.e., $\operatorname{supp}(\nu) = \Ic$.
\end{assumption}
Full support is used to identify the target at every level by continuity.

\begin{assumption}[Residuals and conditional quantile regularity]
\label{assump:qf-regularity}

The following conditions hold.

\begin{enumerate}

\item[(i)] (Bounded residuals)
There exists a constant $B_{\epshat}<\infty$ such that
\(
|\epshat_{t+\ell}| \leq B_{\epshat}
\)
almost surely for all $t\geq 1$.

\item[(ii)] (Regularity around the conditional quantiles)
For each $(z,\tau)\in\Zc \times \Ic$, the true conditional $\tau$-quantile $q^\star(z, \tau)$ is unique and jointly continuous on $\Zc \times \Ic$.
Moreover, there exist constants $c_f>0$ and $\delta>0$ such that, for every $\tau\in\mathcal I$ and every $z\in\mathcal Z$, the conditional distribution of $\epshat_{t+\ell}$ given $Z_t=z$ admits a density $f_{\epshat}(\cdot\mid z)$ on
\(
\left[ q^\star_\tau(z)-\delta,\, q^\star_\tau(z)+\delta \right],
\)
satisfying
\[
f_{\epshat}(u \mid z) \geq c_f\; \text{ if }\; \abs{u - q^\star_\tau(z)} \le \delta, \quad z \in \Zc,\; \tau \in \Ic.
\]

\end{enumerate}
\end{assumption}

\begin{assumption}[Model classes]
\label{assump:qf-model}
The classes $\mathfrak Q_T$ are deterministic, nonempty, and contained in a fixed class
\[
\overline{\mathfrak Q} = \{q_\theta:\theta\in\overline\Theta\} \subseteq C(\mathcal Z\times\mathcal I).
\]
The parameter space $\overline\Theta$ and $\mathcal Z = \operatorname{supp}(P_Z)$ are compact. The mapping
$
(\theta,z,\tau)\mapsto q_\theta(z,\tau)
$
is jointly continuous. Moreover, there exists $B_Q<\infty$ such that
$$
\sup_{\theta\in\overline\Theta} \sup_{(z,\tau)\in\mathcal Z\times\mathcal I} |q_\theta(z,\tau)| \leq B_Q.
$$
\end{assumption}

\begin{assumption}[Joint approximation]
\label{assump:qf-approximation}
The integrated approximation error satisfies, as $T \to \infty$,
$$
a_T\coloneqq\inf_{q\in\mathfrak Q_T} \{\mathcal R(q)-\mathcal R(q^\star)\} \longrightarrow 0.
$$
\end{assumption}

\begin{assumption}[Approximate empirical risk minimization]
\label{assump:qf-optimization}
The estimator $\widehat q_T \in \mathfrak Q_T$ satisfies
\[
\widehat{\mathcal R}_T(\widehat q_T)
\leq\inf_{q\in\mathfrak Q_T}\widehat{\mathcal R}_T(q)+r_T,
\]
where $r_T \ge 0$ and $r_T\xrightarrow{p}0$.
\end{assumption}

The error $r_T$ accommodates optimization and algorithmic randomness. Exact empirical-risk minimization is the special case $r_T=0$. If the integrated empirical objective \eqref{eq:qf-integrated-risk} is optimized instead, impose this same inequality with $\widetilde{\mathcal R}_T$ in place of $\widehat{\mathcal R}_T$. Then, the subsequent results hold under that version of the assumption.

\begin{lemma}[Uniform convergence of the empirical risk]
\label{lem:qf-uniform-risk}
Under Assumptions \ref{assump:stationary_mixing_process}, \ref{assump:qf-sampling}, \ref{assump:qf-regularity}(i), and \ref{assump:qf-model},
$$
\sup_{q\in\mathfrak Q_T} \left|\widehat{\mathcal R}_T(q)-\mathcal R(q)\right| \xrightarrow{p} 0.
$$
The same conclusion holds for $\widetilde{\mathcal R}_T$.
\end{lemma}

\begin{proof}[Proof of Lemma~\ref{lem:qf-uniform-risk}]
    Define the concatenated process
    \begin{equation}
    W_t^U \coloneqq (W_t, U_t) = (Z_t, \epshat_{t + \ell}, U_t).
    \end{equation}
    By Assumption~\ref{assump:stationary_mixing_process}, $\{W_t\}_{t\geq 1}$ is strictly stationary and strongly mixing. Moreover, by Assumption~\ref{assump:qf-sampling}, $\{U_t\}_{t\geq 1}$ is i.i.d.~and independent of the process $\{W_t\}_{t\geq 1}$. Therefore, $\{W_t^U\}_{t\geq 1}$ is strictly stationary and strongly mixing, and by \citet[Theorem 5.1(a)]{bradley2005basic}, its mixing coefficients satisfy
    \begin{equation}
        \alpha_{W^U}(k)\leq\alpha_W(k), \quad k\geq1.
    \end{equation}
    In particular,
    \[
        0 \le \alpha_{W^U}(k) \le \alpha_W(k) \longrightarrow 0.
    \]

    For $q \in \overline{\mathfrak Q}$, set $L_q(W_t^U) = \rho_{U_t} \{ \epshat_{t + \ell}  - q(Z_t, U_t) \}$ and $B_X = B_{\epshat} + B_Q$. Throughout the proof, the uniform norm is taken over $\Zc \times \Ic$. Independence of $U_t$ and $W_t$, Assumptions \ref{assump:qf-regularity}(i) and \ref{assump:qf-model} and \eqref{eq:1lipschitz_quantile_loss} give

    \begin{align*}
        &\E L_q(W_t^U) = \int_{\mathcal I}R_\tau(q(\cdot,\tau))\nu(d\tau) = \mathcal R(q),\quad
        0 \le L_q(W_t^U) \le B_X,\\
        &|L_q(W_t^U)-L_{q'}(W_t^U)| \le |q(Z_t,U_t)-q'(Z_t,U_t)|\le\|q-q'\|_\infty.
    \end{align*}
    Joint continuity in Assumption~\ref{assump:qf-model} makes $\theta\mapsto q_\theta$ continuous in this norm. Thus, $\overline{\mathfrak Q}$ is compact as the continuous image. Applying the arguments of Lemma \ref{lemma:uniform_convergence_empirical_risk} with $W_t^U$, $L_q$, and $\overline{\mathfrak Q}$ in place of $W_t$, $L_{\tau,q}$, and $\overline{\mathcal Q}_\tau$, respectively, gives the conclusion.
    
    For the integrated objective, instead set
    \[
        \widetilde L_q(W_t)=\int_{\mathcal I}\rho_\tau\{\epshat_{t+\ell}-q(Z_t,\tau)\}\nu(d\tau),\quad \widetilde{\mathcal R}_T(q)=\frac{1}{T}\sum_{t=1}^T\widetilde L_q(W_t).
    \]
    Since $\nu$ is a probability law, Tonelli's theorem and the same loss bounds imply
    \[
        \E\widetilde L_q(W_t)=\mathcal R(q),\quad
        0 \le \widetilde L_q(W_t)\le B_X,\quad
        |\widetilde L_q(W_t) - \widetilde L_{q'}(W_t)| \le \|q-q'\|_\infty.
    \]
    The same proof of Lemma~\ref{lemma:uniform_convergence_empirical_risk} now applies directly to $W_t$ and $\widetilde L_q$, and yields the uniform convergence for $\widetilde{\mathcal R}_T$.
\end{proof}

\begin{lemma}[Population excess risk convergence]
\label{lem:qf-excess-risk}
Under the conditions of
Lemma~\ref{lem:qf-uniform-risk} and Assumptions~\ref{assump:qf-approximation} and~\ref{assump:qf-optimization},
\[
\mathcal R(\widehat q_T)-\mathcal R(q^\star)\xrightarrow{p}0.
\]
\end{lemma}

\begin{proof}[Proof of Lemma~\ref{lem:qf-excess-risk}]
    By Lemma~\ref{lem:qf-uniform-risk},
    \[
        \Delta_T \coloneqq \sup_{q\in\mathfrak Q_T}|\widehat{\mathcal R}_T(q)-\mathcal R(q)| \xrightarrow{p} 0.
    \]
    Choose a deterministic $q_T^\circ \in \mathfrak Q_T$ with $\Rc(q_T^\circ)-\Rc(q^\star) \le a_T+T^{-1}$, which exists by Assumption \ref{assump:qf-approximation}.
    The decomposition argument in the proof of Lemma \ref{lemma:population_excess_risk_convergence} applies with the present integrated risks and $a_T$ in place of $a_T^{\mathrm{joint}}$. Consequently,
    \[
        0\le\mathcal R(\widehat q_T)-\mathcal R(q^\star)\le2\Delta_T+r_T+a_T+T^{-1} \xrightarrow{p} 0,
    \]
    by Assumptions \ref{assump:qf-approximation} and \ref{assump:qf-optimization}. For the integrated empirical objective, replace $\widehat{\mathcal R}_T$ by $\widetilde{\mathcal R}_T$ throughout, using the corresponding version of Assumption \ref{assump:qf-optimization}.
\end{proof}

\begin{lemma}[Uniform consistency of the estimated conditional quantile function]
\label{lem:qf-uniform-consistency}
Under the conditions of Lemma~\ref{lem:qf-excess-risk} and Assumptions~\ref{assump:qf-sampling}--\ref{assump:qf-model},
\[
\sup_{z\in\mathcal Z}\sup_{\tau\in\mathcal I}
\left|\widehat q_T(z,\tau)-q^\star(z,\tau)\right|
\xrightarrow{p}0.
\]
\end{lemma}

\begin{proof}[Proof of Lemma~\ref{lem:qf-uniform-consistency}]
    Write $D(q) = \Rc(q)-\Rc(q^\star)$ for $q\in\overline{\mathfrak Q}$. Assumption \ref{assump:qf-regularity}(ii) allows the lower bound \eqref{eq:gamma_lower_bound} from the proof of Lemma~\ref{lemma:quantile_consistency} to be applied at every $\tau\in\Ic$, with $q_\tau^\star(z)=q^\star(z,\tau)$. Integrating that bound gives
    \begin{equation}\label{eq:qf-risk-identification}
        D(q)\ge\frac{c_f}{2}\int_{\mathcal I}\int_{\mathcal Z}\min\{|q(z,\tau)-q^\star(z,\tau)|,\delta\}^2\,P_Z(dz)\nu(d\tau) \ge 0.
    \end{equation}
    If $q(z_0,\tau_0)\ne q^\star(z_0,\tau_0)$, joint continuity gives relative open neighborhoods $N_z$ and $N_\tau$ of $z_0$ and $\tau_0$, and a constant $c>0$, such that $|q-q^\star|\ge c$ on $N_z\times N_\tau$. Since $\mathcal Z=\operatorname{supp}(P_Z)$ and $\mathcal I=\operatorname{supp}(\nu)$,
    \[
        P_Z(N_z)>0,\quad
        \nu(N_\tau)>0,\quad
        D(q)\ge\frac{c_f}{2}\min\{c,\delta\}^2P_Z(N_z)\nu(N_\tau)>0.
    \]
    Consequently,
    \[
        D(q)=0 \implies q=q^\star \text{ on } \Zc\times\Ic.
    \]
    The class $\overline{\mathfrak Q}$ is compact, as established in the proof of Lemma \ref{lem:qf-uniform-risk}, and $D$ is continuous because
    \[
        |D(q)-D(q')|=|\mathcal R(q)-\mathcal R(q')|\le\|q-q'\|_\infty.
    \]
    The compact-separation argument in the proof of Lemma~\ref{lemma:quantile_consistency} therefore applies with $\mathcal Z\times\mathcal I$, $\overline{\mathfrak Q}$, $D$, and $q^\star$ in place of $\mathcal Z$, $\overline{\mathcal Q}_\tau$, $D_\tau$, and $q_\tau^\star$.
    Together with $D(\widehat q_T)\xrightarrow{p}0$ from Lemma \ref{lem:qf-excess-risk}, this gives $\|\widehat q_T-q^\star\|_\infty\xrightarrow{p}0$ on the product domain.
\end{proof}

\begin{theorem}[Conditional coverage for quantile function estimation]\label{thm:qf-coverage}
Suppose Assumption~\ref{assump:stationary_mixing_process} and Assumptions~\ref{assump:qf-sampling}--\ref{assump:qf-optimization} hold. Fix $\alpha_l, \alpha_h \in \Ic$ with $\alpha_h - \alpha_l = 1 - \alpha$. Then, for any deterministic sequence of $t_T \ge T + \ell$, as $T \to \infty$,
    \[
    \left|\P\left(Y_{t_T+\ell} \in \widehat C_T^{\alpha,\ell}(Z_{t_T}, \widehat f(Z'_{t_T})) \mid Z_{t_T}\right)-(1-\alpha) \right| \xrightarrow{p} 0.
    \] 
\end{theorem}

\begin{proof}[Proof of Theorem~\ref{thm:qf-coverage}]
    For the fixed levels $\tau\in\{\alpha_l,\alpha_h\}$ in the theorem, set $\widehat q_{T,\tau}(z)=\widehat q_T(z,\tau)$ and $q_\tau^\star(z)=q^\star(z,\tau)$. Lemma \ref{lem:qf-uniform-consistency} gives
    \[
        d_T \coloneqq \max_{\tau\in\{\alpha_l,\alpha_h\}}\sup_{z\in\mathcal Z}|\widehat q_{T,\tau}(z)-q_\tau^\star(z)|\le\|\widehat q_T-q^\star\|_\infty \xrightarrow{p} 0.
    \]
    Local absolute continuity in Assumption \ref{assump:qf-regularity}(ii) at the two quantiles implies
    \[
        F_{\epshat}(q_\tau^\star(z)\mid z)=F_{\epshat}(q_\tau^\star(z)-\mid z)=\tau,\quad\tau\in\{\alpha_l,\alpha_h\}.
    \]
    Together with stationarity of $W_t$, these verify all ingredients of the coverage argument in Theorem~\ref{thm:method1_conditional_coverage}. Applying that proof to the present endpoint functions and the same point prediction $\widehat f(Z'_{t_T})$ yields the claimed conclusion.
\end{proof}

\subsection{DSCP-Local}\label{subsec:theory-LCP}

The localized estimator first targets the residual distribution conditional on $H_t = \operatorname{Enc}_{\theta'}(Z_t)$. We establish consistency at a fixed representation value, transfer it to the observed representation, and then obtain conditional coverage. Coverage conditional on $Z_t$ is obtained under an additional residual-sufficiency assumption. Throughout this subsection, the encoder is fixed, and its representation space is a Borel subset of a Euclidean space equipped with a jointly measurable metric $d_H$.

\subsubsection{Coverage conditional on the representation \texorpdfstring{$H_t$}{H-t}}\label{subsubsec:H-conditional}

\begin{assumption}[Stationarity and absolute regularity]\label{ass:H1}
    The process $V_t = (H_t, \hat \epsilon_{t+\ell})$, $t \ge 1$, is strictly stationary and $\beta$-mixing with its mixing coefficient
    \[
        \beta_V(k) \longrightarrow 0.
    \]
\end{assumption}


\begin{assumption}[Kernel and local mass]\label{ass:H2}
    The measurable kernel satisfies $0 \le K \le K_{\max} < \infty$, $K(v)=0$ for $v>1$, and $K(v) \ge K_{\min}>0$ for $0 \le v \le a_0$, where $a_0 \in (0,1]$. There is a common $d_0 > 0$ such that, for $P_H$-almost every $h$, constants $0<c_h<C_h<\infty$ satisfy
    \begin{align}\label{eq:kernel-local-mass}
        \P(d_H(H,h) \le a_0b_T) \ge c_h b_T^{d_0}, \quad
        \P(d_H(H,h)\le b_T) \le C_h b_T^{d_0},
    \end{align}
for all sufficiently large $T$, where $H \sim P_H$ is a generic representation.
    
\end{assumption}

Several kernels, including uniform, triangular, Epanechnikov, as well as the truncated Gaussian kernel satisfy the kernel conditions in Assumption \ref{ass:H2}.


\begin{assumption}[Local CDF regularity]\label{ass:H3}
    For $P_H$-almost every $h$, $\FH(\cdot \mid h)$ is a continuous proper CDF and
    \begin{equation}\label{eq:locality}
        \omega_T(h) \coloneqq \esssup_{h': \, d_H(h',h) \le b_T} \sup_{u \in \R}|\FH(u \mid h')-\FH(u \mid h)| \to 0.
    \end{equation}
    The essential supremum is with respect to $P_H$.
\end{assumption}

The quantity $\omega_T(h)$ controls the bias from pooling nearby representations. A sufficient condition is local H\"older continuity of the conditional CDF in the representation metric, uniformly in $u$, which gives $\omega_T(h)\le L_h b_T^\gamma$ for some $\gamma>0$. Continuity in $u$ excludes atoms and permits inversion at neighboring probability levels.

For example, let $\epshat_{t+\ell}=m(H_t)+\xi_{t+\ell}$, where $\xi_{t+\ell}$ is independent of $H_t$ and has a CDF $G$ with density bounded by $L_G$. Suppose $|m(h')-m(h)|\le L_{m,h}d_H(h',h)^\gamma$ locally at $P_H$-almost every $h$, for some $\gamma>0$. Then $\FH(u\mid h)=G(u-m(h))$ is a continuous proper CDF and
\[
    \sup_u|\FH(u\mid h')-\FH(u\mid h)|\le L_GL_{m,h}d_H(h',h)^\gamma,
\]
so $\omega_T(h)\le L_GL_{m,h}b_T^\gamma\to0$ whenever $b_T\to0$.

\begin{assumption}[Bandwidth and dependence structure]\label{ass:H4}
    Let $p_T=b_T^{d_0}$ and
    \begin{equation}\label{eq:Lambda}
        \Lambda_T = p_T + \sum_{k=1}^{T-1} \min\{p_T, \beta_V(k)\}.
    \end{equation}
    Assume $b_T \to 0$ and
    \begin{equation}\label{eq:variance-condition}
        \kappa_T^\beta \coloneqq \frac{\Lambda_T}{T p_T^2} \to 0.
    \end{equation}
\end{assumption}

\begin{lemma}[Conditional CDF consistency at a fixed representation]\label{lem:prefix}
    Assume \ref{ass:H1}--\ref{ass:H3}, $b_T \to 0$, and the variance condition \eqref{eq:variance-condition}. Let $M \coloneqq M_T$ be deterministic integers with $1 \le M_T \le T$ and $M_T / T \to 1$. Write
    \begin{align*}
        S_{T,M}(h) &= \sum_{i=1}^M K_{T,i}(h), &B_{T,M}(u,h) &= \sum_{i=1}^M K_{T,i}(h) \ind \{\epshat_{i+\ell} \le u\},\\
        \mu_T(h) &= \E K(d_H(H,h)/b_T), &\widehat F^H_{T,M}(u \mid h) &= B_{T,M}(u,h)/S_{T,M}(h),
    \end{align*}
    if $S_{T,M}(h) > 0$. When $S_{T,M}(h)=0$, set $\widehat F^H_{T,M}(u \mid h) = \ind\{u \ge 0\}$. Then, for $P_H$-almost every fixed $h$,
    \begin{equation}\label{eq:prefix}
        \frac{S_{T,M_T}(h)}{M_T \mu_T(h)} \xrightarrow{p} 1,
        \qquad
        \sup_u |\widehat F^H_{T,M_T}(u \mid h) - \FH(u \mid h)| \xrightarrow{p} 0.
    \end{equation}
\end{lemma}

\begin{proof}[Proof of Lemma \ref{lem:prefix}]
    Fix $h$ for which \eqref{eq:kernel-local-mass} and \eqref{eq:locality} hold. In this proof, $C$ denotes a finite positive constant that may change depending on $h$ but not on $T$, $M$, or $u$. For all sufficiently large $T$, the kernel bounds \eqref{eq:kernel-local-mass} give
    \begin{align}\label{eq:mu-bounds}
        K_{\min} c_h p_T \le K_{\min} \P(d_H(H,h) \le a_0b_T) \le \mu_T(h) \le K_{\max} \P(d_H(H,h)\le b_T) \le K_{\max} C_h p_T.
    \end{align}
    For either $X_i = K_{T,i}(h)$ or $X_i = K_{T,i}(h) \ind\{\epshat_{i + \ell} \le u\}$,
    \begin{align*}
        \Var(X_i) &\le \E X_i^2 \le K_{\max}^2 \P(d_H(H,h)\le b_T) \le K_{\max}^2 C_h p_T,\\
        \abs{\Cov(X_i, X_{i+k})} &\le \E(X_i X_{i+k}) + \E X_i \E X_{i+k} \le 2K_{\max}^2 \P(d_H(H,h)\le b_T) \le 2K_{\max}^2 C_h p_T.
    \end{align*}
    Also, from the coupling argument of \citet[Corollary 1]{dedecker2006parametrized}, we have $\abs{\Cov(X_i, X_{i+k})} \le K_{\max}^2 \beta_V(k)$. Therefore,
    \[
        \abs{\Cov(X_i,X_{i+k})} \le K_{\max}^2\min\{2C_h p_T,\beta_V(k)\} \le C \min \{p_T, \beta_V(k)\},
    \]
    and hence we can bound the variance term as
    \begin{align}
        \Var \left(\sum_{i=1}^M X_i\right) &= M \Var(X_1) + 2\sum_{k=1}^{M-1}(M-k) \Cov(X_1, X_{1+k})\nonumber\\
        &\le CM \left[p_T + \sum_{k=1}^{T-1} \min\{p_T, \beta_V(k)\}\right] = CM \Lambda_T. \label{eq:decomp_variance}
    \end{align}
    To bound the bias term, define the population localized CDF by
    \[
        \overline F_T^H(u \mid h) = \frac{\E[K(d_H(H,h)/b_T) \FH(u \mid H)]}{\mu_T(h)}.
    \]
    Assumption \ref{ass:H3} gives
    \[
        \abs{\overline F_T^H(u \mid h) - \FH(u \mid h)} = \frac{\abs{\E[K(d_H(H,h)/b_T)\{\FH(u \mid H) - \FH(u \mid h)\}]}}{\mu_T(h)} \le \omega_T(h).
    \]
    In particular,
    \begin{align}\label{eq:decomp_bias}
        \sup_u \abs{\overline F_T^H(u \mid h) - \FH(u \mid h)} \le \omega_T(h) \to 0.
    \end{align}
    Since $\E S_{T,M}(h) = M \mu_T(h)$ and
    \[
        \E B_{T,M}(u,h) = \sum_{i=1}^M \E\left[K_{T,i}(h)\E\{\ind(\epshat_{i+\ell} \le u) \mid H_i\}\right] = M \mu_T(h) \overline F_T^H(u \mid h),
    \]
    Chebyshev's inequality gives, for every $v > 0$,
    \[
    \begin{aligned}
    &\P\left(\abs{\frac{S_{T,M}(h)}{M\mu_T(h)} - 1} >v \right)
    \begin{aligned}[t]
        &= \P(\abs{S_{T,M}(h) - M\mu_T(h)} > v M \mu_T(h))\\
        &\le \frac{CM \Lambda_T}{v^2 M^2 \mu_T(h)^2}
        \le \frac{C \kappa_T^\beta}{v^2} \to 0,
    \end{aligned}\\
    &\P\left(\abs{B_{T,M}(u,h) - M \mu_T(h) \overline F_T^H(u \mid h)}
    > vM \mu_T(h) \right)\\
    &\phantom{\P\left(\abs{\frac{S_{T,M}(h)}{M\mu_T(h)} - 1} >v \right)}
    \le \frac{CM \Lambda_T}{v^2 M^2 \mu_T(h)^2}
    \le \frac{C \kappa_T^\beta}{v^2} \to 0.
    \end{aligned}
    \]
    The first bound proves the first conclusion of \eqref{eq:prefix}. For $0 < v < 1$, on the event where the following inequality,
    \begin{align*}
        \max\left\{\abs{S_{T,M}(h) - M \mu_T(h)}, \abs{B_{T,M}(u,h) - M \mu_T(h) \overline F_T^H(u \mid h)}\right\} \le \frac{vM \mu_T(h)}{4},
    \end{align*}
    holds, we have $S_{T,M}(h) > 0$ and
    \begin{align*}
        \abs{\widehat F^H_{T,M}(u \mid h) - \overline F_T^H(u \mid h)}
        &= \frac{\abs{B_{T,M}(u,h) - S_{T,M}(h) \overline F_T^H(u \mid h)}}{S_{T,M}(h)}\\
        &\le \frac{2}{M\mu_T(h)}\abs{B_{T,M}(u,h) - M\mu_T(h) \overline F_T^H(u \mid h)}\\
        &\quad+\frac{2}{M\mu_T(h)}\abs{M\mu_T(h)-S_{T,M}(h)}\\
        &\le v.
    \end{align*}
    Therefore,
    \begin{align}
        &\P\left(\abs{\widehat F^H_{T,M}(u \mid h) - \FH(u \mid h)} > \omega_T(h) + v \right)\nonumber\\
        &\le \P\left(\abs{S_{T,M}(h) - M \mu_T(h)} > \frac{vM \mu_T(h)}{4} \right)\nonumber\\
        &\quad+\P\left(\abs{B_{T,M}(u,h) - M \mu_T(h) \overline F_T^H(u \mid h)} > \frac{vM \mu_T(h)}{4}\right)\nonumber\\
        &\le C \kappa_T^\beta/v^2.\label{eq:fixed-tail}
    \end{align}
    Now we use the finite-grid argument of \citet[Lemma 2.11]{van2000asymptotic} for the uniformity result. Fix $m \ge 2$ and abbreviate $F=\FH(\cdot\mid h)$. Take $x_j = F^{-1}(j/m)$ for $j = 1,\dots,m-1$, and $x_0 = -\infty, x_m = +\infty$. For any proper CDF $G$ and $x_j \le u < x_{j+1}$,
    \begin{align*}
        G(u)-F(u) & \le G(x_{j+1}) - F(x_j) \le \max_{0 \le k \le m} \abs{G(x_k)-F(x_k)} + m^{-1},\\
        F(u)-G(u) &\le F(x_{j+1}) - G(x_j) \le \max_{0 \le k \le m} \abs{G(x_k)-F(x_k)} + m^{-1},\\
        \norm{G-F}_\infty &\le \max_{0 \le k \le m} \abs{G(x_k)-F(x_k)} + m^{-1}.
    \end{align*}
    Applying this to $G = \widehat F^H_{T,M}(\cdot \mid h)$ and using \eqref{eq:fixed-tail},
    \begin{align}
        &\P\left(\norm{\widehat F^H_{T,M}(\cdot \mid h) -\FH(\cdot \mid h)}_\infty > \omega_T(h) + v + m^{-1}\right)\nonumber\\
        &\le \sum_{j=1}^{m-1} \P\left(\abs{G(x_j) - F(x_j)} > \omega_T(h) + v\right)\nonumber\\
        &\le Cm \kappa_T^\beta/v^2.\label{eq:uniform-tail}
    \end{align}
    Given $0 < \varepsilon < 1$, choose $m \ge 4/\varepsilon$ and $v = \varepsilon/4$. For all sufficiently large $T$, $\omega_T(h) \le \varepsilon/4$, so
    \[
        \P\left(\norm{\widehat F^H_{T,M}(\cdot \mid h) - \FH(\cdot \mid h)}_\infty > \varepsilon\right)
        \le \frac{16 Cm \kappa_T^\beta}{\varepsilon^2} \to 0.
    \]
    This proves the second conclusion of \eqref{eq:prefix}.
\end{proof}

\begin{theorem}[$H$-conditional CDF consistency]\label{thm:H-cdf}
    Under Assumptions \ref{ass:H1}--\ref{ass:H4},
    \begin{equation}
        \sup_u |\FhatH(u \mid H_{t_T}) - \FH(u \mid H_{t_T})| \xrightarrow{p} 0
    \end{equation}
\end{theorem}

\begin{proof}[Proof of Theorem \ref{thm:H-cdf}]
    Set $s_T=\floor{\sqrt{Tp_T}}$ and $M_T=T-s_T$, and retain the prefix $\Vc_T=(V_1,\dots,V_{M_T})$. Since $\Lambda_T\ge p_T$, Assumption~\ref{ass:H4} gives $(Tp_T)^{-1}\le\kappa_T^\beta\to0$. Consequently, $s_T\to\infty$, $s_T/(Tp_T)\le(Tp_T)^{-1/2}\to0$, $M_T/T\to1$, and $\beta_V(s_T)\to0$. On $S_{T, M_T}(h) > 0$,
    \begin{align*}
        S_T(h)\left\{\FhatH(u \mid h) - \widehat F^H_{T,M_T}(u \mid h)\right\}
        &= \sum_{i=1}^T K_{T,i}(h) \ind\{\epshat_{i+\ell} \le u\} - S_T(h) \widehat F^H_{T,M_T}(u \mid h)\\
        &= \sum_{i=M_T+1}^T K_{T,i}(h) \left\{\ind\{\epshat_{i+\ell}\le u\} - \widehat F^H_{T,M_T}(u \mid h)\right\}.
    \end{align*}
    Since $0 \le \widehat F^H_{T,M_T} \le 1$ and $S_T(h) \ge S_{T,M_T}(h)$,
    \begin{equation*}
        \norm{\FhatH(\cdot \mid h) - \widehat F^H_{T,M_T}(\cdot \mid h)}_\infty
        \le \frac{\sum_{i=M_T+1}^T K_{T,i}(h)}{S_T(h)}
        \le \frac{K_{\max} \floor{\sqrt{T p_T}}}{S_{T,M_T}(h)}.
    \end{equation*}
    Now, the triangle inequality gives
    \begin{align*}
        \norm{\FhatH(\cdot \mid h) - \FH(\cdot \mid h)}_\infty \le \norm{\widehat F^H_{T,M_T}(\cdot \mid h) - \FH(\cdot \mid h)}_\infty + \frac{K_{\max} \floor{\sqrt{T p_T}}}{S_{T,M_T}(h)}
    \end{align*}
    Define the bound $g_T$, which is the function of only $\Vc_T$ and $h$:
    \[
        g_T(\Vc_T,h)=
        \begin{cases}
            \min\left\{1, \norm{\widehat F^H_{T,M_T}(\cdot \mid h) - \FH(\cdot \mid h)}_\infty + \frac{K_{\max} s_T}{S_{T,M_T}(h)} \right\},
            &S_{T,M_T}(h) > 0,\\
            1,
            &S_{T,M_T}(h) = 0,
        \end{cases}
    \]
    so that
    \begin{equation}\label{eq:uniform-bound-g_T}
        \norm{\FhatH(\cdot \mid h) - \FH(\cdot \mid h)}_\infty \le g_T(\Vc_T,h) \le 1.
    \end{equation}
    For almost every fixed $h$, Lemma \ref{lem:prefix} and \eqref{eq:mu-bounds} imply
    \begin{align*}
        \norm{\widehat F^H_{T,M_T}(\cdot \mid h) - \FH(\cdot \mid h)}_\infty \xrightarrow{p} 0, \quad
        \frac{S_{T,M_T}(h)}{M_T \mu_T(h)} \xrightarrow{p} 1, \quad
        \frac{K_{\max} s_T}{M_T \mu_T(h)} \le \frac{K_{\max}}{K_{\min} c_h} \frac{T}{M_T} \frac{s_T}{T p_T} \to 0.
    \end{align*}
    On $\{S_{T,M_T}(h) \ge M_T \mu_T(h)/2\}$,
    \[
        \frac{K_{\max}s_T}{S_{T,M_T}(h)} \le \frac{2K_{\max} s_T}{M_T \mu_T(h)}.
    \]
    Therefore, given $\varepsilon > 0$,
    \begin{align*}
        &\P\{g_T(\Vc_T,h) > \varepsilon\}\\
        &\le \P\left\{\norm{\widehat F^H_{T,M_T}(\cdot \mid h) - \FH(\cdot \mid h)}_\infty > \frac{\varepsilon}{2}\right\}\\
        &\quad+\P\left\{\abs{\frac{S_{T,M_T}(h)}{M_T \mu_T(h)} - 1} > \frac{1}{2}\right\}
        +\ind\left\{\frac{4K_{\max}s_T}{M_T \mu_T(h)} > \varepsilon\right\}\\
        &\to 0.
    \end{align*}
    Since $0\le g_T(\Vc_T,h)\le1$ and $g_T(\Vc_T,h)\xrightarrow{p}0$, we have $\E g_T(\Vc_T,h)\to0$. Dominated convergence therefore gives
    \begin{equation}\label{eq:int-prefix}
        \int \E g_T(\Vc_T,h)P_H(dh) \to 0.
    \end{equation}
    Applying \citet[Corollary 1]{dedecker2006parametrized} gives
    \[
        \abs{\E g_T(\Vc_T,H_{t_T}) - \int \E g_T(\Vc_T,h)P_H(dh)} \le \beta_V(\floor{\sqrt{T p_T}}).
    \]
    Consequently, combining \eqref{eq:uniform-bound-g_T} and \eqref{eq:int-prefix},
    \begin{align*}
        \E \left[ \sup_u \abs{\FhatH(u \mid H_{t_T}) - \FH(u \mid H_{t_T})} \right]
        &\le \E g_T(\Vc_T, H_{t_T})\\
        &\le \int \E g_T(\Vc_T,h)P_H(dh) + \beta_V(\floor{\sqrt{T p_T}}) \to 0,
    \end{align*}
    and Markov's inequality gives the conclusion.
\end{proof}

\begin{corollary}[$H$-conditional coverage]\label{cor:H-coverage}
    Under Assumptions \ref{ass:H1}--\ref{ass:H4},
    \begin{align}
        \left| \P\left(Y_{t_T+\ell} \in \widehat C_T^{\alpha,\ell}(Z_{t_T}, \widehat Y_{t_T + \ell}) \mid H_{t_T}\right) - (1-\alpha) \right| \xrightarrow{p} 0.
    \end{align}
    
\end{corollary}

\begin{proof}[Proof of Corollary~\ref{cor:H-coverage}]
    For $0<a<1$, define the true generalized inverse
    \[
        q_a^{H,\star}(h) = \inf \{u: \FH(u \mid h) \ge a\}.
    \]
    For every ``good'' $h$ of Assumption \ref{ass:H3}, properness makes $q_a^{H,\star}(h)$ finite, and continuity gives
    \begin{equation}\label{eq:lcp-true-quantile-level}
        \FH(q_a^{H,\star}(h) - \mid h) = \FH(q_a^{H,\star}(h) \mid h) = a.
    \end{equation}
    Indeed, $u<q_a^{H,\star}(h)$ implies $\FH(u\mid h)<a$, and right-continuity implies $\FH(q_a^{H,\star}(h) \mid h) \ge a$. The quantile is measurable in $h$, since for every $r \in \R$,
    \[
        \{h: q_a^{H,\star}(h) < r\} = \bigcup_{v \in \mathbb Q: \, v < r}\{h: \FH(v \mid h) \ge a\}.
    \]
    Fix
    \[
        0 < \eta < \min \{\alpha_l, 1-\alpha_h, (\alpha_h-\alpha_l)/2\}, \quad D_T = \sup_u|\FhatH(u \mid H_{t_T}) - \FH(u \mid H_{t_T})|.
    \]
    Whenever $\|\FhatH(\cdot \mid h) - \FH(\cdot \mid h)\|_\infty \le \eta$ and $\tau \in \{\alpha_l,\alpha_h\}$,
    \begin{align*}
        u<q_{\tau-\eta}^{H,\star}(h) &\implies \FhatH(u \mid h) \le \FH(u \mid h) + \eta < \tau,\\ 
        \FhatH(q_{\tau+\eta}^{H,\star}(h) \mid h) &\ge \FH(q_{\tau+\eta}^{H,\star}(h) \mid h) - \eta = \tau. \end{align*}
    Taking the infimum in the definition of $\widehat q^H_{T,\tau}$ gives
    \begin{equation}\label{eq:lcp-quantile-bracket}
        q_{\tau-\eta}^{H,\star}(h) \le \widehat q^H_{T,\tau}(h) \le q_{\tau+\eta}^{H,\star}(h).
    \end{equation}
    Write
    \begin{align*}
        A_T &= \{Y_{t_T+\ell} \in \widehat C_T^{\alpha,\ell} (Z_{t_T}, \widehat Y_{t_T+\ell})\}\\
        &= \{\widehat q^H_{T,\alpha_l}(H_{t_T}) \le \epshat_{t_T+\ell} \le \widehat q^H_{T,\alpha_h}(H_{t_T})\}.
    \end{align*}
    By \eqref{eq:lcp-quantile-bracket}, on $\{D_T\le\eta\}$,
    \begin{align*}
        \{q_{\alpha_l+\eta}^{H,\star}(H_{t_T}) \le \epshat_{t_T+\ell} \le q_{\alpha_h-\eta}^{H,\star}(H_{t_T})\} \subseteq A_T \subseteq \{q_{\alpha_l-\eta}^{H,\star}(H_{t_T}) \le \epshat_{t_T+\ell} \le q_{\alpha_h+\eta}^{H,\star}(H_{t_T})\}.
    \end{align*}
    For any fixed $0<a<b<1$, stationarity and \eqref{eq:lcp-true-quantile-level} imply, almost surely,
    \begin{align}
        &\P\{q_a^{H,\star}(H_{t_T}) \le \epshat_{t_T+\ell} \le q_b^{H,\star}(H_{t_T}) \mid H_{t_T}\} \nonumber\\ &=\FH(q_b^{H,\star}(H_{t_T}) \mid H_{t_T}) -\FH(q_a^{H,\star}(H_{t_T}) - \mid H_{t_T}) =b-a.\label{eq:lcp-oracle-H}
    \end{align}
    The left limit accounts for the closed lower endpoint.
    Conditioning the event inclusions on $H_{t_T}$ therefore gives
    \begin{align*}
        1-\alpha-2\eta-\P(D_T>\eta \mid H_{t_T}) \le\P(A_T\mid H_{t_T}) \le 1-\alpha+2\eta+\P(D_T>\eta \mid H_{t_T}),
    \end{align*}
    and hence
    \begin{equation}\label{eq:lcp-H-coverage-bound}
        |\P(A_T\mid H_{t_T})-(1-\alpha)| \le2\eta+\P(D_T>\eta\mid H_{t_T}).
    \end{equation}
    Given $\delta>0$, choose the preceding $\eta$ also to satisfy $2\eta<\delta$. Conditional expectation and Markov's inequality yield
    \begin{align*}
        &\P\{|\P(A_T\mid H_{t_T})-(1-\alpha)| > \delta\}\\
        &\le \P\{\P(D_T>\eta \mid H_{t_T})>\delta-2\eta\}\\ &\le \frac{\E[\P(D_T>\eta \mid H_{t_T})]}{\delta-2\eta} =\frac{\P(D_T>\eta)}{\delta-2\eta} \longrightarrow 0,
    \end{align*}
where the last step is Theorem \ref{thm:H-cdf}.
\end{proof}

\subsubsection{Extension to coverage conditional on the context \texorpdfstring{$Z_t$}{Z-t}}\label{subsubsec:Z-conditional}

\begin{assumption}[Stationarity and residual sufficiency]\label{ass:Z1}
    The process $W_t = (Z_t, \epshat_{t+\ell})$ is strictly stationary, and for every $u \in \R$ and $\P_Z$-almost every $z$,
    \begin{equation}\label{eq:Z-sufficiency}
        \FZ(u \mid z)=\FH(u \mid \operatorname{Enc}_{\theta'}(z)).
    \end{equation}
\end{assumption}

\paragraph{Interpretation of residual sufficiency.} Because $H_t=\operatorname{Enc}_{\theta'}(Z_t)$ is a function of $Z_t$, \eqref{eq:Z-sufficiency} is equivalent to $\epshat_{t+\ell}\mathrel{\perp\!\!\!\perp} Z_t\mid H_t$. This is the sufficient dimension reduction condition for the residual \citep{cook2007dimension, lee2013nonlinear}. The encoder may discard features that do not affect its conditional distribution. For example, if $Z_t=(S_t,N_t)$, $H_t=S_t$, and $\epshat_{t+\ell}=m(S_t)+\sigma(S_t)\xi_{t+\ell}$, where $\sigma>0$ and $\xi_{t+\ell}$ is independent of $Z_t$ with CDF $G$, then \[\FZ(u\mid s,n)=G\!\left(\frac{u-m(s)}{\sigma(s)}\right)=\FH(u\mid s),\]although the nuisance component $N_t$ cannot in general be recovered from $H_t$. This preserves the residual distribution rather than all information in $Z_t$.

\begin{theorem}[$Z$-conditional CDF consistency and $Z$-conditional coverage]\label{thm:Z-coverage}
    Under Assumptions \ref{ass:H1}--\ref{ass:H4} and \ref{ass:Z1},
    \begin{align}\label{eq:Z-cdf}
        &\sup_u |\widehat F_{T,\epshat}(u \mid Z_{t_T}) - \FZ(u \mid Z_{t_T})| \xrightarrow{p} 0,
    \end{align}
    and hence,
    \begin{align}
        \left| \P\left(Y_{t_T+\ell} \in \widehat C_T^{\alpha,\ell}(Z_{t_T}, \widehat Y_{t_T + \ell}) \mid Z_{t_T}\right) - (1-\alpha) \right| \xrightarrow{p} 0.
    \end{align}
\end{theorem}

\begin{proof}[Proof of Theorem~\ref{thm:Z-coverage}]
    Take regular conditional CDF versions in Assumption \ref{ass:Z1}. For each rational $u$, the equality in \eqref{eq:Z-sufficiency} holds outside a $P_Z$-null set. The countable union of these null sets, together with the preimage of the exceptional set in Assumption \ref{ass:H3}, is still $P_Z$-null. For every $z$ outside this union and every real $u$, choose rationals $u_n\downarrow u$. Right-continuity gives
    \begin{align*}
        \FZ(u \mid z) &= \lim_{n\to\infty}\FZ(u_n \mid z) = \lim_{n\to\infty}\FH(u_n\mid\operatorname{Enc}_{\theta'}(z))\\
        &= \FH(u \mid \operatorname{Enc}_{\theta'}(z)).
    \end{align*}
    Thus the equality holds simultaneously for all $u$, on one full-measure set of contexts. By stationarity and the definition of $\widehat F_{T,\epshat}$,
    \begin{align*}
        &\sup_u|\widehat F_{T,\epshat}(u\mid Z_{t_T})-\FZ(u\mid Z_{t_T})|\\
        &=\sup_u|\FhatH(u \mid H_{t_T}) - \FH(u \mid H_{t_T})| \xrightarrow{p}0,
    \end{align*}
    which proves \eqref{eq:Z-cdf}.
    
    Use $q_a^{H,\star}$, $D_T$, and $A_T$ from the proof of Corollary \ref{cor:H-coverage}. On the preceding ``good'' full-measure set, the identity of conditional CDFs and \eqref{eq:lcp-true-quantile-level} imply, for $0<a<b<1$,
    \begin{align*}
        &\P\{q_a^{H,\star}(H_{t_T})\le\epshat_{t_T+\ell} \le q_b^{H,\star}(H_{t_T})\mid Z_{t_T}\}\\
        &= \FZ(q_b^{H,\star}(H_{t_T})\mid Z_{t_T}) -\FZ(q_a^{H,\star}(H_{t_T})-\mid Z_{t_T})\\
        &= \FH(q_b^{H,\star}(H_{t_T})\mid H_{t_T}) -\FH(q_a^{H,\star}(H_{t_T})-\mid H_{t_T})=b-a.
    \end{align*}
    For $0<\eta<\min\{\alpha_l,1-\alpha_h,(\alpha_h-\alpha_l)/2\}$, condition the same event inclusions directly on $Z_{t_T}$ to obtain
    \begin{align*}
        1-\alpha-2\eta-\P(D_T>\eta\mid Z_{t_T})
        &\le \P(A_T\mid Z_{t_T})\\
        &\le 1-\alpha+2\eta+\P(D_T>\eta\mid Z_{t_T}).
    \end{align*}
    For any $\delta>0$, choose $\eta$ also with $2\eta<\delta$. Then
    \begin{align*}
        &\P\{|\P(A_T\mid Z_{t_T})-(1-\alpha)|>\delta\}\\
        &\le \P\{\P(D_T>\eta\mid Z_{t_T})>\delta-2\eta\} \le\frac{\P(D_T>\eta)}{\delta-2\eta} \longrightarrow 0.
    \end{align*}
    This proves the coverage assertion by conditioning the oracle events on $Z_{t_T}$.
\end{proof}

\section{Additional Details of Experiments}
\label{appendix:experiment_details}

\subsection{Datasets}
\label{appendix:datasets}

We evaluate the proposed methods and baselines on three real-world datasets that are commonly used as benchmarks in conformal prediction for time series. This section provides details of the datasets, and Table~\ref{tbl:dataset_summary} summarizes their key characteristics.

\paragraph{Air quality dataset.}

The Air Quality dataset~\citep{zhang2017cautionary} comprises 35,064 air quality measurements collected from 12 locations in Beijing, China. It includes two air quality variables, PM10 and PM2.5. We use PM10 as the outcome variable. PM2.5 is excluded from the feature set. The dataset also contains ten additional environmental variables that are used as features. This dataset has also been used in prior work on conformal prediction for time series~\citep{auer2024conformal,neglia2026rescp}.

\paragraph{Solar dataset.}

The Solar dataset is obtained from the U.S. National Solar Radiation Database~\citep{sengupta2018national} and contains half-hourly measurements of solar radiation, along with eight environmental features, at 50 locations. We use diffuse horizontal irradiance (DHI) as the outcome variable and the eight environmental variables as input features. We consider three years of data and resample the original half-hourly measurements to an hourly frequency. This results in 50 time series, each consisting of 26,304 hourly observations resampled from 52,609 half-hourly measurements. This dataset has also been used in prior works~\citep{xu2021conformal,xu2023conformal,auer2024conformal,neglia2026rescp}.

\paragraph{Sapflow dataset.}

Sap flow refers to the movement of water within a plant and is commonly used as a proxy for plant transpiration in environmental science. The Sapflow dataset is obtained from the SAPFLUXNET database~\citep{poyatos2020global}. Following~\citet{auer2024conformal}, we use a subset of 24 time series, each containing between 15,000 and 20,000 observations. The sequences have varying sampling frequencies, and each sequence contains 10 features.

\begin{table}[t]
\caption{Summary statistics of the three real-world datasets, including the number of features, sequence length, and number of time series.}
\label{tb:dataset_details}
\begin{center}
\begin{small}
\begin{tabular}{lccc}
\toprule
 & Number of features & Sequence length & Number of sequences  \\
\midrule
Air Quality & 10 &  35064 & 12 \\
Solar  & 8 & 26304 & 50 \\
Sapflow & 10 & 15007–19934 & 24 \\

\bottomrule
\end{tabular}
\end{small}
\end{center}
\label{tbl:dataset_summary}
\end{table}

\subsection{Base predictors}
\label{appendix:base_predictors}

\paragraph{Linear regression.}

We use \texttt{LinearRegressionModel} from the Python Darts~\citep{herzen2022darts} library to implement the linear regression base predictor. Each time series sequence is chronologically split into training and held-out sets according to the specified training ratio. A separate model is trained for each sequence. Predictions are made sequentially using a rolling window. At each prediction step, the model uses the most recent \texttt{past\_window} outcomes and covariates as input features to predict the next \texttt{prediction\_step} outcomes. The window is then advanced by \texttt{prediction\_step} time steps, incorporating the new true outcomes and covariates to make the next prediction. Covariates at the forecasting timestamps are not provided. \texttt{prediction\_step} corresponds to \texttt{output\_chunk\_length} hyperparameter in \texttt{LinearRegressionModel}. If $\texttt{prediction\_step}>1$, Darts fits a separate linear regressor for each prediction step. Training is conducted by minimizing the mean squared error between the predicted and true outcomes on the training set. After training, the model forecasts the held-out set of the sequence. We make the true covariates and outcomes available for use as windowed features when forecasting the held-out set. MSE and MAE are computed over each sequence's held-out set and then averaged across sequences. In all experiments, we set \texttt{past\_window} to 100 and \texttt{prediction\_step} to 1. Table~\ref{tbl:base_predictors_results} reports the MAE and RMSE of the linear regression base predictor on all datasets.

\paragraph{Recurrent neural networks.}

The recurrent neural network (RNN) base predictor is implemented as a global model trained jointly across all time series sequences in each dataset. Each time series sequence is chronologically split into training and held-out sets according to the specified training ratio, with the last $10\%$ of the training set is used as the validation set for model selection. The covariates and outcomes are then normalized using the training sets across all sequences in the dataset. Predictions are made sequentially using a rolling window. At each prediction step, the model uses the most recent \texttt{past\_window} outcomes and covariates as input features to predict the next \texttt{prediction\_step} outcomes. The window is then advanced by \texttt{prediction\_step} time steps, incorporating the new true outcomes and covariates to make the next prediction. Covariates at the forecasting timestamps are not provided. The model consists of a linear input embedding layer, followed by a multi-layer LSTM~\citep{hochreiter1997long}, and a final fully connected layer that maps the hidden representation to the predicted outcomes. During training, the model is trained by minimizing the mean squared error using the Adam optimizer~\citep{kingma2014adam}. After training, the model forecasts the held-out set of each sequence, and the predictions are transformed back to the original scale. We make the true covariates and outcomes available for use as windowed features when forecasting the held-out set. MSE and MAE are computed over each sequence's held-out set and then averaged across sequences. In all experiments, we set \texttt{window\_length} to 100 and \texttt{prediction\_step} to 1. Table~\ref{tbl:base_predictors_results} reports the MAE and RMSE of the LSTM base predictor on all datasets.

\paragraph{Chronos.}

We implement the Chronos base predictor using the pretrained
\texttt{amazon/chronos-2} checkpoint\footnote{\url{https://huggingface.co/amazon/chronos-2}} through the \texttt{Chronos2Pipeline}\footnote{\url{https://github.com/amazon-science/chronos-forecasting}}. We use the model in a zero-shot setting without fine-tuning it on our datasets.
Therefore, unlike the linear regression and RNN base predictors, entire sequence is used as the held-out set. Predictions are made sequentially using a rolling window. At each prediction step, the model uses the most recent \texttt{past\_window} outcomes and covariates as input features to predict the next \texttt{prediction\_step} outcomes. The window is then advanced by \texttt{prediction\_step} time steps, incorporating the new true outcomes and covariates to make the next prediction. Covariates at the forecasting timestamps are not provided. Although Chronos-2 can produce forecasts at flexible quantile levels, we use the median prediction, corresponding to the $0.5$ quantile, as the predicted outcome. MSE and MAE are computed over each sequence's held-out set and then averaged across sequences. In all experiments, we set \texttt{window\_length} to 100 and \texttt{prediction\_length} to 10.  Table~\ref{tbl:base_predictors_results} reports the MAE and RMSE of the Chronos base predictor on all datasets.

\begin{table}[t]
\caption{Mean absolute errors (MAE) and root mean squared errors (RMSE) of three base predictors on all datasets.}
\label{tbl:base_predictors_results}
\begin{center}
\begin{small}
\begin{tabular}{lcccccc}
\toprule
 & \multicolumn{2}{c}{Air} 
 & \multicolumn{2}{c}{Solar} 
 & \multicolumn{2}{c}{Sapflow} \\
 & MAE & RMSE & MAE & RMSE & MAE & RMSE \\
\midrule
Linear Regression 
& $22.1$ & $34.8$ 
& $26.3$ & $140.8$ 
& $161.2$ & $375.7$ \\

LSTM 
& $17.2$ & $30.3$ 
& $14.5$ & $33.3$ 
& $230.7$ & $761.1$ \\

Chronos 
& $39.7$ & $65.0$ 
& $24.6$ & $51.4$ 
& $420.2$ & $1007.9$ \\
\bottomrule
\end{tabular}
\end{small}
\end{center}
\end{table}

\begin{table}[t]
\caption{Hyperparameter search spaces used for the proposed methods. ``Use of current covariate'' indicates whether the covariate at the prediction time is included. For context features, \texttt{xr} denotes the concatenation of covariates and residuals, while \texttt{xry} denotes the concatenation of covariates, residuals, and observed outcomes.}
\label{table:proposed_hyperparameters_space}
\begin{center}
\begin{small}
\begin{tabular}{lll}
\toprule
\textbf{Method} & \textbf{Hyperparameter} & \textbf{Search space} \\
\midrule

\multirow{6}{*}{DSCP-QR-RNN}
& Model dimension & $32,\;64,\;128$ \\
& Number of layers & $1,\;2,\;4$ \\
& Use of current covariate & True, False \\
& Context window size & $100,\;200$ \\
& Prediction head hidden layers & $1,\;2$ \\
& Context features & \texttt{xr}, \texttt{xry} \\

\midrule

\multirow{7}{*}{DSCP-QR-Transformer}
& Model dimension & $32,\;64,\;128$ \\
& Number of heads & $2,\;4$ \\
& Number of layers & $2,\;4$ \\
& Use of current covariate & True, False \\
& Context window size & $100,\;200$ \\
& Prediction head hidden layers & $1,\;2$ \\
& Context features & \texttt{xr}, \texttt{xry} \\

\midrule

\multirow{6}{*}{DSCP-QFE-RNN}
& Model dimension & $32,\;64,\;128$ \\
& Number of layers & $1,\;2,\;4$ \\
& Use of current covariate & True, False \\
& Context window size & $100,\;200$ \\
& Prediction head hidden layers & $1,\;2$ \\
& Context features & \texttt{xr}, \texttt{xry} \\

\midrule

\multirow{7}{*}{DSCP-QFE-Transformer}
& Model dimension & $32,\;64,\;128$ \\
& Number of heads & $2,\;4$ \\
& Number of layers & $2,\;4$ \\
& Use of current covariate & True, False \\
& Context window size & $100,\;200$ \\
& Prediction head hidden layers & $1,\;2$ \\
& Context features & \texttt{xr}, \texttt{xry} \\

\midrule

\multirow{9}{*}{DSCP-Local-RNN}
& Model dimension & $32,\;64,\;128$ \\
& Number of layers & $1,\;2,\;4$ \\
& Use of current covariate & True, False \\
& Context window size & $100,\;200$ \\
& Prediction head hidden layers & $1,\;2$ \\
& Context features & \texttt{xr}, \texttt{xry} \\
& Calibration window size &  $100,\; 500$ \\
& Similarity function & \texttt{dot}, \texttt{cos}, \texttt{Euclidean} \\
& Temperature &  $0.1,\;1.0,\;10.0$ \\

\midrule

\multirow{10}{*}{DSCP-Local-Transformer}
& Model dimension & $32,\;64,\;128$ \\
& Number of heads & $2,\;4$ \\
& Number of layers & $2,\;4$ \\
& Use of current covariate & True, False \\
& Context window size & $100,\;200$ \\
& Prediction head hidden layers & $1,\;2$ \\
& Context features & \texttt{xr}, \texttt{xry} \\
& Calibration window size & $100,\; 500$ \\
& Similarity function & \texttt{dot}, \texttt{cos}, \texttt{Euclidean} \\
& Temperature &  $0.1,\;1.0,\;10.0$ \\

\bottomrule
\end{tabular}
\end{small}
\end{center}
\end{table}

\subsection{Implementation Details of Proposed Methods}
\label{appendix:proposed_methods}

\paragraph{DSCP-QR}

DSCP-QR consists of a linear projection layer that maps the input to the model dimension, a backbone sequence encoder, and quantile-specific prediction heads. Each head $g_{\theta^{\tau_j}}$ is a fully connected neural network with one or more hidden layers and ReLU activations. To prevent quantile crossing across the ordered target levels $0 < \tau_1 < \cdots < \tau_J < 1$, we parameterize successive quantiles using nonnegative increments:
\[
\begin{aligned}
\hat q_{T,\tau_1}(z_t)
&= g_{\theta^{\tau_1}}(h_t),\\
\hat q_{T,\tau_j}(z_t)
&= \hat q_{T,\tau_{j-1}}
+ \sigma\!\left(g_{\theta^{\tau_j}}(h_t)\right),
\quad j = 2,\ldots,J,
\end{aligned}
\]
where
$\sigma:\mathbb{R}\to[0,\infty)$ is a nonnegative activation function, such as ReLU or softplus. This construction guarantees
\[
\widehat{q}_{T,\tau_1}(z_t)
\leq \cdots \leq
\widehat{q}_{T,\tau_J}(z_t)
\]
for every context $z_t$. Without this constraint to prevent quantile crossing, each prediction head directly estimates its corresponding quantile:
\[
\widehat{q}_{T,\tau_j}(z_t)
= g_{\theta^{\tau_j}}(h_t),
\quad j = 1,\ldots,J.
\]

We use the first 75\% of the calibration set for training and the remaining 25\% for validation. Training is performed for at most 50 epochs with early stopping using a patience of three epochs, and the checkpoint with the lowest validation loss is selected. The validation loss is the same quantile loss used for training, evaluated at the target quantile levels. We use AdamW~\citep{loshchilov2017decoupled} with $\beta_1=0.9$, $\beta_2=0.999$, $\epsilon=10^{-8}$, and a weight decay of $0.01$. The initial learning rate is set to $10^{-3}$ for the RNN encoder and $10^{-4}$ for the Transformer encoder. To construct prediction intervals with target coverage levels of $0.8$, $0.9$, and $0.95$, we use the target quantile sets $\{0.10,0.90\}$, $\{0.05,0.95\}$, and $\{0.025,0.975\}$, respectively.

The best hyperparameter configuration is selected through a grid search over the search space specified in Table~\ref{table:proposed_hyperparameters_space}. Each configuration is trained and evaluated separately using a subset of sequences, with 20\% of the sequences in each dataset used for hyperparameter selection. For each selected sequence, we use a chronological split, with the first 85\% of the training data used for model fitting and the remaining 15\% used for validation.

Because we prioritize coverage, we adopt a two-stage hyperparameter selection procedure. First, we filter configurations based on empirical coverage. A configuration is considered eligible only if the difference between its empirical and nominal coverage is at least $-0.01$ for every target quantile pair and every selected sequence. Among the eligible configurations, we select the one with the lowest average Winkler score. The Winkler score is first averaged over evaluation observations within each sequence and then averaged equally across sequences. This coverage-first selection procedure is similar to the hyperparameter selection criterion used by \citet{auer2024conformal}. The best hyperparameter configurations selected for each dataset and base predictor are presented in Tables~\ref{table:proposed_hyperparameters_air}--\ref{table:proposed_hyperparameters_sapflow}.

\paragraph{DSCP-QFE}
\label{sec:monotonic_prediction_head}

DSCP-QFE consists of a linear projection layer that maps the input to the model dimension, a backbone sequence encoder, and a prediction head. The prediction head is a fully connected neural network with one or more hidden layers and ReLU activations. To prevent quantile crossing in conditional quantile function estimation, we use a partially monotonic prediction head. As defined in \S~\ref{sec:method_cqfe}, let
\[
h_t
=
\operatorname{Enc}_{\theta'}(z_t)
\in\mathbb{R}^{d_h}
\]
denote the latent representation of the historical context $z_t$.
The encoder does not take the quantile level as an input, so
$h_t$ remains fixed when evaluating different quantile levels
for the same context. The scalar quantile level $\tau\in(0,1)$
is provided directly to the prediction head, yielding
\begin{equation}
\widehat q_T(z_t,\tau)
=
g_{\theta''}(h_t,\tau)
=
g_{\theta''}\!\left(
\operatorname{Enc}_{\theta'}(z_t),\tau
\right).
\end{equation}

Let $K\geq 1$ denote the number of hidden layers in the
prediction head, and let $d_k$ denote the width of
the $k$-th hidden layer. The hidden representation from the first hidden layer $r_t^{(1)}(\tau)\in\mathbb{R}^{d_1}$ is defined by
\begin{equation}
\label{eq:pm-head-first-layer}
r_t^{(1)}(\tau)
=
\phi_1\!\left(
A_1h_t+b_1^{+}\tau+c_1
\right),
\end{equation}
where
$A_1\in\mathbb{R}^{d_1\times d_h}$, $b_1^{+}\in[0,\infty)^{d_1}$, and $c_1\in\mathbb{R}^{d_1}$.
For $k=2,\ldots,K$, the hidden representations from the subsequent hidden layers $r_t^{(k)}(\tau)\in\mathbb{R}^{d_k}$ are defined recursively by
\begin{equation}
\label{eq:pm-head-hidden-layers}
r_t^{(k)}(\tau)
=
\phi_k\!\left(
A_kh_t+B_k^{+}r_t^{(k-1)}(\tau)+c_k
\right),
\qquad k=2,\ldots,K,
\end{equation}
where $A_k\in\mathbb{R}^{d_k\times d_h}$, $B_k^{+}\in[0,\infty)^{d_k\times d_{k-1}}$, and $c_k\in\mathbb{R}^{d_k}$. When $K=1$, the recursive layers in \eqref{eq:pm-head-hidden-layers} are omitted.

The output of the prediction head is
\begin{equation}
\label{eq:pm-head-output}
g_{\theta''}(h_t,\tau)
=
a(h_t)+(v^{+})^{\top}r_t^{(K)}(\tau),
\end{equation}
where $v^{+}\in[0,\infty)^{d_K}$ and $a:\mathbb{R}^{d_h}\rightarrow\mathbb{R}$ is a continuous scalar-valued function that does not depend on $\tau$. We use a linear mapping,
\begin{equation}
a(h_t)=w_a^{\top}h_t+b_a,
\qquad
w_a\in\mathbb{R}^{d_h},
\quad
b_a\in\mathbb{R}.
\end{equation}

Each $\phi_k:\mathbb{R}\rightarrow\mathbb{R}$ is a continuous,
nondecreasing activation function applied componentwise, such
as $\tanh$ or sigmoid. The matrices $A_k$ and bias vectors $c_k$ have no sign constraints. In contrast, the weight vectors $b_1^{+}$ and $v^{+}$ and the weight matrices $B_k^{+}$ are elementwise nonnegative. These constraints preserve the ordering of the quantile-dependent hidden representations throughout the network. Neither the hidden representations nor the final quantile estimates are required to be nonnegative.

The context representation enters each hidden layer through
$A_kh_t$. These context-dependent terms and $a(h_t)$
remain unchanged when computing different quantile levels
at the same context and therefore do not affect the
monotonicity guarantee. To verify this property, fix a context representation $h_t$ and consider any $0<\tau_1<\tau_2<1$. Since $b_1^{+}$ is
elementwise nonnegative and $\phi_1$ is nondecreasing,
\[
r_t^{(1)}(\tau_1)
\leq
r_t^{(1)}(\tau_2)
\]
componentwise. At each subsequent layer, multiplication by
the elementwise nonnegative matrix $B_k^{+}$, addition of the
common term $A_kh_t+c_k$, and application of the nondecreasing
activation $\phi_k$ preserve this ordering. Hence,
$$
r_t^{(k)}(\tau_1)
\leq
r_t^{(k)}(\tau_2),
\qquad k=1,\ldots,K,
$$
componentwise. Finally, the nonnegativity of $v^{+}$ implies

$$
g_{\theta''}(h_t,\tau_2)
-
g_{\theta''}(h_t,\tau_1)
=
(v^{+})^{\top}
\left[
r_t^{(K)}(\tau_2)-r_t^{(K)}(\tau_1)
\right]
\geq 0.
$$
Therefore, for every context $z_t$,
$$
\widehat q_T(z_t,\tau_1)
\leq
\widehat q_T(z_t,\tau_2),
\qquad
0<\tau_1<\tau_2<1.
$$
Thus, the estimated conditional quantile function is nondecreasing in the quantile level, preventing quantile crossing by construction.

If we do not use the partially monotonic prediction head, we use quantile embedding and fuse it with the encoded representation from the sequence encoder following \citet{dabney2018implicit,gouttes2021probabilistic}. Specifically, for each quantile level $\tau\in(0,1)$, we construct a vector of $K$ cosine features and apply a learned affine transformation followed by a ReLU activation:
\[
\begin{aligned}
\mathbf{c}(\tau)
&= \bigl[\cos(\pi i\tau)\bigr]_{i=0}^{K-1},\\
\phi(\tau)
&= \operatorname{ReLU}\!\left(
    \mathbf{W}_{\phi}\mathbf{c}(\tau) + \mathbf{b}_{\phi}
\right),
\end{aligned}
\]
where $\mathbf{W}_{\phi}\in\mathbb{R}^{d_h\times K}$ and $\mathbf{b}_{\phi}\in\mathbb{R}^{d_h}$ are learned parameters, and $d_h$ is the dimension of the encoded representation $\mathbf{h}_t$. We then fuse the quantile embedding with $\mathbf{h}_t$ through residual element-wise multiplication:
\[
\mathbf{h}_{t,\tau}
= \mathbf{h}_t \odot \bigl(\mathbf{1}+\phi(\tau)\bigr)
= \mathbf{h}_t + \mathbf{h}_t \odot \phi(\tau),
\]
where $\odot$ denotes element-wise multiplication and $\mathbf{1}$ is the all-ones vector. This formulation retains an identity pathway while allowing quantile-dependent modulation of the encoded features. The fused representation is fed into the prediction head. The encoder, quantile embedding, and prediction head are trained jointly. Unlike the partially monotonic head, this construction does not guarantee to prevent quantile crossing.

We use the first 75\% of the calibration set for training and the remaining 25\% for validation. The model is trained for at most 50 epochs with early stopping using a patience of three epochs, and the checkpoint with the lowest validation loss is selected. Training minimizes the quantile loss evaluated at 64 quantile levels sampled for each training observation. Increasing the number of sampled quantile levels did not affect performance in our experiments. For validation, we evaluate the same quantile loss at the target quantile sets $\{0.10,0.90\}$, $\{0.05,0.95\}$, and $\{0.025,0.975\}$, corresponding to target coverage levels of $0.8$, $0.9$, and $0.95$, respectively. We use AdamW~\citep{loshchilov2017decoupled} with $\beta_1=0.9$, $\beta_2=0.999$, $\epsilon=10^{-8}$, and a weight decay of $0.01$. The initial learning rate is set to $10^{-3}$ for the RNN encoder and $10^{-4}$ for the Transformer encoder. The best hyperparameter configuration is selected using the same grid-search procedure as for DSCP-QR, over the search space specified in Table~\ref{table:proposed_hyperparameters_space}. The best hyperparameter configurations selected for each dataset and base predictor are presented in Tables~\ref{table:proposed_hyperparameters_air}--\ref{table:proposed_hyperparameters_sapflow}.

\paragraph{DSCP-Local}

DSCP-Local constructs prediction intervals using contextual representations from the sequence encoder rather than directly using prediction heads. This design allows flexibility in the choice of the encoder training objective. To encourage the learned representations to capture information about the conditional distribution of prediction residuals, we adopt the same architecture as DSCP-QR and train the model using the quantile loss. The quantile-specific prediction heads support representation learning during training but are not used to construct prediction intervals; instead, DSCP-Local uses only the learned encoder representations.

For similarity computation, we consider dot product, cosine similarity,
and squared Euclidean distance, with a temperature parameter controlling
the concentration of calibration weights. Given a query representation
$h_t$ and a calibration representation $h_i$, we compute
\[
m_i =
\begin{cases}
T h_t^\top h_i, & \text{dot product}, \\[4pt]
\displaystyle T \frac{h_t^\top h_i}{\|h_t\|_2 \|h_i\|_2},
& \text{cosine similarity}, \\[8pt]
\displaystyle -\frac{\|h_t-h_i\|_2^2}{T},
& \text{squared Euclidean distance},
\end{cases}
\]
where $T>0$. We normalize these scores across the calibration window $\mathcal{C}$ using softmax:
\[
w_i = \frac{\exp(m_i)}{\sum_{j \in \mathcal{C}} \exp(m_j)}.
\]
The resulting weights define a sampling distribution over calibration residuals, from which we estimate the residual quantiles used to construct prediction intervals. The experiments use softmax weights, while the DSCP-Local theoretical analysis uses a radial kernel estimator under Assumptions \ref{ass:H1}--\ref{ass:H4}. The experimental softmax variants are therefore not directly covered by the stated compact-support and bandwidth conditions. Both formulations use representation-based weighting, and our analysis identifies sufficient conditions under which this principle yields asymptotic conditional coverage.

We split the calibration set chronologically, using the first 60\% for training, the next 15\% for validation, and the remaining 25\% as a held-out calibration subset. This held-out subset is used to define a similarity-weighted sampling distribution. We train the model for at most 50 epochs with early stopping using a patience of three epochs and select the checkpoint with the lowest validation loss. The validation loss is the same quantile loss used for training, evaluated at the target quantile levels. We compare three choices of target quantile levels: the median alone ($\{0.5\}$), three levels ($\{0.1,0.5,0.9\}$), and nine levels ($\{0.1,0.2,0.3,0.4,0.5,0.6,0.7,0.8,0.9\}$). As shown in Table~\ref{table:ablation_local}, using multiple quantile levels yields modest performance improvements over median-only training, while increasing the number of levels from three to nine produces no noticeable further improvement. We therefore train the encoder using the three quantile levels $\{0.1,0.5,0.9\}$ in all experiments.

We use AdamW~\citep{loshchilov2017decoupled} with $\beta_1=0.9$, $\beta_2=0.999$, $\epsilon=10^{-8}$, and a weight decay of $0.01$. The initial learning rate is set to $10^{-3}$ for the RNN encoder and $10^{-4}$ for the Transformer encoder. To construct prediction intervals with target coverage levels of $0.8$, $0.9$, and $0.95$, we use the target quantile sets $\{0.10,0.90\}$, $\{0.05,0.95\}$, and $\{0.025,0.975\}$, respectively. The best hyperparameter configuration is selected using the same grid-search procedure as for DSCP-QR, over the search space specified in Table~\ref{table:proposed_hyperparameters_space}. The best hyperparameter configurations selected for each dataset and base predictor are presented in Tables~\ref{table:proposed_hyperparameters_air}--\ref{table:proposed_hyperparameters_sapflow}.

\begin{table}[t]
\captionsetup{justification=raggedright,singlelinecheck=false}
\caption{Selected hyperparameters for the proposed methods on the Air dataset across different base predictors.}
\label{table:proposed_hyperparameters_air}
\begin{center}
\begin{small}
\begin{tabular}{lllll}
\toprule
\textbf{Method} & \textbf{Hyperparameter} & \textbf{LR} & \textbf{LSTM} & \textbf{Chronos} \\
\midrule

\multirow{6}{*}{DSCP-QR-RNN}
& Model dimension & 64 & 64 & 32 \\
& Number of layers & 1 & 1 & 2 \\
& Use of current covariate & True & True &  True\\
& Context window size & 200 & 200 & 200 \\
& Prediction head hidden layers & 1 & 1 & 1 \\
& Context features & \texttt{xry} &  \texttt{xr} &   \texttt{xry}\\

\midrule

\multirow{7}{*}{DSCP-QR-Transformer}
& Model dimension & 64 &  32& 64\\
& Number of heads & 4 & 2 &2 \\
& Number of layers & 2 & 4 & 2\\
& Use of current covariate & True & False & False \\
& Context window size & 200 & 100 & 200 \\
& Prediction head hidden layers & 1 & 1 & 1\\
& Context features & \texttt{xry} & \texttt{xr} & \texttt{xry}\\

\midrule

\multirow{6}{*}{DSCP-QFE-RNN}
& Model dimension & 64 & 64 & 64 \\
& Number of layers & 1 & 1 & 1 \\
& Use of current covariate & True & True & True\\
& Context window size & 100 & 100 & 100\\
& Prediction head hidden layers & 1 &1 &1 \\
& Context features & \texttt{xry} & \texttt{xry} & \texttt{xry} \\

\midrule

\multirow{7}{*}{DSCP-QFE-Transformer}
& Model dimension & 64 & 64 & 64\\
& Number of heads & 2& 2& 2\\
& Number of layers &2 & 2& 2\\
& Use of current covariate &False & False & False \\
& Context window size & 100& 100& 100 \\
& Prediction head hidden layers & 1& 1& 1\\
& Context features & \texttt{xry} & \texttt{xry} &\texttt{xry} \\

\midrule

\multirow{9}{*}{DSCP-Local-RNN}
& Model dimension & 64 & 64 & 32 \\
& Number of layers & 2 & 2 & 1 \\
& Use of current covariate & False & False &  False\\
& Context window size & 200 & 200 & 200 \\
& Prediction head hidden layers & 1 & 1 & 1 \\
& Context features & \texttt{xr} &  \texttt{xr} &   \texttt{xry}\\
& Calibration window size  & 500 & 500 & 500 \\
& Similarity function & \texttt{cos} & \texttt{cos} & \texttt{cos} \\
& Temperature & 10.0 & 10.0& 10.0\\

\midrule

\multirow{10}{*}{DSCP-Local-Transformer}
& Model dimension &32 & 64& 32\\
& Number of heads & 2& 4& 4\\
& Number of layers & 4& 2&4 \\
& Use of current covariate &True  & False &  False\\
& Context window size & 200 & 200 &200 \\
& Prediction head hidden layers &1 &1 & 1\\
& Context features & \texttt{xry} & \texttt{xry}& \texttt{xry}\\
& Calibration window size  & 500 & 500 & 500 \\
& Similarity function & \texttt{cos} & \texttt{cos} & \texttt{cos} \\
& Temperature & 10.0 & 10.0& 10.0\\

\bottomrule
\end{tabular}
\end{small}
\end{center}
\end{table}


\begin{table}[t]
\captionsetup{justification=raggedright,singlelinecheck=false}
\caption{Selected hyperparameters for the proposed methods on the Solar dataset across different base predictors.}
\label{table:proposed_hyperparameters_solar}
\begin{center}
\begin{small}
\begin{tabular}{lllll}
\toprule
\textbf{Method} & \textbf{Hyperparameter} & \textbf{LR} & \textbf{LSTM} & \textbf{Chronos} \\
\midrule

\multirow{6}{*}{DSCP-QR-RNN}
& Model dimension & 64 & 32 & 32 \\
& Number of layers & 1 & 1 & 1 \\
& Use of current covariate & True & True & True\\
& Context window size & 200 & 200 & 100\\
& Prediction head hidden layers & 1 & 1 & 1\\
& Context features & \texttt{xry} & \texttt{xry}  & \texttt{xry} \\

\midrule

\multirow{7}{*}{DSCP-QR-Transformer}
& Model dimension & 64 & 64 & 64\\
& Number of heads & 4& 4 & 4\\
& Number of layers &2 &4 & 2\\
& Use of current covariate & True &  False & False\\
& Context window size & 200 & 100 & 200\\
& Prediction head hidden layers &1 & 1& 1\\
& Context features & \texttt{xry} & \texttt{xr} & \texttt{xry} \\

\midrule

\multirow{6}{*}{DSCP-QFE-RNN}
& Model dimension & 64 & 64 & 64 \\
& Number of layers & 1 & 1 & 1 \\
& Use of current covariate & True & True & True\\
& Context window size & 100 & 100 & 100\\
& Prediction head hidden layers & 1 &1 &1 \\
& Context features & \texttt{xry} & \texttt{xry} & \texttt{xry} \\

\midrule

\multirow{7}{*}{DSCP-QFE-Transformer}
& Model dimension & 64 & 64 & 64\\
& Number of heads &2 &2 & 2\\
& Number of layers & 2&2 &2 \\
& Use of current covariate & False & False & False\\
& Context window size & 100 & 100 & 100\\
& Prediction head hidden layers & 1&1 &1 \\
& Context features & \texttt{xry} & \texttt{xry} & \texttt{xry}\\

\midrule

\multirow{9}{*}{DSCP-Local-RNN}
& Model dimension & 64 &64  & 32 \\
& Number of layers & 2 & 2 & 1\\
& Use of current covariate & True & False & False\\
& Context window size & 200 & 200 & 200 \\
& Prediction head hidden layers & 1 & 1 & 1\\
& Context features & \texttt{xr} & \texttt{xry}& \texttt{xr}\\
& Calibration window size  & 500 & 500 & 500 \\
& Similarity function & \texttt{cos} & \texttt{cos} & \texttt{cos} \\
& Temperature & 10.0 & 10.0& 10.0\\

\midrule

\multirow{10}{*}{DSCP-Local-Transformer}
& Model dimension & 32& 64 & 32\\
& Number of heads & 2& 4& 4\\
& Number of layers &4 & 2 & 4\\
& Use of current covariate & True & False & False\\
& Context window size &200 &200 & 200 \\
& Prediction head hidden layers & 1&1 &1 \\
& Context features & \texttt{xry} & \texttt{xry}& \texttt{xry}\\
& Calibration window size  & 500 & 500 & 500 \\
& Similarity function & \texttt{cos} & \texttt{cos} & \texttt{cos} \\
& Temperature & 10.0 & 10.0& 10.0\\

\bottomrule
\end{tabular}
\end{small}
\end{center}
\end{table}


\begin{table}[t]
\captionsetup{justification=raggedright,singlelinecheck=false}
\caption{Selected hyperparameters for the proposed methods on the Sapflow dataset across different base predictors.}
\label{table:proposed_hyperparameters_sapflow}
\begin{center}
\begin{small}
\begin{tabular}{lllll}
\toprule
\textbf{Method} & \textbf{Hyperparameter} & \textbf{LR} & \textbf{LSTM} & \textbf{Chronos} \\
\midrule

\multirow{6}{*}{DSCP-QR-RNN}
& Model dimension & 64 & 32 & 32\\
& Number of layers & 1 & 1 & 1\\
& Use of current covariate & False & False & False\\
& Context window size & 100 & 100 &  100\\
& Prediction head hidden layers &  1&  1& 1 \\
& Context features & \texttt{xr} & \texttt{xr} & \texttt{xr} \\

\midrule

\multirow{7}{*}{DSCP-QR-Transformer}
& Model dimension & 64 &64  & 64 \\
& Number of heads & 4& 4& 4\\
& Number of layers & 2& 2& 2\\
& Use of current covariate &False  &False &False \\
& Context window size & 200 & 200 & 200 \\
& Prediction head hidden layers & 1& 1&1 \\
& Context features & \texttt{xr} & \texttt{xr}& \texttt{xr}\\

\midrule

\multirow{6}{*}{DSCP-QFE-RNN}
& Model dimension & 64 & 64 & 64 \\
& Number of layers & 1 & 1 & 1 \\
& Use of current covariate & False & False & False\\
& Context window size & 100 & 100 & 100\\
& Prediction head hidden layers & 1 &1 &1 \\
& Context features & \texttt{xry} & \texttt{xry} & \texttt{xry} \\

\midrule

\multirow{7}{*}{DSCP-QFE-Transformer}
& Model dimension & 64 & 64 & 64\\
& Number of heads &2 &2 & 4\\
& Number of layers & 2&2 &2 \\
& Use of current covariate & False & False & False\\
& Context window size & 100 & 100 & 100\\
& Prediction head hidden layers & 1&1 &1 \\
& Context features & \texttt{xry} & \texttt{xry} & \texttt{xry}\\

\midrule

\multirow{9}{*}{DSCP-Local-RNN}
& Model dimension & 64 & 64 & 32\\
& Number of layers &2 & 2& 1\\
& Use of current covariate & False & False & False\\
& Context window size & 200 & 200 & 200\\
& Prediction head hidden layers &1 &1 & 1\\
& Context features & \texttt{xr} & \texttt{xry}& \texttt{xr}\\
& Calibration window size  & 500 & 500 & 500 \\
& Similarity function & \texttt{cos} & \texttt{cos} & \texttt{cos} \\
& Temperature & 10.0 & 10.0& 10.0\\

\midrule

\multirow{10}{*}{DSCP-Local-Transformer}
& Model dimension & 32& 64& 32\\
& Number of heads & 2& 4& 4\\
& Number of layers & 4& 2& 4\\
& Use of current covariate & True & False & False\\
& Context window size &200 &200 & 200 \\
& Prediction head hidden layers &1 & 1& 1\\
& Context features & \texttt{xry} & \texttt{xry}& \texttt{xry}\\
& Calibration window size  & 500 & 500 & 500 \\
& Similarity function & \texttt{cos} & \texttt{cos} & \texttt{cos} \\
& Temperature & 10.0 & 10.0& 10.0\\

\bottomrule
\end{tabular}
\end{small}
\end{center}
\end{table}

\begin{table}[t]
\caption{Hyperparameter search spaces used for the baseline methods.}
\label{table:baseline_hyperparameters}
\begin{center}
\begin{small}
\begin{tabular}{lll}
\toprule
\textbf{Method} & \textbf{Hyperparameter} & \textbf{Search space} \\
\midrule

\multirow{2}{*}{NexCP}
& Decay factor $\rho$
& $0.9,\; 0.95,\; 0.99,\; 0.999$ \\
& Maximum past-window length
& $100,\; 200,\; 500$ \\

\midrule
\multirow{3}{*}{SPCI}
& Number of estimators 
& $10,\; 20,\; 50$ \\
& Max depth
& $2,\; 5,\; 10$ \\
& Window size
& $100,\; 200,\; 500$ \\

\midrule
\multirow{2}{*}{KOWCPI}
& Past-window length
& $5,\; 10,\; 20,\; 50,\; 100,\; 200$ \\
& Bandwidth
& $1,\; 5,\; 10,\; 20,\; 50,\; 100$ \\

\midrule
\multirow{5}{*}{ResCP}
& Spectral radius $\gamma$
& $0.5,\; 0.9,\; 1.2,\; 1.5$ \\
& Leak rate
& $0.5,\; 0.8,\; 1.0$ \\
& Input scaling
& $0.1,\; 0.25,\; 0.75,\; 1.5$ \\
& Temperature
& $0.03,\; 0.1,\; 0.3$ \\
& Calibration size
& $100,\; 500,\; 1500,\; 4000$ \\

\midrule
\multirow{2}{*}{DistMatch}
& KS threshold
& $0.01,\; 0.05,\; 0.1$ \\
& Past-window length
& $100$ \\

\bottomrule
\end{tabular}
\end{small}
\end{center}
\end{table}


\subsection{Implementation Details of Baseline Methods}
\label{appendix:baselines}

\paragraph{NexCP~\citep{barber2023conformal}}

We use the calibration set to calibrate the prediction intervals. For hyperparameter selection, we further split the calibration set into training and validation subsets using an 85\% / 15\% split. The best hyperparameter configuration is selected using the same grid-search procedure as for the proposed methods, over the search space specified in Table~\ref{table:baseline_hyperparameters}. The best hyperparameter configurations selected for each dataset and base predictor are presented in Tables~\ref{table:baseline_hyperparameters_air}--\ref{table:baseline_hyperparameters_sapflow}.

\paragraph{HopCPT~\citep{auer2024conformal}}

Because the datasets used in the original HopCPT experiments include the three datasets used in our study (Air, Solar, and Sapflow), we follow the hyperparameter selection procedure described in the original paper. We split the calibration set equally into training and validation subsets.

\paragraph{SPCI~\citep{xu2023sequential}}

We calibrate SPCI prediction intervals using the calibration set. For hyperparameter selection, we divide this set into training and validation subsets containing 85\% and 15\% of the observations, respectively. The best hyperparameter configuration is selected using the same grid-search procedure as for the proposed methods, over the search space specified in Table~\ref{table:baseline_hyperparameters}. The best hyperparameter configurations selected for each dataset and base predictor are presented in Tables~\ref{table:baseline_hyperparameters_air}--\ref{table:baseline_hyperparameters_sapflow}.

In its original formulation, SPCI retrains its random forest regressor at each prediction step, which can be computationally prohibitive for larger datasets. \citet{auer2024conformal} identify the same computational bottleneck and adopt a modified implementation that trains the regressor only on the calibration set, reporting only a small performance decrease in experiments on their smallest dataset. We therefore adopt the same modification and keep the regressor fixed during prediction. This choice is also consistent with the other methods in our comparison, whose models are not retrained after calibration.

\paragraph{KOWCPI~\citep{lee2024kernel}}

We use calibration set to calibrate the prediction intervals. For hyperparameter selection, we divide this set into training and validation subsets containing 85\% and 15\% of the observations, respectively. We use Epanechnikov kernel for all experiments. The best hyperparameter configuration is selected using the same grid-search procedure as for the proposed methods, over the search space specified in Table~\ref{table:baseline_hyperparameters}. The best hyperparameter configurations selected for each dataset and base predictor are presented in Tables~\ref{table:baseline_hyperparameters_air}--\ref{table:baseline_hyperparameters_sapflow}.

\paragraph{ResCP~\citep{neglia2026rescp}}

We calibrate the prediction intervals using the calibration set. Because the Air and Solar datasets were also evaluated in the original publication, we adopt the hyperparameter configurations reported for these datasets and perform hyperparameter selection only for Sapflow. For this selection, we divide the Sapflow calibration set into training and validation subsets containing 85\% and 15\% of the observations, respectively. We select the best configuration using the same grid-search procedure as for the proposed methods, over the search space specified in Table~\ref{table:baseline_hyperparameters}. The hyperparameter configurations used for each dataset and base predictor are reported in Tables~\ref{table:baseline_hyperparameters_air}--\ref{table:baseline_hyperparameters_sapflow}.

\paragraph{DistMatch~\citep{menadjiev2026distmatch}}

We tune the KS matching threshold $\gamma \in \{0.01,0.05,0.1\}$ while fixing the residual patch length at $w=100$, following the authors experimental setup. The number of matching trees, bootstrap sampling ratio, and minimum node-size parameter are fixed at $10$, $0.9$, and $0$, respectively, as specified in Appendix of \citet{menadjiev2026distmatch}. We select the best configuration using the same grid-search procedure as for the proposed methods, over the search space specified in Table~\ref{table:baseline_hyperparameters}. The hyperparameter configurations used for each dataset and base predictor are reported in Tables~\ref{table:baseline_hyperparameters_air}--\ref{table:baseline_hyperparameters_sapflow}.


\begin{table}[t]
\caption{Selected hyperparameters for the baseline methods across different base predictors on Air dataset.}
\label{table:baseline_hyperparameters_air}
\begin{center}
\begin{small}
\begin{tabular}{lllll}
\toprule
\textbf{Method} 
& \textbf{Hyperparameter} 
& \textbf{LR} 
& \textbf{LSTM} 
& \textbf{Chronos} \\
\midrule

\multirow{2}{*}{NexCP}
& Decay factor $\rho$
& 0.99 & 0.999 & 0.99 \\
& Maximum past-window length
&  500 & 200 & 500  \\

\midrule

\multirow{3}{*}{SPCI}
& Number of estimators 
&  50 & 50 & 50 \\
& Max depth
& 10 & 10 & 10 \\
& Window size
& 100 & 100 & 200 \\

\midrule

\multirow{1}{*}{KOWCPI}
& Past-window length
& 100 & 100 & 100 \\

\midrule

\multirow{5}{*}{ResCP}
& Spectral radius $\gamma$
& 1.45  & 1.3 &  1.0 \\
& Leak rate
& 0.65 & 0.95 & 0.8 \\
& Input scaling
& 0.75 & 0.25 &  0.25 \\
& Temperature
& 0.15 & 0.1 &  0.1 \\
& Calibration size
& 3800 & 3200 &  10500 \\

\midrule

\multirow{2}{*}{DistMatch}
& KS threshold
& 0.1  & 0.1 & 0.1 \\
& Past-window length
& 100 & 100 & 100 \\

\bottomrule
\end{tabular}
\end{small}
\end{center}
\end{table}


\begin{table}[t]
\caption{Selected hyperparameters for the baseline methods across different base predictors on Solar dataset.}
\label{table:baseline_hyperparameters_solar}
\begin{center}
\begin{small}
\begin{tabular}{lllll}
\toprule
\textbf{Method} 
& \textbf{Hyperparameter} 
& \textbf{LR} 
& \textbf{LSTM} 
& \textbf{Chronos} \\
\midrule

\multirow{2}{*}{NexCP}
& Decay factor $\rho$
& 0.999 & 0.99 & 0.99 \\
& Maximum past-window length
& 100 & 100 &  100 \\

\midrule

\multirow{3}{*}{SPCI}
& Number of estimators 
&  50 & 50 & 50 \\
& Max depth
& 10 & 10 & 10 \\
& Window size
& 100 & 100 & 100 \\

\midrule

\multirow{1}{*}{KOWCPI}
& Past-window length
& 100 & 100 & 100 \\

\midrule

\multirow{5}{*}{ResCP}
& Spectral radius $\gamma$
& 1.0 &  0.9 &  0.9 \\
& Leak rate
& 0.8 & 0.75 & 0.8 \\
& Input scaling
& 0.2 & 0.7 & 0.4 \\
& Temperature
& 0.1 & 0.1 & 0.1 \\
& Calibration size
& 1500 & 3900 & 7600 \\

\midrule

\multirow{2}{*}{DistMatch}
& KS threshold
& 0.1  & 0.1 & 0.1 \\
& Past-window length
& 100 & 100 & 100 \\

\bottomrule
\end{tabular}
\end{small}
\end{center}
\end{table}


\begin{table}[t]
\caption{Selected hyperparameters for the baseline methods across different base predictors on Sapflow dataset.}
\label{table:baseline_hyperparameters_sapflow}
\begin{center}
\begin{small}
\begin{tabular}{lllll}
\toprule
\textbf{Method} 
& \textbf{Hyperparameter} 
& \textbf{LR} 
& \textbf{LSTM} 
& \textbf{Chronos} \\
\midrule

\multirow{2}{*}{NexCP}
& Decay factor $\rho$
& 0.99  & 0.99 &  0.99 \\
& Maximum past-window length
& 500 & 200 &  200 \\

\midrule

\multirow{3}{*}{SPCI}
& Number of estimators 
&  50 & 50 & 50 \\
& Max depth
& 10 & 10 & 10 \\
& Window size
& 200 & 200 & 200 \\

\midrule

\multirow{1}{*}{KOWCPI}
& Past-window length
& 100 & 100 & 100 \\

\midrule

\multirow{5}{*}{ResCP}
& Spectral radius $\gamma$
& 1.2 & 1.2 & 1.2 \\
& Leak rate
& 0.9 &0.9  & 0.9 \\
& Input scaling
& 0.25 & 0.25 & 0.25 \\
& Temperature
& 0.1 & 0.1 & 0.1 \\
& Calibration size
& 3800 & 3800 & 3800 \\

\midrule

\multirow{2}{*}{DistMatch}
& KS threshold
& 0.1  & 0.1 & 0.1 \\
& Past-window length
& 100 & 100 & 100 \\

\bottomrule
\end{tabular}
\end{small}
\end{center}
\end{table}

\subsection{Evaluation Metrics}
\label{appendix:eval_metrics}

Given a prediction interval at time $t$,
\[
\widehat{C}^{\alpha}_t
=
\widehat{C}^{\alpha}
\left(
\hat{q}_t^{\alpha_l},
\hat{q}_t^{\alpha_h},
\hat{y}_t
\right)
=
\left[
\hat{y}_t + \hat{q}_t^{\alpha_l},
\hat{y}_t + \hat{q}_t^{\alpha_h}
\right],
\]
we denote its lower and upper endpoints by
\[
L_t = \hat{y}_t + \hat{q}_t^{\alpha_l},
\qquad
U_t = \hat{y}_t + \hat{q}_t^{\alpha_h},
\]
respectively. Here, $\alpha$ denotes the significance level, such that the target coverage is $1-\alpha$. Let $\mathcal{D}_{\mathrm{test}}$ denote the set of observations in a test sequence. We evaluate the prediction intervals using the following metrics.

\paragraph{Coverage gap ($\Delta \mathrm{Cov}$).}

The coverage gap for a test sequence at significance level $\alpha$ is
defined as
\begin{equation}
\Delta \mathrm{Cov}
=
\frac{1}{\lvert \mathcal{D}_{\mathrm{test}} \rvert}
\sum_{t \in \mathcal{D}_{\mathrm{test}}}
\mathds{1}
\left(
y_t \in \widehat{C}^{\alpha}_t
\right)
-
(1-\alpha).
\end{equation}
The coverage gap measures the difference between the empirical coverage
and the target coverage $1-\alpha$. A positive value indicates
overcoverage, whereas a negative value indicates undercoverage.
For each dataset, we report the coverage gap averaged across all test
sequences.

\paragraph{Average prediction interval width (PI-W).}

The average prediction interval width for a test sequence is defined as
\begin{equation}
\mathrm{PI\text{-}W}
=
\frac{1}{\lvert \mathcal{D}_{\mathrm{test}} \rvert}
\sum_{t \in \mathcal{D}_{\mathrm{test}}}
\left(
U_t - L_t
\right)
=
\frac{1}{\lvert \mathcal{D}_{\mathrm{test}} \rvert}
\sum_{t \in \mathcal{D}_{\mathrm{test}}}
\left(
\hat{q}_t^{\alpha_h}
-
\hat{q}_t^{\alpha_l}
\right).
\end{equation}
A smaller interval width indicates a more efficient prediction interval,
provided that the target coverage is obtained. For each dataset, we report the average prediction interval width across all test sequences.

\paragraph{Winkler score (WS).}

The Winkler score~\citep{winkler1972decision} jointly evaluates prediction
interval width and coverage. It is equal to the interval width when the
true outcome falls within the prediction interval and adds a penalty when
the outcome falls outside the interval. The penalty is proportional to the distance between the true outcome and the nearest interval endpoint, with a factor of $2/\alpha$. At time $t$, the Winkler score is defined as
\begin{equation}
\textrm{WS}_t =
\begin{cases}
(U_t - L_t)
+ \dfrac{2}{\alpha}(L_t - y_t),
& \text{if } y_t < L_t, \\[6pt]
(U_t - L_t),
& \text{if } L_t \leq y_t \leq U_t, \\[6pt]
(U_t - L_t)
+ \dfrac{2}{\alpha}(y_t - U_t),
& \text{if } y_t > U_t.
\end{cases}
\end{equation}
The average Winkler score for a test sequence is then given by
\begin{equation}
\mathrm{WS}
=
\frac{1}{\lvert \mathcal{D}_{\mathrm{test}} \rvert}
\sum_{t \in \mathcal{D}_{\mathrm{test}}} \textrm{WS}_t.
\end{equation}
Lower Winkler scores indicate smaller prediction intervals with fewer
or less severe coverage violations. For each dataset, we report the
Winkler score averaged across all test sequences.

\begin{table}[!htbp]
\centering
\caption{Coverage gap ($\Delta$Cov), prediction interval width (PI-W), and Winkler score (WS) averaged across sequences in each dataset for the proposed methods with different sequence encoder: RNN and Transformer. The target coverage is $1-\alpha=0.9$. The best result among two encoder choices for each method except the undercoveraged results based on Winkler score is \underline{underlined}.}
\label{tab:...}

\begin{adjustbox}{max width=\linewidth}
\begin{tabular}{lccccccccc}
\toprule
& \multicolumn{9}{c}{Air} \\
\cmidrule(lr){2-10}
Method
& \multicolumn{3}{c}{LR}
& \multicolumn{3}{c}{LSTM}
& \multicolumn{3}{c}{Chronos} \\
\cmidrule(lr){2-4}\cmidrule(lr){5-7}\cmidrule(lr){8-10}
& $\Delta$Cov & PI-W & WS
& $\Delta$Cov & PI-W & WS
& $\Delta$Cov & PI-W & WS \\
\midrule
QR-RNN   & $+0.012_{\scriptscriptstyle \pm .016}$ & $85.3_{\scriptscriptstyle \pm 10.5}$ & $\underline{121.3}_{\scriptscriptstyle \pm 12.5}$ & $-0.021_{\scriptscriptstyle \pm .017}$ & $61.1_{\scriptscriptstyle \pm 9.32}$ & $\underline{103.4}_{\scriptscriptstyle \pm 12.1}$ & $+0.014_{\scriptscriptstyle \pm .014}$ & $87.1_{\scriptscriptstyle \pm 10.7}$ & $\underline{125.5}_{\scriptscriptstyle \pm 10.6}$ \\
QR-Transformer   & $-0.004_{\scriptscriptstyle \pm .014}$ & $92.3_{\scriptscriptstyle \pm 14.7}$ & $136.3_{\scriptscriptstyle \pm 18.5}$ & $-0.018_{\scriptscriptstyle \pm .018}$ & $69.4_{\scriptscriptstyle \pm 10.6}$ & $124.2_{\scriptscriptstyle \pm 16.0}$ & $-0.002_{\scriptscriptstyle \pm .013}$ & $93.5_{\scriptscriptstyle \pm 9.82}$ & $146.3_{\scriptscriptstyle \pm 13.2}$ \\

QFE-RNN   & $-0.064_{\scriptscriptstyle \pm .015}$ & $61.3_{\scriptscriptstyle \pm 6.40}$ & $\underline{132.7}_{\scriptscriptstyle \pm 15.6}$ & $-0.126_{\scriptscriptstyle \pm .026}$ & $37.9_{\scriptscriptstyle \pm 4.73}$ & $\underline{124.4}_{\scriptscriptstyle \pm 16.3}$ & $-0.091_{\scriptscriptstyle \pm .032}$ & $52.5_{\scriptscriptstyle \pm 7.39}$ & $\underline{170.1}_{\scriptscriptstyle \pm 13.1}$ \\
QFE-Transformer   & $-0.084_{\scriptscriptstyle \pm .021}$ & $63.4_{\scriptscriptstyle \pm 6.78}$ & $152.7_{\scriptscriptstyle \pm 19.8}$ & $-0.112_{\scriptscriptstyle \pm .028}$ & $45.7_{\scriptscriptstyle \pm 6.78}$ & $143.4_{\scriptscriptstyle \pm 22.3}$ & $-0.096_{\scriptscriptstyle \pm .022}$ & $58.2_{\scriptscriptstyle \pm 6.31}$ & $196.1_{\scriptscriptstyle \pm 17.1}$ \\
\bottomrule
\end{tabular}
\end{adjustbox}

\end{table}


\begin{table}[!htbp]
\centering
\caption{Coverage gap ($\Delta$Cov), prediction interval width (PI-W), and Winkler score (WS) averaged across sequences in each dataset for the proposed methods with different sequence encoder: RNN and Transformer. The target coverage is $1-\alpha=0.95$. The best result among two encoder choices for each method except the undercoveraged results based on Winkler score is \underline{underlined}.}
\label{tab:...}

\begin{adjustbox}{max width=\linewidth}
\begin{tabular}{lccccccccc}
\toprule

& \multicolumn{9}{c}{Air} \\
\cmidrule(lr){2-10}

Method
& \multicolumn{3}{c}{LR}
& \multicolumn{3}{c}{LSTM}
& \multicolumn{3}{c}{Chronos} \\

\cmidrule(lr){2-4}
\cmidrule(lr){5-7}
\cmidrule(lr){8-10}

& $\Delta$Cov & PI-W & WS
& $\Delta$Cov & PI-W & WS
& $\Delta$Cov & PI-W & WS \\

\midrule

QR-RNN
& $-0.005_{\scriptscriptstyle \pm .009}$
& $107.4_{\scriptscriptstyle \pm 15.0}$
& $\underline{155.6}_{\scriptscriptstyle \pm 17.1}$
& $-0.027_{\scriptscriptstyle \pm .011}$
& $78.1_{\scriptscriptstyle \pm 7.74}$
& $\underline{135.8}_{\scriptscriptstyle \pm 19.0}$
& $+0.001_{\scriptscriptstyle \pm .013}$
& $109.0_{\scriptscriptstyle \pm 13.3}$
& $\underline{159.8}_{\scriptscriptstyle \pm 13.8}$ \\

QR-Transformer
& $-0.010_{\scriptscriptstyle \pm .012}$
& $116.0_{\scriptscriptstyle \pm 18.2}$
& $172.7_{\scriptscriptstyle \pm 21.7}$
& $-0.016_{\scriptscriptstyle \pm .009}$
& $93.7_{\scriptscriptstyle \pm 14.1}$
& $166.2_{\scriptscriptstyle \pm 21.8}$
& $-0.006_{\scriptscriptstyle \pm .012}$
& $121.1_{\scriptscriptstyle \pm 14.9}$
& $191.3_{\scriptscriptstyle \pm 16.6}$ \\

QFE-RNN
& $-0.097_{\scriptscriptstyle \pm .017}$
& $64.8_{\scriptscriptstyle \pm 7.03}$
& $\underline{196.8}_{\scriptscriptstyle \pm 27.0}$
& $-0.153_{\scriptscriptstyle \pm .032}$
& $40.3_{\scriptscriptstyle \pm 4.47}$
& $\underline{203.7}_{\scriptscriptstyle \pm 32.7}$
& $-0.113_{\scriptscriptstyle \pm .035}$
& $60.0_{\scriptscriptstyle \pm 8.84}$
& $\underline{282.4}_{\scriptscriptstyle \pm 22.6}$ \\

QFE-Transformer
& $-0.118_{\scriptscriptstyle \pm .022}$
& $66.8_{\scriptscriptstyle \pm 6.93}$
& $234.7_{\scriptscriptstyle \pm 37.4}$
& $-0.138_{\scriptscriptstyle \pm .025}$
& $49.2_{\scriptscriptstyle \pm 5.23}$
& $230.0_{\scriptscriptstyle \pm 36.9}$
& $-0.128_{\scriptscriptstyle \pm 0.024}$
& $61.8_{\scriptscriptstyle \pm 7.13}$
& $325.3_{\scriptscriptstyle \pm 27.1}$ \\
\bottomrule
\end{tabular}
\end{adjustbox}

\end{table}


\begin{table}[!htbp]
\centering
\caption{Coverage gap ($\Delta$Cov), prediction interval width (PI-W), and Winkler score (WS) averaged across sequences in each dataset for the proposed methods with different sequence encoder: RNN and Transformer. The target coverage is $1-\alpha=0.8$. The best result among two encoder choices for each method except the undercoveraged results based on Winkler score is \underline{underlined}.}
\label{tab:...}

\begin{adjustbox}{max width=\linewidth}
\begin{tabular}{lccccccccc}
\toprule
& \multicolumn{9}{c}{Air} \\
\cmidrule(lr){2-10}
Method
& \multicolumn{3}{c}{LR}
& \multicolumn{3}{c}{LSTM}
& \multicolumn{3}{c}{Chronos} \\
\cmidrule(lr){2-4}\cmidrule(lr){5-7}\cmidrule(lr){8-10}
& $\Delta$Cov & PI-W & WS
& $\Delta$Cov & PI-W & WS
& $\Delta$Cov & PI-W & WS \\
\midrule
QR-RNN
& $+0.011_{\scriptscriptstyle \pm 0.019}$
& $60.9_{\scriptscriptstyle \pm 7.9}$
& $\underline{93.7}_{\scriptscriptstyle \pm 10.9}$
& $-0.022_{\scriptscriptstyle \pm 0.028}$
& $43.1_{\scriptscriptstyle \pm 6.9}$
& $\underline{77.7}_{\scriptscriptstyle \pm 9.5}$
& $+0.016_{\scriptscriptstyle \pm 0.026}$
& $62.5_{\scriptscriptstyle \pm 8.1}$
& $\underline{97.4}_{\scriptscriptstyle \pm 7.8}$ \\

QR-Transformer
& $+0.016_{\scriptscriptstyle \pm 0.023}$
& $69.3_{\scriptscriptstyle \pm 10.7}$
& $105.6_{\scriptscriptstyle \pm 14.6}$
& $-0.022_{\scriptscriptstyle \pm 0.023}$
& $47.9_{\scriptscriptstyle \pm 4.0}$
& $91.9_{\scriptscriptstyle \pm 11.7}$
& $-0.000_{\scriptscriptstyle \pm 0.032}$
& $68.2_{\scriptscriptstyle \pm 7.6}$
& $113.8_{\scriptscriptstyle \pm 9.9}$ \\

QFE-RNN
& $-0.002_{\scriptscriptstyle \pm 0.022}$
& $55.0_{\scriptscriptstyle \pm 5.8}$
& $\underline{96.4}_{\scriptscriptstyle \pm 10.3}$
& $-0.073_{\scriptscriptstyle \pm 0.022}$
& $33.1_{\scriptscriptstyle \pm 4.1}$
& $\underline{81.9}_{\scriptscriptstyle \pm 9.7}$
& $-0.032_{\scriptscriptstyle \pm 0.034}$
& $45.4_{\scriptscriptstyle \pm 6.6}$
& $\underline{109.8}_{\scriptscriptstyle \pm 8.6}$ \\

QFE-Transformer
& $-0.033_{\scriptscriptstyle \pm 0.032}$
& $55.2_{\scriptscriptstyle \pm 4.4}$
& $109.5_{\scriptscriptstyle \pm 13.7}$
& $-0.055_{\scriptscriptstyle \pm 0.030}$
& $40.2_{\scriptscriptstyle \pm 4.6}$
& $95.0_{\scriptscriptstyle \pm 12.8}$
& $-0.029_{\scriptscriptstyle \pm 0.028}$
& $51.9_{\scriptscriptstyle \pm 6.4}$
& $127.4_{\scriptscriptstyle \pm 10.7}$ \\
\bottomrule
\end{tabular}
\end{adjustbox}

\end{table}

\begin{table}[!htbp]
\centering
\caption{Coverage gap ($\Delta$Cov), prediction interval width (PI-W), and Winkler score (WS), averaged across sequences within each dataset, for all methods. The target coverage is $1-\alpha=0.95$.}
\label{table:main_results_2}

\begin{adjustbox}{max width=\linewidth}
\begin{tabular}{lccccccccc}
\toprule
& \multicolumn{9}{c}{Air} \\
\cmidrule(lr){2-10}
Method
& \multicolumn{3}{c}{LR}
& \multicolumn{3}{c}{LSTM}
& \multicolumn{3}{c}{Chronos} \\
\cmidrule(lr){2-4}\cmidrule(lr){5-7}\cmidrule(lr){8-10}
& $\Delta$Cov & PI-W & WS
& $\Delta$Cov & PI-W & WS
& $\Delta$Cov & PI-W & WS \\
\midrule
Split CP
& $+0.014_{\scriptscriptstyle \pm 0.008}$
& $136.3_{\scriptscriptstyle \pm 14.4}$
& $198.8_{\scriptscriptstyle \pm 17.4}$
& $-0.002_{\scriptscriptstyle \pm 0.012}$
& $104.3_{\scriptscriptstyle \pm 11.2}$
& $187.2_{\scriptscriptstyle \pm 20.0}$
& $+0.002_{\scriptscriptstyle \pm 0.005}$
& $256.0_{\scriptscriptstyle \pm 20.4}$
& $422.4_{\scriptscriptstyle \pm 29.6}$ \\

NexCP
& $-0.002_{\scriptscriptstyle \pm 0.001}$
& $118.2_{\scriptscriptstyle \pm 14.4}$
& $185.6_{\scriptscriptstyle \pm 19.1}$
& $-0.005_{\scriptscriptstyle \pm 0.001}$
& $107.3_{\scriptscriptstyle \pm 13.2}$
& $179.6_{\scriptscriptstyle \pm 19.6}$
& $-0.005_{\scriptscriptstyle \pm 0.002}$
& $253.1_{\scriptscriptstyle \pm 18.0}$
& $362.8_{\scriptscriptstyle \pm 22.4}$ \\

HopCPT      & $-0.014_{\scriptscriptstyle \pm .015}$ & $108.4_{\scriptscriptstyle \pm 18.8}$ & $196.9_{\scriptscriptstyle \pm 19.8}$ & $-0.007_{\scriptscriptstyle \pm .009}$ & $100.6_{\scriptscriptstyle \pm 12.2}$ & $181.9_{\scriptscriptstyle \pm 18.8}$ & $-0.010_{\scriptscriptstyle \pm .011}$ & $232.2_{\scriptscriptstyle \pm 28.0}$ & $412.9_{\scriptscriptstyle \pm 41.2}$ \\

SPCI
& $-0.140_{\scriptscriptstyle \pm 0.005}$
& $80.3_{\scriptscriptstyle \pm 9.4}$
& $235.1_{\scriptscriptstyle \pm 23.7}$
& $-0.139_{\scriptscriptstyle \pm 0.005}$
& $67.9_{\scriptscriptstyle \pm 8.4}$
& $208.2_{\scriptscriptstyle \pm 24.6}$
& $-0.118_{\scriptscriptstyle \pm 0.008}$
& $86.2_{\scriptscriptstyle \pm 7.7}$
& $230.9_{\scriptscriptstyle \pm 19.2}$ \\
\bottomrule
\end{tabular}
\end{adjustbox}

\end{table}

\begin{table}[!htbp]
\centering
\caption{Coverage gap ($\Delta$Cov), prediction interval width (PI-W), and Winkler score (WS), averaged across sequences within each dataset, for all methods. The target coverage is $1-\alpha=0.8$.}
\label{table:main_results_3}

\begin{adjustbox}{max width=\linewidth}
\begin{tabular}{lccccccccc}
\toprule
& \multicolumn{9}{c}{Air} \\
\cmidrule(lr){2-10}
Method
& \multicolumn{3}{c}{LR}
& \multicolumn{3}{c}{LSTM}
& \multicolumn{3}{c}{Chronos} \\
\cmidrule(lr){2-4}\cmidrule(lr){5-7}\cmidrule(lr){8-10}
& $\Delta$Cov & PI-W & WS
& $\Delta$Cov & PI-W & WS
& $\Delta$Cov & PI-W & WS \\
\midrule
Split CP
& $+0.045_{\scriptscriptstyle \pm 0.015}$
& $68.4_{\scriptscriptstyle \pm 7.9}$
& $110.1_{\scriptscriptstyle \pm 11.7}$
& $-0.001_{\scriptscriptstyle \pm 0.017}$
& $48.9_{\scriptscriptstyle \pm 5.2}$
& $98.4_{\scriptscriptstyle \pm 11.0}$
& $+0.030_{\scriptscriptstyle \pm 0.016}$
& $125.8_{\scriptscriptstyle \pm 13.7}$
& $227.9_{\scriptscriptstyle \pm 17.0}$ \\

NexCP
& $+0.001_{\scriptscriptstyle \pm 0.002}$
& $61.0_{\scriptscriptstyle \pm 7.6}$
& $106.7_{\scriptscriptstyle \pm 12.0}$
& $-0.005_{\scriptscriptstyle \pm 0.002}$
& $51.0_{\scriptscriptstyle \pm 5.9}$
& $97.3_{\scriptscriptstyle \pm 10.9}$
& $-0.005_{\scriptscriptstyle \pm 0.002}$
& $118.0_{\scriptscriptstyle \pm 10.8}$
& $214.8_{\scriptscriptstyle \pm 15.6}$ \\

HopCPT      & $-0.023_{\scriptscriptstyle \pm .028}$ & $55.1_{\scriptscriptstyle \pm 8.32}$ & $107.7_{\scriptscriptstyle \pm 11.8}$ & $-0.033_{\scriptscriptstyle \pm .033}$ & $44.9_{\scriptscriptstyle \pm 4.87}$ & $96.9_{\scriptscriptstyle \pm 10.6}$ & $-0.010_{\scriptscriptstyle \pm .025}$ & $105.7_{\scriptscriptstyle \pm 17.8}$ & $219.6_{\scriptscriptstyle \pm 29.9}$ \\

SPCI
& $-0.151_{\scriptscriptstyle \pm 0.008}$
& $50.2_{\scriptscriptstyle \pm 6.1}$
& $118.1_{\scriptscriptstyle \pm 13.4}$
& $-0.146_{\scriptscriptstyle \pm 0.008}$
& $40.4_{\scriptscriptstyle \pm 5.1}$
& $99.7_{\scriptscriptstyle \pm 12.5}$
& $-0.113_{\scriptscriptstyle \pm 0.011}$
& $53.2_{\scriptscriptstyle \pm 4.9}$
& $116.7_{\scriptscriptstyle \pm 10.3}$ \\
\bottomrule
\end{tabular}
\end{adjustbox}

\end{table}


\begin{table}[!htbp]
\centering
\caption{Ablation study of architectural constraints for preventing quantile crossing in DSCP-QR and DSCP-QFE. The suffix ``-C'' denotes the corresponding method without these constraints. CR denotes the quantile crossing occurrence rate across all sequences.}
\label{table:ablation_crossing}

\begin{adjustbox}{max width=\linewidth}
\begin{tabular}{lcccccccccccc}
\toprule

& \multicolumn{12}{c}{Air} \\
\cmidrule(lr){2-13}

Method
& \multicolumn{4}{c}{LR}
& \multicolumn{4}{c}{LSTM}
& \multicolumn{4}{c}{Chronos} \\

\cmidrule(lr){2-5}
\cmidrule(lr){6-9}
\cmidrule(lr){10-13}

& $\Delta$Cov & PI-W & WS & CR
& $\Delta$Cov & PI-W & WS & CR
& $\Delta$Cov & PI-W & WS & CR \\

\midrule

QR-RNN   & $+0.012_{\scriptscriptstyle \pm .016}$ & $85.3_{\scriptscriptstyle \pm 10.5}$ & $121.3_{\scriptscriptstyle \pm 12.5}$ & $0.000\%$ & $-0.021_{\scriptscriptstyle \pm .017}$ & $61.1_{\scriptscriptstyle \pm 9.32}$ & $103.4_{\scriptscriptstyle \pm 12.1}$ & $0.000\%$  & $+0.014_{\scriptscriptstyle \pm .014}$ & $87.1_{\scriptscriptstyle \pm 10.7}$ & $125.5_{\scriptscriptstyle \pm 10.6}$ & $0.000\%$  \\
QR-Transformer   
& $-0.004_{\scriptscriptstyle \pm .014}$ & 
$92.3_{\scriptscriptstyle \pm 14.7}$ & 
$136.3_{\scriptscriptstyle \pm 18.5}$ & 
$0.000\%$  &
$-0.018_{\scriptscriptstyle \pm .018}$ & 
$69.4_{\scriptscriptstyle \pm 10.6}$ & 
$124.2_{\scriptscriptstyle \pm 16.0}$ & 
$0.000\%$  &
$-0.002_{\scriptscriptstyle \pm .013}$ & 
$93.5_{\scriptscriptstyle \pm 9.82}$ & 
$146.3_{\scriptscriptstyle \pm 13.2}$ &
$0.000\%$  \\

QFE-RNN   
& $-0.064_{\scriptscriptstyle \pm .015}$ 
& $61.3_{\scriptscriptstyle \pm 6.40}$ 
& $132.7_{\scriptscriptstyle \pm 15.6}$
& $0.000\%$
& $-0.126_{\scriptscriptstyle \pm .026}$ 
& $37.9_{\scriptscriptstyle \pm 4.73}$ 
& $124.4_{\scriptscriptstyle \pm 16.3}$ 
& $0.000\%$
& $-0.091_{\scriptscriptstyle \pm .032}$ 
& $52.5_{\scriptscriptstyle \pm 7.39}$ 
& $170.1_{\scriptscriptstyle \pm 13.1}$ 
& $0.000\%$ \\
\bottomrule
\end{tabular}
\end{adjustbox}

\end{table}


\begin{table}[!htbp]
\centering
\caption{Coverage gap ($\Delta$Cov), prediction interval width (PI-W), and Winkler score (WS), averaged across sequences within each dataset, for DSCP-Local using RNN and Transformer encoders with three sets of target quantile levels: the median only (${0.5}$), three quantile levels (${0.1,0.5,0.9}$), and nine quantile levels (${0.1,0.2,0.3,0.4,0.5,0.6,0.7,0.8,0.9}$). The target coverage is $1-\alpha=0.9$.}
\label{table:ablation_local}

\begin{adjustbox}{max width=\linewidth}
\begin{tabular}{lccccccccc}
\toprule
& \multicolumn{9}{c}{Air} \\
\cmidrule(lr){2-10}
Method
& \multicolumn{3}{c}{LR}
& \multicolumn{3}{c}{LSTM}
& \multicolumn{3}{c}{Chronos} \\
\cmidrule(lr){2-4}\cmidrule(lr){5-7}\cmidrule(lr){8-10}
& $\Delta$Cov & PI-W & WS
& $\Delta$Cov & PI-W & WS
& $\Delta$Cov & PI-W & WS \\

\midrule

DSCP-RNN-Local-Median
& $-0.017_{\scriptscriptstyle \pm 0.005}$
& $80.8_{\scriptscriptstyle \pm 10.4}$
& $137.0_{\scriptscriptstyle \pm 14.6}$
& $-0.019_{\scriptscriptstyle \pm 0.006}$
& $66.0_{\scriptscriptstyle \pm 9.2}$
& $119.5_{\scriptscriptstyle \pm 14.7}$
& $+0.029_{\scriptscriptstyle \pm 0.006}$
& $120.1_{\scriptscriptstyle \pm 10.2}$
& $163.9_{\scriptscriptstyle \pm 11.8}$ \\

DSCP-RNN-Local-3   & $-0.015_{\scriptscriptstyle \pm .003}$ & $81.0_{\scriptscriptstyle \pm 10.0}$ & $133.2_{\scriptscriptstyle \pm 14.4}$ & $-0.013_{\scriptscriptstyle \pm .005}$ & $67.8_{\scriptscriptstyle \pm 8.67}$ & $116.0_{\scriptscriptstyle \pm 13.3}$ & $+0.027_{\scriptscriptstyle \pm .006}$ & $117.4_{\scriptscriptstyle \pm 9.14}$ & $160.7_{\scriptscriptstyle \pm 11.4}$ \\

DSCP-RNN-Local-9
& $-0.006_{\scriptscriptstyle \pm 0.004}$
& $82.4_{\scriptscriptstyle \pm 10.3}$
& $134.3_{\scriptscriptstyle \pm 14.8}$
& $-0.004_{\scriptscriptstyle \pm 0.004}$
& $67.7_{\scriptscriptstyle \pm 9.0}$
& $114.0_{\scriptscriptstyle \pm 13.6}$
& $+0.033_{\scriptscriptstyle \pm 0.006}$
& $118.3_{\scriptscriptstyle \pm 10.9}$
& $160.4_{\scriptscriptstyle \pm 12.3}$ \\

DSCP-Transformer-Local-Median
& $-0.013_{\scriptscriptstyle \pm 0.003}$
& $81.1_{\scriptscriptstyle \pm 10.4}$
& $136.1_{\scriptscriptstyle \pm 15.5}$
& $-0.009_{\scriptscriptstyle \pm 0.003}$
& $69.4_{\scriptscriptstyle \pm 8.2}$
& $121.7_{\scriptscriptstyle \pm 13.6}$
& $+0.011_{\scriptscriptstyle \pm 0.005}$
& $135.1_{\scriptscriptstyle \pm 12.8}$
& $188.6_{\scriptscriptstyle \pm 14.7}$ \\

DSCP-Transformer-Local-3   & $-0.007_{\scriptscriptstyle \pm .004}$ & $83.4_{\scriptscriptstyle \pm 9.60}$ & $140.8_{\scriptscriptstyle \pm 15.5}$ & $-0.007_{\scriptscriptstyle \pm .003}$ & $71.8_{\scriptscriptstyle \pm 9.29}$ & $127.6_{\scriptscriptstyle \pm 15.1}$ & $+0.015_{\scriptscriptstyle \pm .005}$ & $130.4_{\scriptscriptstyle \pm 11.9}$ & $179.9_{\scriptscriptstyle \pm 14.2}$ \\

DSCP-Transformer-Local-9
& $-0.012_{\scriptscriptstyle \pm 0.002}$
& $80.8_{\scriptscriptstyle \pm 10.3}$
& $135.8_{\scriptscriptstyle \pm 15.2}$
& $-0.008_{\scriptscriptstyle \pm 0.003}$
& $69.5_{\scriptscriptstyle \pm 8.4}$
& $123.0_{\scriptscriptstyle \pm 13.6}$
& $+0.017_{\scriptscriptstyle \pm 0.003}$
& $127.9_{\scriptscriptstyle \pm 11.6}$
& $178.4_{\scriptscriptstyle \pm 13.7}$ \\

\midrule
& \multicolumn{9}{c}{Solar} \\
\cmidrule(lr){2-10}
Method
& \multicolumn{3}{c}{LR}
& \multicolumn{3}{c}{LSTM}
& \multicolumn{3}{c}{Chronos} \\
\cmidrule(lr){2-4}\cmidrule(lr){5-7}\cmidrule(lr){8-10}
& $\Delta$Cov & PI-W & WS
& $\Delta$Cov & PI-W & WS
& $\Delta$Cov & PI-W & WS \\

\midrule

DSCP-RNN-Local-Median
& $-0.013_{\scriptscriptstyle \pm 0.009}$
& $89.6_{\scriptscriptstyle \pm 19.3}$
& $141.7_{\scriptscriptstyle \pm 30.2}$
& $-0.003_{\scriptscriptstyle \pm 0.009}$
& $61.0_{\scriptscriptstyle \pm 15.7}$
& $117.1_{\scriptscriptstyle \pm 20.5}$
& $+0.007_{\scriptscriptstyle \pm 0.007}$
& $82.3_{\scriptscriptstyle \pm 13.5}$
& $116.2_{\scriptscriptstyle \pm 13.8}$ \\

DSCP-RNN-Local-3  & $-0.001_{\scriptscriptstyle \pm .009}$ & $91.0_{\scriptscriptstyle \pm 18.9}$ & $124.3_{\scriptscriptstyle \pm 28.7}$ & $+0.016_{\scriptscriptstyle \pm .010}$ & $64.0_{\scriptscriptstyle \pm 15.3}$ & $87.0_{\scriptscriptstyle \pm 17.6}$ & $+0.011_{\scriptscriptstyle \pm .007}$ & $83.1_{\scriptscriptstyle \pm 14.3}$ & $109.9_{\scriptscriptstyle \pm 13.2}$ \\

DSCP-RNN-Local-9
& $+0.001_{\scriptscriptstyle \pm 0.009}$
& $91.5_{\scriptscriptstyle \pm 19.1}$
& $125.9_{\scriptscriptstyle \pm 31.6}$
& $+0.011_{\scriptscriptstyle \pm 0.010}$
& $62.3_{\scriptscriptstyle \pm 15.9}$
& $87.1_{\scriptscriptstyle \pm 17.8}$
& $+0.010_{\scriptscriptstyle \pm 0.007}$
& $83.3_{\scriptscriptstyle \pm 14.2}$
& $110.4_{\scriptscriptstyle \pm 13.1}$ \\

DSCP-Transformer-Local-Median
& $+0.009_{\scriptscriptstyle \pm 0.007}$
& $93.6_{\scriptscriptstyle \pm 18.9}$
& $138.1_{\scriptscriptstyle \pm 27.6}$
& $+0.005_{\scriptscriptstyle \pm 0.007}$
& $65.0_{\scriptscriptstyle \pm 15.9}$
& $113.2_{\scriptscriptstyle \pm 19.8}$
& $-0.000_{\scriptscriptstyle \pm 0.009}$
& $86.3_{\scriptscriptstyle \pm 15.6}$
& $145.0_{\scriptscriptstyle \pm 17.6}$ \\

DSCP-Transformer-Local-3   & $-0.001_{\scriptscriptstyle \pm .008}$ & $88.6_{\scriptscriptstyle \pm 18.4}$ & $123.8_{\scriptscriptstyle \pm 27.6}$ & $+0.002_{\scriptscriptstyle \pm .009}$ & $60.7_{\scriptscriptstyle \pm 15.7}$ & $105.7_{\scriptscriptstyle \pm 19.0}$ & $+0.001_{\scriptscriptstyle \pm .011}$ & $75.9_{\scriptscriptstyle \pm 12.4}$ & $122.6_{\scriptscriptstyle \pm 14.3}$ \\

DSCP-Transformer-Local-9
& $+0.003_{\scriptscriptstyle \pm 0.008}$
& $88.5_{\scriptscriptstyle \pm 18.4}$
& $130.9_{\scriptscriptstyle \pm 30.7}$
& $+0.002_{\scriptscriptstyle \pm 0.007}$
& $60.2_{\scriptscriptstyle \pm 16.2}$
& $110.4_{\scriptscriptstyle \pm 20.0}$
& $+0.007_{\scriptscriptstyle \pm 0.009}$
& $79.2_{\scriptscriptstyle \pm 13.1}$
& $128.7_{\scriptscriptstyle \pm 13.6}$ \\

\midrule
& \multicolumn{9}{c}{Sapflow} \\
\cmidrule(lr){2-10}
Method
& \multicolumn{3}{c}{LR}
& \multicolumn{3}{c}{LSTM}
& \multicolumn{3}{c}{Chronos} \\
\cmidrule(lr){2-4}\cmidrule(lr){5-7}\cmidrule(lr){8-10}
& $\Delta$Cov & PI-W & WS
& $\Delta$Cov & PI-W & WS
& $\Delta$Cov & PI-W & WS \\

\midrule
DSCP-RNN-Local-Median
& $-0.004_{\scriptscriptstyle \pm 0.008}$
& $542.9_{\scriptscriptstyle \pm 348.2}$
& $878.3_{\scriptscriptstyle \pm 605.7}$
& $+0.007_{\scriptscriptstyle \pm 0.010}$
& $625.7_{\scriptscriptstyle \pm 527.0}$
& $941.8_{\scriptscriptstyle \pm 746.5}$
& $+0.016_{\scriptscriptstyle \pm 0.012}$
& $1512.6_{\scriptscriptstyle \pm 1199.0}$
& $2040.5_{\scriptscriptstyle \pm 1520.1}$ \\

DSCP-RNN-Local-3  & $+0.001_{\scriptscriptstyle \pm .012}$ & $558.2_{\scriptscriptstyle \pm 375.5}$ & $864.2_{\scriptscriptstyle \pm 619.6}$ & $+0.007_{\scriptscriptstyle \pm .013}$ & $632.1_{\scriptscriptstyle \pm 550.7}$ & $911.4_{\scriptscriptstyle \pm 743.6}$ & $+0.017_{\scriptscriptstyle \pm .010}$ & $1400.5_{\scriptscriptstyle \pm 1102.3}$ & $1860.4_{\scriptscriptstyle \pm 1382.5}$ \\

DSCP-RNN-Local-9
& $-0.001_{\scriptscriptstyle \pm 0.014}$
& $546.3_{\scriptscriptstyle \pm 356.4}$
& $856.3_{\scriptscriptstyle \pm 609.4}$
& $+0.008_{\scriptscriptstyle \pm 0.014}$
& $616.9_{\scriptscriptstyle \pm 526.0}$
& $905.2_{\scriptscriptstyle \pm 737.6}$
& $+0.025_{\scriptscriptstyle \pm 0.013}$
& $1370.2_{\scriptscriptstyle \pm 1075.7}$
& $1771.3_{\scriptscriptstyle \pm 1321.5}$ \\

DSCP-Transformer-Local-Median
& $+0.009_{\scriptscriptstyle \pm 0.007}$
& $589.2_{\scriptscriptstyle \pm 384.4}$
& $916.0_{\scriptscriptstyle \pm 639.3}$
& $+0.007_{\scriptscriptstyle \pm 0.007}$
& $868.0_{\scriptscriptstyle \pm 1172.2}$
& $1207.7_{\scriptscriptstyle \pm 1387.7}$
& $+0.010_{\scriptscriptstyle \pm 0.010}$
& $2031.6_{\scriptscriptstyle \pm 1703.9}$
& $3242.6_{\scriptscriptstyle \pm 2676.1}$ \\

DSCP-Transformer-Local-3  & $+0.009_{\scriptscriptstyle \pm .006}$ & $586.1_{\scriptscriptstyle \pm 380.5}$ & $912.5_{\scriptscriptstyle \pm 634.6}$ & $+0.007_{\scriptscriptstyle \pm .008}$ & $898.3_{\scriptscriptstyle \pm 1357.8}$ & $1245.6_{\scriptscriptstyle \pm 1264.3}$ & $+0.007_{\scriptscriptstyle \pm .005}$ & $1927.9_{\scriptscriptstyle \pm 1557.9}$ & $3044.2_{\scriptscriptstyle \pm 2340.9}$  \\

DSCP-Transformer-Local-9
& $+0.009_{\scriptscriptstyle \pm 0.007}$
& $585.5_{\scriptscriptstyle \pm 380.2}$
& $913.0_{\scriptscriptstyle \pm 635.2}$
& $+0.006_{\scriptscriptstyle \pm 0.007}$
& $870.7_{\scriptscriptstyle \pm 1204.3}$
& $1220.5_{\scriptscriptstyle \pm 1439.9}$
& $+0.007_{\scriptscriptstyle \pm 0.007}$
& $1856.9_{\scriptscriptstyle \pm 1535.2}$
& $2953.7_{\scriptscriptstyle \pm 2358.1}$ \\

\bottomrule
\end{tabular}
\end{adjustbox}

\end{table}

\end{document}